%% file: main.tex
\documentclass{article}

\PassOptionsToPackage{numbers, sort&compress}{natbib}

 \usepackage[preprint]{nips2026/neurips_2026}

\usepackage[utf8]{inputenc} 
\usepackage[T1]{fontenc}    
\usepackage[colorlinks=true,allcolors=perfblue]{hyperref}
\usepackage{url}            
\usepackage{booktabs}       
\usepackage{amsfonts}       
\usepackage{nicefrac}       
\usepackage{microtype}      
\usepackage{xcolor}         

\definecolor{perfblue}{RGB}{120, 140, 250}

\usepackage{amsmath}
\usepackage{amssymb}
\usepackage{mathtools}
\usepackage{amsthm}
\usepackage{float}
\input{notations}

\usepackage{wrapfig}
\usepackage{multirow}
\usepackage{subcaption}

\theoremstyle{plain}
\newtheorem{theorem}{Theorem}[section]
\newtheorem{proposition}[theorem]{Proposition}
\newtheorem{lemma}[theorem]{Lemma}

\theoremstyle{definition}

\theoremstyle{remark}
\newtheorem{remark}[theorem]{Remark}

\title{Extracting Algorithms in Pre-trained LLMs: \\A Case on Hidden Markov Models}

\author{
  Yijia~Dai\thanks{Correspondence to \texttt{yd73@cornell.edu}.} \quad Zhaolin~Gao \quad Yahya~Sattar \quad Jennifer~J.~Sun \quad Sarah~Dean\\\\
  Cornell University
}

\begin{document}

\maketitle

\begin{abstract}
   Large language models (LLMs) display a striking ability to predict next observations from Hidden Markov Models (HMMs) via in-context learning (ICL), but the algorithm underlying this capability remains undetermined: prior work has proposed several candidates without consensus, and none has been grounded in the model's internal activations. We close this gap with a three-stage pipeline. First, we empirically compare LLM behavior against a suite of candidate algorithms and narrow the space to three classes---though no single class explains LLM behavior across all HMM settings and sequence lengths. Second, we derive theoretical connections between the three classes and show how each can be implemented in-context by a Transformer, validating the construction in a small trained Transformer. Third, returning to pre-trained LLMs, we introduce the \textit{Principal Activations Probe} (PAP), a layer-wise probing and intervention method that isolates algorithmic signals in model activations. PAP reveals low-dimensional linear representations that causally drive model predictions and track empirical ICL performance. PAP further reveals how these representations shift with properties of the underlying HMM regime; distinct computational stages are localized to different layers. Together, our results connect the in-context behavior of pre-trained LLMs to the underlying internal mechanisms and advance our understanding of how LLMs perform ICL on HMMs.
\end{abstract}

\section{Introduction}
\label{sec:intro}
\input{introduction}

\section{In-context Learning on HMMs}
\label{sec:icl_on_hmms}
\input{section_icl_on_hmms}

\section{Transformers are Capable of Implementing Any of the Ansatzes}
\label{sec:theory}

In Section \ref{sec:llm_emp}, we found that the three classes of algorithms (\texttt{Spectral}, \texttt{Linear $n$-gram}, and \texttt{Non-linear $n$-gram}) closely match LLM predictions empirically. We refer to these as \emph{ansatzes} and show that each can, in principle, be implemented by a Transformer as in-context learning.

\subsection{Construction of Each Ansatz}
\label{sec:construction_ansatz}
\input{section_construction_ansatzes}

\subsection{Which Algorithm does a Small Trained Transformer use?}
\input{section_alg_trained_transformer}

\section{Probing Algorithmic Representations in LLMs Internal Activations}
\label{sec:llm_internal}
\input{section_llm_internal}

\section{Related Works}
\input{related_works}

\section{Limitations and Conclusions}


We study how pre-trained LLMs perform in-context sequence prediction on Hidden Markov 
Models, connecting observed behavior to mechanistic implementation through a three-stage 
pipeline.

\paragraph{Empirically.} No single classical algorithm matches LLM predictions across HMM regimes, but three algorithm classes---\texttt{Linear $n$-gram}, \texttt{Non-linear 
$n$-gram}, and \texttt{Spectral}---collectively span the observed behavior.

\paragraph{Theoretically.} Transformers can implement all three through a common 
nonlinear-feature construction; moreover, a small Transformer trained on a fixed HMM 
recovers the finite-window \texttt{Linear $n$-gram} predictor almost exactly.

\paragraph{Mechanistically.} Our Principal Activations Probe (PAP) reveals that the 
operative algorithmic representation shifts with the HMM regime. Grouping \texttt{Linear $n$-gram} and \texttt{Non-linear $n$-gram} into a unified \texttt{Soft $n$-gram} class, 
we find that $n$-gram feature learning is the most causally effective representation 
across all regimes: bigram-like statistics suffice when emissions are nearly deterministic, while the richer \texttt{Soft $n$-gram} representation drives prediction when genuine 
belief tracking is required.

The central conclusion is that LLMs do not implement statistically optimal, iterative 
inference. Instead, they appear to learn finite-window, gradient-descent-like 
approximations well-described by \texttt{Soft $n$-gram} statistics---even in regimes where longer-range integration would yield strictly lower cross-entropy loss.

Our analysis is restricted to synthetic HMMs with small finite alphabets; linear probing further limits what we can conclude about nonlinear or distributed representations. Extending PAP to natural-language tasks with latent structure and to circuit-level rather than layer-level analysis are natural next steps.

\clearpage
\bibliographystyle{plainnat}
\bibliography{reference}

\newpage
\appendix
\input{appendix}

\input{checklist}

\end{document}

%% file: notations.tex
\usepackage{amsmath,amssymb}
\usepackage[ruled]{algorithm2e}
\usepackage{algpseudocode}

\newcommand{\bgl}{{~\big |~}}
\renewcommand{\P}{\operatorname{\mathbb{P}}}

\newcommand{\bv}[1]{\mathbf{#1}}		

\newcommand{\vA}{{\bv{A}}}
\newcommand{\vB}{{\bv{B}}}

\newcommand{\vW}{{\bv{W}}}

\newcommand{\Gc}{{\mathcal{G}}}
\newcommand{\Kc}{{\mathcal{K}}}

\newcommand{\R}{\mathbb{R}}				

\usepackage{stackengine}
\usepackage{scalerel}

\newcommand{\onenorm}[1]{\left\|#1\right\|_{\ell_1}}

%% file: introduction.tex
Hidden Markov Models (HMMs) describe systems whose observations are governed by latent Markov chain dynamics, and arise broadly across the natural and behavioral sciences --- from animal decision-making to ecological and climate processes~\citep{glennie2023hidden,mcclintock2020uncovering,zucchini1991hidden}. Recent work has shown that pre-trained large language models (LLMs) can perform in-context learning (ICL) on Markov-structured sequences, approaching optimal predictors~\citep{edelman2024the,makkuva2025attention,rajaraman2024transformers}. The setting becomes considerably richer under HMMs, where the underlying Markov chain is unobserved and optimal prediction requires integrating noisy emissions to track a latent state. Given a single emission stream as a prompt, LLMs produce next-token predictions that approach the Bayes-optimal posterior as context grows, regularly outperforming classical inference algorithms applied to the same data~\citep{dai2026pretrained}.

How the model achieves this remains an open question. Several candidate mechanisms have been proposed for Transformer-based sequence modeling on Markovian and hidden-Markovian data, including spectral methods, belief-state tracking, and gradient-based feature learning~\citep{hu2024limitation,shai2024transformers,hao2025transformers}. Each is consistent with some subset of the observed empirical behavior, yet the field has not converged on which algorithm, if any, the LLM actually implements. Crucially, no proposal has been validated against the model's own internal activations.

We close this gap with a three-stage pipeline that connects in-context learning behavior to the underlying internal mechanisms on HMM prediction tasks. First, we systematically benchmark $12$ pre-trained LLMs against a suite of classical and learning-based HMM predictors across $75$ HMM configurations and a wide range of context lengths. No single baseline matches LLM behavior in every regime, but the comparison narrows the space of plausible mechanisms to three algorithm classes, which we call \emph{ansatzes}: Linear $n$-gram, Non-linear $n$-gram, and Spectral. Second, we ask whether Transformers are \emph{capable} of implementing these algorithms. We derive a unifying construction that expresses next-token probabilities as linear functions of nonlinear features built from past observations, and use it to give explicit Transformer implementations of all three classes; we verify that one such implementation -- a finite-window predictor -- arises naturally in a small Transformer trained on a single HMM. Third, we move from \emph{capability} to \emph{actual use}: we introduce the \emph{Principal Activations Probe} (PAP), a layer-wise probing-and-intervention pipeline that isolates low-dimensional algorithmic signals in the residual stream and tests whether they causally drive the model's predictions. PAP reveals that pre-trained LLMs recruit different algorithmic representations across different HMM regimes, and that distinct computational stages localize to different layers. PAP also shows that decodability and causal use must be tested separately: algorithmic information is present in earlier layers but bypassed by downstream computation.

\paragraph{Contributions.}
\begin{enumerate}
  \item We empirically compare in-context LLM predictions against classical HMM algorithms across a controlled set of HMM configurations, and show that no single algorithm class accounts for LLM behavior in all regimes, narrowing the space of plausible mechanisms to three candidate classes (Section~\ref{sec:icl_on_hmms}).
  \item We derive theoretical constructions showing that Transformers can implement each of the three candidate ansatzes in-context, unify them through a common nonlinear-feature representation, and validate the resulting predictions in small trained Transformers (Section~\ref{sec:theory}).
  \item We introduce the \emph{Principal Activations Probe} (PAP), a layer-wise probing and causal-intervention pipeline that isolates low-dimensional algorithmic signals in the residual stream, demonstrates their causal effect on predictions, and reveals a link between the model's internal strategy and the underlying HMM regime (Section~\ref{sec:llm_internal}).
\end{enumerate}

\begin{figure}
  \centering
  \vspace{-1em}
  \includegraphics[width=\textwidth]{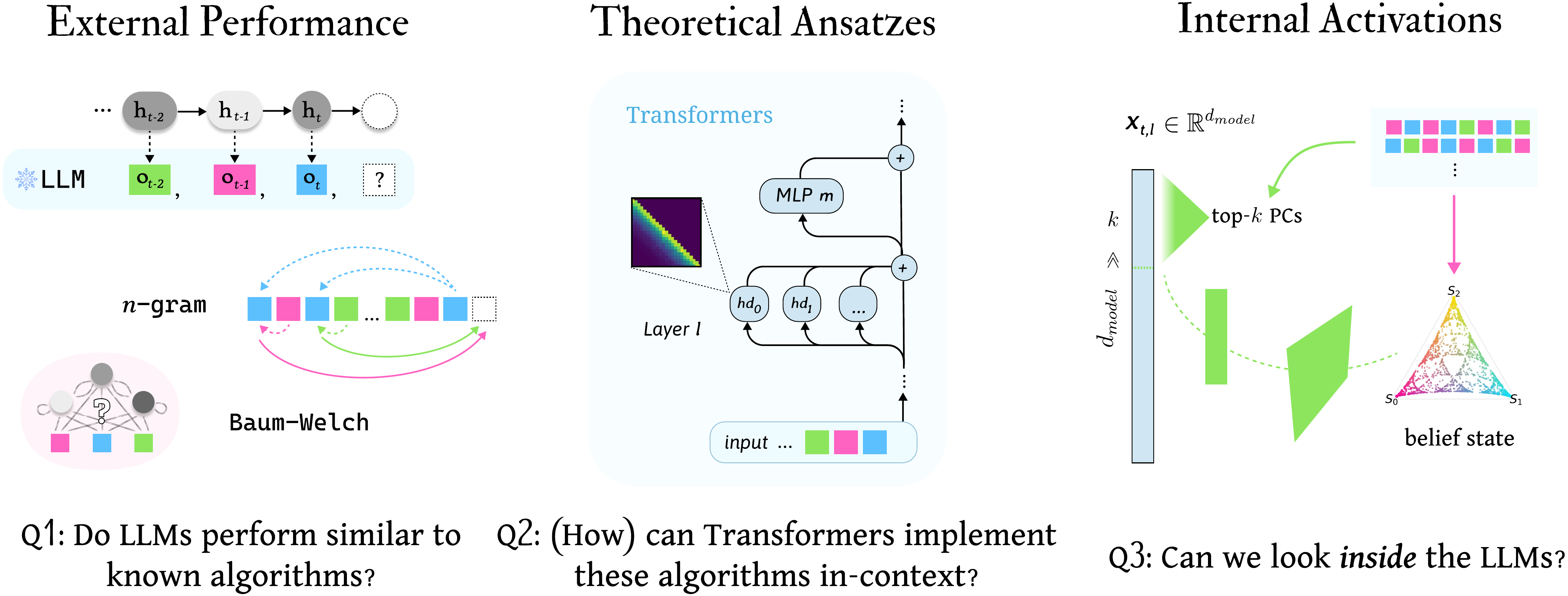}
  \caption{\textbf{Overview of our study.}}
  \label{fig:main_figure}
  \vspace{-1em}
\end{figure}

%% file: section_icl_on_hmms.tex
We study how pre-trained LLMs perform in-context learning on synthetic HMM sequences.
We first review key HMM properties (Section~\ref{sec:background}) and theoretical algorithms (Section~\ref{sec:alg_descriptions}).
We then empirically demonstrate that the performance of LLMs consistently converges to the theoretical optimum, with nuances when compared against theoretical algorithms (Section~\ref{sec:llm_emp}).

\subsection{HMM Background}
\label{sec:background}
\paragraph{Hidden Markov Model.}

At each time step $t$, an HMM's \emph{hidden state} $H_t \in\mathcal H$ emits an \emph{observation} $O_t$ from an \emph{emission distribution}, then transitions to $H_{t+1}$ according to \emph{transition probabilities}. Implicit in this description are three standard assumptions: the \emph{Markov property}, that each state depends only on its immediate predecessor; \emph{output independence}, that each observation depends only on the current hidden state; and \emph{stationarity}, that neither the transition nor the emission probabilities depend on $t$.

In the finite-alphabet setting, take $\mathcal{H} = \{1, \ldots, M\}$ and $\mathcal{O} = \{1, \ldots, N\}$. An HMM is specified by $\boldsymbol{\lambda} = (\boldsymbol{\pi}, \vA, \vB)$, where $\boldsymbol{\pi} \in \mathbb{R}^M$ is the initial state distribution, $\vA \in \mathbb{R}^{M \times M}$ the transition matrix, and $\vB \in \mathbb{R}^{M \times N}$ the emission matrix.

Under standard conditions, the chain converges to a unique \emph{stationary distribution} $\boldsymbol{\mu} = \boldsymbol{\mu}\vA$ \citep{1003838}, governing long-run behavior of the hidden state and central to prediction and parameter learning. The rate of convergence is the \emph{mixing rate}: the hidden-state distribution approaches $\boldsymbol{\mu}$ geometrically, with smaller rate meaning faster mixing. For finite-alphabet HMMs this rate equals $\lambda_2$, the second-largest eigenvalue of $\vA$; we restrict to transition matrices with real positive eigenvalues (Appendix~\ref{app:background}). 

To quantify where an HMM sits on the deterministic-to-random spectrum, we use normalized entropy
\vspace{-6pt}
\[
\mathcal{E}(\vA) = -\frac{1}{\log M}\sum_{i,j} \mu_i a_{ij} \log a_{ij}, \qquad \mathcal{E}(\vB, \boldsymbol{\mu}) = -\frac{1}{\log N}\sum_{j,k} \mu_j b_{jk} \log b_{jk},
\]
measuring stationary-distribution uncertainty in the next hidden state and in the emission; both lie in $[0,1]$, with smaller values indicating more predictable dynamics.

\begin{figure}
  \centering
  \includegraphics[width=\textwidth]{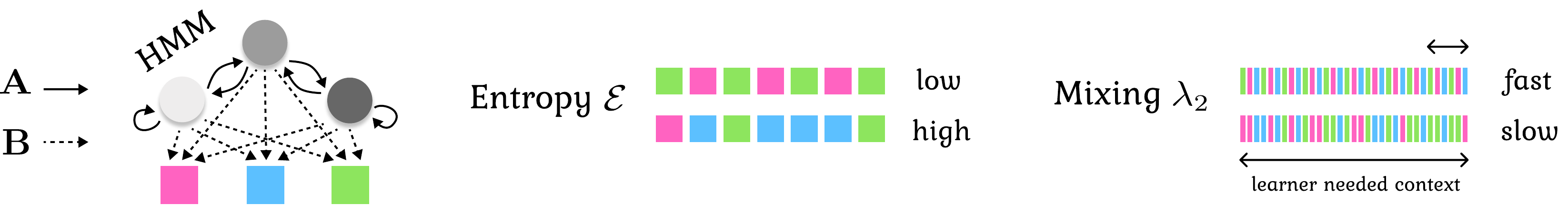}
  \caption{Properties of HMMs that impact the difficulty of next observation prediction task.}
  \label{fig:hmm_property}
  \vskip -0.5em
\end{figure}

\begin{wrapfigure}{r}{0.32\textwidth}
  \centering
  \vspace{-1.2em}
  \includegraphics[width=0.32\textwidth]{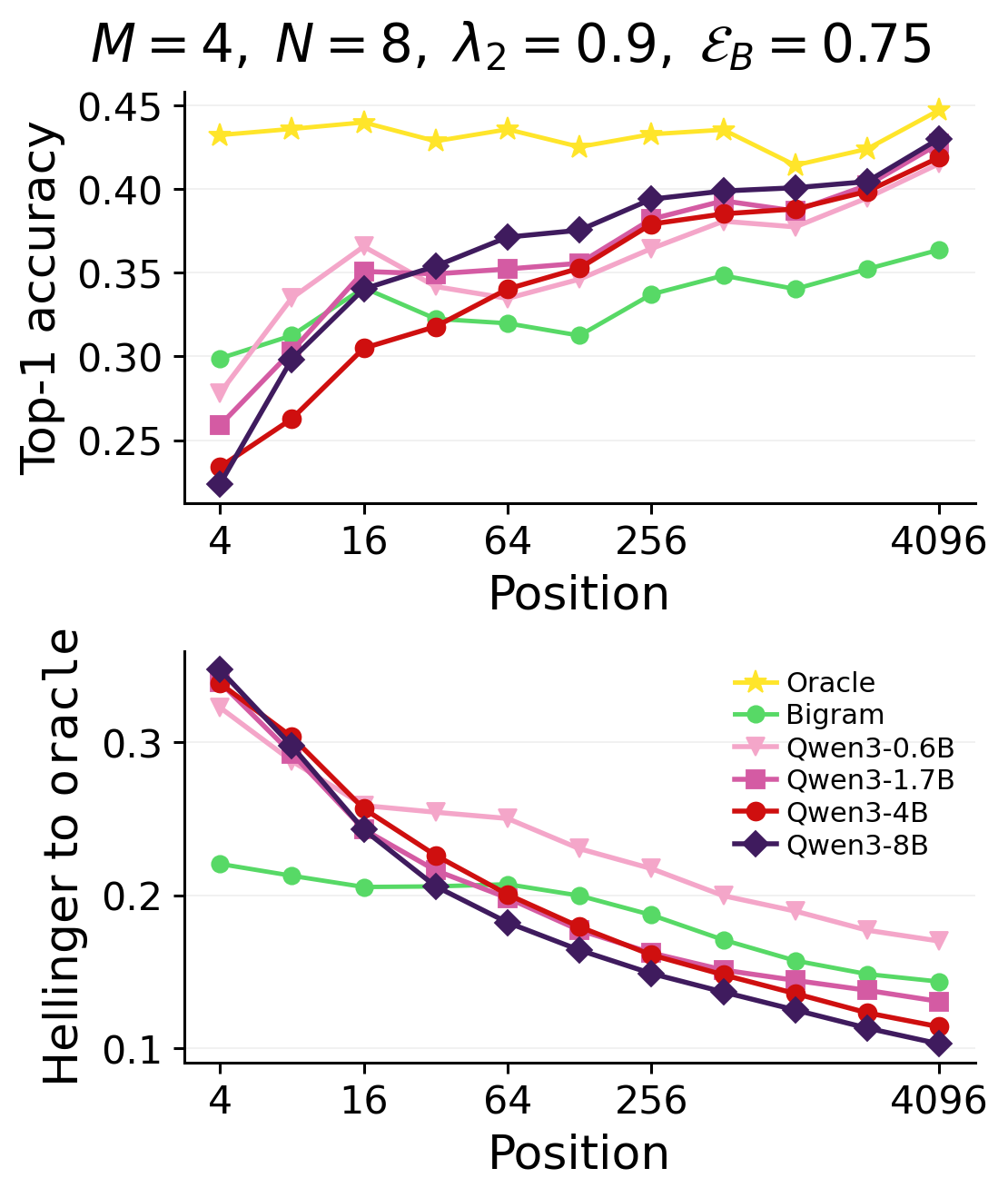}
  \caption{An example of LLMs converging to $\texttt{Oracle}$. \texttt{Bigram} is added for comparison, as it's the theoretical optimal learner for Markov sequences.} 
  \label{fig:llm_optimal}
  \vspace{-1.2em}
\end{wrapfigure}

\subsection{Theoretical Algorithms}
\label{sec:alg_descriptions}

\begin{table}
\centering
\small
\begin{tabular}{@{}ll@{}}
\toprule
Class & Description \\
\midrule
Oracle                 & Bayes-optimal posterior using the true HMM parameters (forward recursion) \\
Spectral                        & Observation operators from empirical low-order moments \\
Linear $n$-gram                 & Linear predictor on the last $n$ observations with one-hot features \\
Non-linear $n$-gram             & Linear predictor on the Kronecker feature of the last $n$ observations \\
Baum--Welch        & Expectation-Maximization (EM) for HMM parameter estimation \\
\bottomrule
\end{tabular}
\vspace{0.5em}
\caption{Algorithmic classes benchmarked for HMM next-observation prediction. Class variants and full implementation details are in Appendix~\ref{app:bench}.}
\label{tab:bench}
\vskip -0.7cm
\end{table}

We benchmark pre-trained LLMs against four classes of HMM predictors, with full specifications in Appendix~\ref{app:bench} and a summary in Table~\ref{tab:bench}. The \texttt{Oracle} serves as an upper bound, accessing the true parameters $(\boldsymbol{\pi}, \mathbf{A}, \mathbf{B})$ and returning the Bayes-optimal next-observation distribution via the forward recursion. \texttt{Spectral} methods estimate empirical uni-, bi-, and trigram statistics and combine them into observation operators~\citep{hsu2012spectral,ma2023bridging}; the two variants \texttt{Norm} and \texttt{SVD} differ in how the bigram statistic is inverted. \texttt{Linear $n$-gram} methods fit a linear map from the one-hot encoding of the last $n$ observations to the next-token distribution, with three variants: \texttt{Gradient CE} (OGD on cross-entropy), \texttt{Gradient MSE} (OGD on squared error), and \texttt{Ridge MSE} (ridge regression). \texttt{Non-linear $n$-gram} methods exploit Lemma~\ref{lemma:kron_prediction_main}, which shows that the Bayes-optimal distribution is linear in the Kronecker product of the last $n$ one-hot observations; we benchmark an explicit form \texttt{Kron} and a kernelized form \texttt{Kernel} that avoids the exponential-in-$n$ blowup in feature dimension. Finally, \texttt{BW(EM)}~\citep{baum1970maximization} represents a qualitatively different paradigm: rather than predicting directly from observations, it uses EM to iteratively recover the HMM parameters given the number of hidden states.

\subsection{Empirical Comparisons}
\label{sec:llm_emp}

\paragraph{Experimental Setup.}
For each HMM configuration we sample observation sequences $\mathbf{o}_{1:T}$ from $\boldsymbol{\lambda} = (\boldsymbol{\pi}, \vA, \vB)$ and evaluate next-observation prediction $o_{t+1}$ given $\mathbf{o}_{1:t}$.

We fix $M = 4$ and vary three controls across $75$ configurations: transition entropy $\mathcal{E}(\vA) \in \{0,\,0.25,\,0.5,\,0.75,\,1\}$ (inducing mixing rates $|\lambda_2(\vA)| \in \{1, 0.9, 0.75, 0.5, 0\}$); emission entropy $\mathcal{E}(\vB) \in \{0,\,0.25,\,0.5,\,0.75,\,1\}$ (from injective when $N \geq M$ to uniform); and alphabet size $N \in \{2, 4, 8\}$, where $N < M$ induces hidden-state aliasing. We set $\boldsymbol{\pi} = \boldsymbol{\mu}$. Per configuration, we sample $4{,}096$ sequences of length $4{,}097$ and evaluate at context lengths $\{4, 8, \ldots, 4096\}$.

We report two metrics, averaged over $4{,}096$ samples per HMM. \emph{Accuracy} is the fraction of positions whose argmax prediction matches the truth. The Hellinger distance between predicted next-token distributions plays two roles: \emph{Hellinger-to-oracle} measures distance to the Bayes-optimal posterior; \emph{Hellinger-to-baseline} measures algorithmic similarity to each baseline.

We evaluate $12$ open-weight pre-trained LLMs from the Qwen, Llama, Gemma, and OLMo families ($0.6$B--$8$B parameters), without any fine-tuning. Each model sees raw observation tokens via a fixed alphabet-to-vocabulary map; at every position, we restrict the next-token logits to the observation vocabulary and renormalize via softmax.

\begin{wrapfigure}{l}{0.7\textwidth}
  \centering
  \vspace{-1.2em}
  \includegraphics[width=0.7\textwidth]{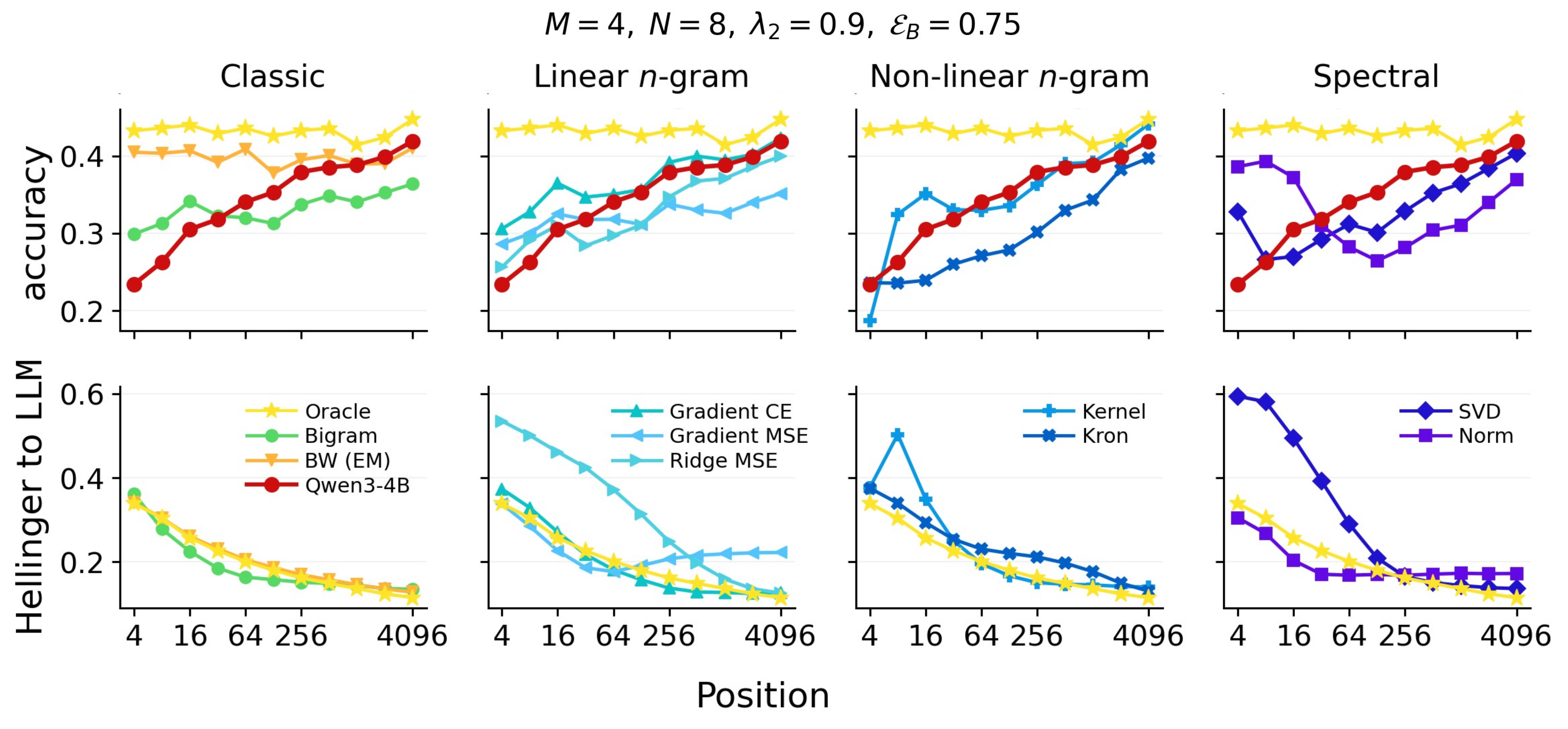}
  \caption{An example HMM setting, showing how LLMs empirically compare with other algorithms. Accuracy $\uparrow$, Hellinger $\downarrow$.}
  \label{fig:llm_empirical}
  \vspace{-1.8em}
\end{wrapfigure}

\paragraph{LLMs converge to Bayes-optimal predictor~\citep{dai2026pretrained}.} Across HMM configurations, LLM predictions approach \texttt{Oracle} as context length grows, and this trend is consistent across model sizes (Figure~\ref{fig:llm_optimal}). Appendix \ref{app:llm_emp_perf} provides analysis of more model families and HMM configurations, including a discussion of cases where convergence fails or unexpected behavior arises.

\paragraph{No single algorithm explains LLM behavior across all regimes.} When compared against the baselines, no method consistently matches LLM predictions across all HMM configurations and context lengths (Figure~\ref{fig:llm_empirical}). Among the baselines, \texttt{BW(EM)} (Baum--Welch) is the one method that can be confidently ruled out: its strong performance at short context lengths reflects a fundamentally different learning mechanism, namely explicit parameter estimation, rather than the gradual in-context generalization characteristic of LLMs. We therefore exclude it from the subsequent theoretical analysis. See Appendix~\ref{app:baseline_compare} for detailed comparisons across all baseline classes.

%% file: section_construction_ansatzes.tex


In this section, we show that Transformers can implement these ansatzes (described in Section \ref{sec:alg_descriptions}) to predict the probability vector $q_{t+1} := \P\left(O_{t+1} \bgl O_{1:t} = o_1, o_2, \dots, o_t\right)$. We begin with the setting in which the HMM parameters $(\boldsymbol{\pi}, \vA, \vB)$ are known, and show that $q_{t+1}$ can be expressed in terms of novel non-linear features constructed from the one-hot vectors of past observations.

\begin{lemma}\label{lemma:kron_prediction_main}
Consider an HMM specified by $\boldsymbol{\lambda} = (\boldsymbol{\pi}, \vA, \vB)$. Let $\vA_k := \operatorname{diag} \left( \vB[:,k] \right)\vA^\top$ for all $k\in[N]$, where $\vB[:,k]$ denotes the $k$-th column of the emission matrix $\vB$. Let 
\vspace{-6pt}
\begin{align}
\phi_t^{(n)}
:=
u_t\otimes u_{t-1}\otimes \cdots \otimes u_{t-n+1}\in\mathbb{R}^{N^n}, \quad \psi_t^{(n)} := [u_{t}^\top~~u_{t-1}^\top~~\cdots~~u_{t-n+1}^\top]^\top \in \R^{nN}\label{eqn:kron_features_main}
\end{align}
denote a nonlinear and the corresponding linear feature vector, constructed from $n$ past observations, for $u_t\in \mathbb R^N$ one-hot vector encoding $o_t$.
Then, there exists a matrix $\Gc_{t,n}\in \mathbb{R}^{N\times N^n}$ such that 
$q_{t+1} = \Gc_{t,n}\,\phi_t^{(n)}$. Moreover, there exists matrices $\vW_{lin} \in \mathbb{R}^{nN \times N^n}$, $\vW_{kron} \in \mathbb{R}^{N^n \times nN}$, and a vector $v \in \R^{N^n}$ such that, $\psi_t^{(n)} = \vW_{lin} \phi_t^{(n)}$, and $\phi_t^{(n)} = \mathrm{ReLU}( \vW_{kron} \psi_t^{(n)} + v)$. 
    where $\mathrm{ReLU}(x):= x^+ = \max(0,x)$, and is applied entry-wise to vectors.
\end{lemma}
The proof of Lemma~\ref{lemma:kron_prediction_main} is deferred to the Appendix~\ref{app:$n$-gram_prediction}. Lemma~\ref{lemma:kron_prediction_main} implies that, when $(\boldsymbol{\pi}, \vA, \vB)$ are unknown, the prediction problem reduces to learning a linear operator. Concretely, we can treat the observed sequence $\{o_t\}_{t \geq n}$ as training data and minimize the cross-entropy loss of predicting $o_{t+1}$ from either the nonlinear feature $\phi_t^{(n)}$ or the linear feature $\psi_t^{(n)}$. Since both mappings from feature to prediction are linear, this amounts to fitting a single linear matrix via standard regression or online gradient descent. Our next result shows that Transformers can learn the features $\phi_t^{(n)}$ and $\psi_t^{(n)}$ from raw observation sequences (one-hot), and can emulate GD algorithm to learn a linear map from $\phi_t^{(n)}$ and $\psi_t^{(n)}$ to the probability distribution $q_{t+1}$.
\begin{theorem}\label{thrm_CE_GD_main}
Consider the same setup of Lemma~\ref{lemma:kron_prediction_main}. Furthermore, consider a dataset
    \begin{align}
        \mathcal{D}_m=\{(\mathfrak{h}_t,y_t)\}_{t=n}^{n+m-1},\qquad
        \mathfrak{h}_t=(u_{t},u_{t-1},\ldots,u_{t-n+1}),
        \qquad y_t\in \{e_1,\ldots,e_N\}. \label{eqn:dataset}
    \end{align}
     Let $\vW\in \mathbb{R}^{N\times N^n}$ be the parameter of multinomial logistic regression, $q_\vW(\cdot\mid \phi)=\operatorname{softmax}(\vW\phi)$, 
    trained by full-batch gradient descent on the empirical cross-entropy loss
    \vspace{-6pt}
    \[
        \mathcal{L}_m(\vW)
        :=
        \frac{1}{m}\sum_{t=n}^{n+m-1}
        \ell(\vW\phi_{t}^{(n)},y_t),
        \qquad
        \ell(z,y):=-y^\top \log \operatorname{softmax}(z),
    \]
    If the number of GD iterations satisfies $T=\mathcal{O}(\log(n))$, then there exists a
    Transformer of depth $\mathcal{O}(\log(n))$ that emulates $T$ steps of full-batch GD on the CE loss of the logistic-regression model $q_\vW(\cdot\mid \phi)=\operatorname{softmax}(\vW\phi)$.
\end{theorem}

The proof of Theorem~\ref{thrm_CE_GD_main} is deferred to Appendix~\ref{app:$n$-gram_prediction}.
Our next result shows that a sufficiently deep Transformer can also emulate the spectral learning
algorithm in-context, with proof deferred to Appendix~\ref{app:spectral}.

\begin{remark}\label{lemma:spectral}
Consider the setup of Lemma~\ref{lemma:kron_prediction_main} and the dataset in
\eqref{eqn:dataset}. The spectral learning algorithm (Appendix~\ref{app:bench:spectral})
estimates low-order moment statistics $P_1, P_2, P_3$ from the observation prefix, constructs
observation operators $B_o := P_2^{-1}P_3(o)$ for each $o \in \mathcal{O}$, and predicts
$O_{t+1}$ by recursively updating a belief state via $b_t \propto B_{o_t} b_{t-1}$.
By combining several existing results (see Appendix~\ref{app:spectral}), there exists a
Transformer of depth $\mathcal{O}(\log(n+m))$ that emulates this algorithm in-context.
\end{remark}


%% file: section_alg_trained_transformer.tex






When a Transformer is trained directly on sequences from a single fixed HMM, the transition and emission structure is shared across all training sequences, so the model need not learn a general-purpose in-context inference procedure. This adjacent and arguably easier task (the model can memorize the HMM parameters rather than learning it in-context) lets us probe what type of algorithm the Transformer architecture naturally converges to.

\begin{wrapfigure}{l}{0.28\textwidth}
    \vspace{-1em}
    \centering
    \includegraphics[width=0.28\textwidth]{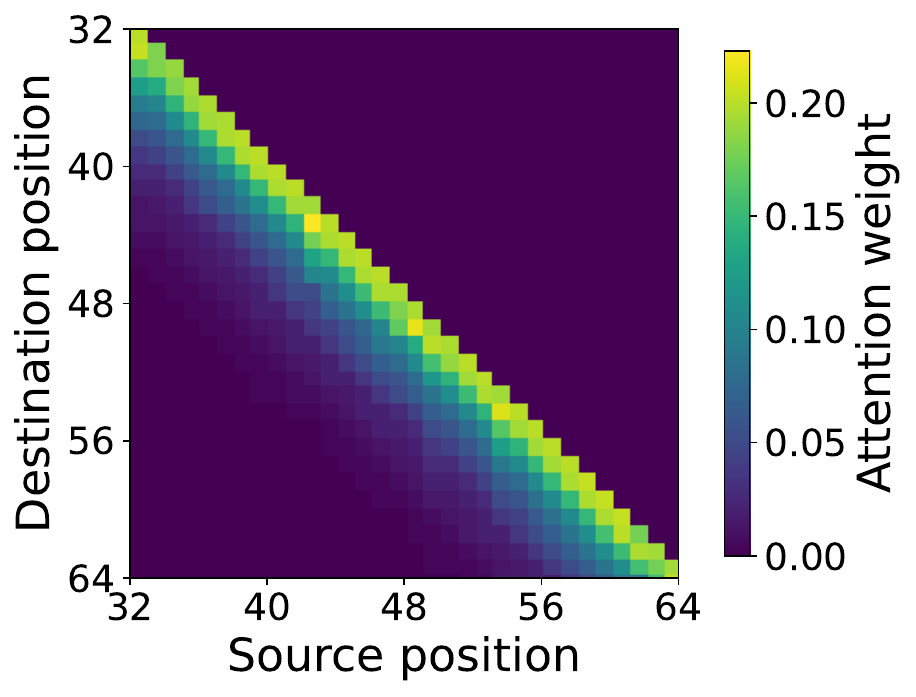}
    \vspace{-1.5em}
    \caption{Attention pattern for a small Transformer.}
    \label{fig:attention_pattern}
    \vspace{-1em}
\end{wrapfigure}

We first identify the smallest Transformer that reliably solves the task by ablating over layers, heads, model dimension, and feedforward dimension (Appendix~\ref{app:1_hmm_transformer}), following the trained-HMM setup of~\citet{hu2024limitation}. For a $3$-state HMM with $M = N = 3$, $\mathcal{E}(\mathbf{A}) = 0.25$, and $\mathcal{E}(\mathbf{B}) = 0.75$, a single layer with one attention head, model dimension $64$, and feedforward dimension $256$ proves sufficient. This is notably more efficient than the $\mathcal{O}(\log(n + m))$ layers required for a Transformer to implement the Spectral algorithm from theoretical operators without in-context learning~\citep{hu2024limitation}.


Figure~\ref{fig:attention_pattern} visualizes the learned attention pattern. The attention mass is concentrated on a small number of recent tokens rather than being spread over the full history, which is consistent with a finite-window prediction mechanism --- namely, the linear and non-linear $n$-gram classes of algorithms.

Lemma~\ref{lemma:kron_prediction_main} implies that the optimal logits are given by a linear map applied to the finite-window feature $z_t := \psi_t^{(n)}$ or $\phi_t^{(n)}$. This motivates a two-part investigation of the trained Transformer's internal computation. First, we ask whether the Transformer hidden state $x_t \in \mathbb{R}^{d_{\mathrm{model}}}$ linearly encodes $z_t$, by fitting a linear probe $R \in \mathbb{R}^{Nn \times d_{\mathrm{model}}}$ such that $Rx_t \approx z_t$. Second, we ask whether the model's learned hidden-to-logit map $W \in \mathbb{R}^{N \times d_{\mathrm{model}}}$ is consistent with the theoretically predicted one: if the Transformer hidden state encodes $z_t$, then the composition $W_{\mathrm{th}} R$, where $W_{\mathrm{th}} \in \mathbb{R}^{N \times Nn}$ is the optimal linear predictor from Lemma~\ref{lemma:kron_prediction_main}, should closely approximate $W$.


\begin{wrapfigure}{r}{0.26\textwidth}
    \vspace{-1em}
    \centering
    \includegraphics[width=0.26\textwidth]{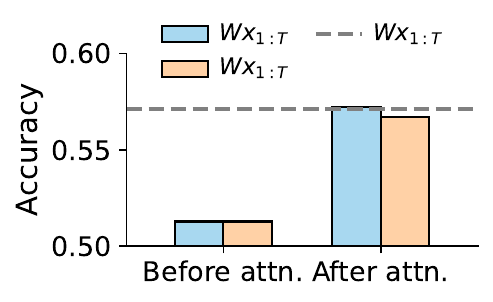}
    \vspace{-1.5em}
    \caption{Prediction accuracy using hidden states before and after attention.}
    \label{fig:attention_probe_accuracy}
    \vspace{-1em}
\end{wrapfigure}

\textbf{Attention constructs a linearly decodable finite-window representation.}
We probe using $z_t = \psi_t^{(n)}$, the linear $n$-gram feature; results for $\phi_t^{(n)}$ are in Appendix~\ref{app:feature_compare}. We test whether the attention layer constructs a representation of recent observations by comparing linear probes applied to hidden states before and after attention. As shown in Figure~\ref{fig:attention_probe_accuracy}, before attention, both $Wx_{1:T}$ and $W_{\mathrm{th}}Rx_{1:T}$ achieve similarly low accuracy. After attention, both improve substantially and closely match the accuracy of the theoretical finite-window predictor $W_{\mathrm{th}}z_{1:T}$. This supports two conclusions: (1) the attention layer aggregates recent observation tokens into a linearly decodable $n$-gram feature in the hidden state; and (2) the learned hidden-to-logit map $W$ is well-approximated by the composition $W_{\mathrm{th}}R$, meaning the output projection effectively applies the theoretically predicted linear map to the extracted feature.

\begin{figure}[h]  
    \centering  
    \includegraphics[width=\linewidth]{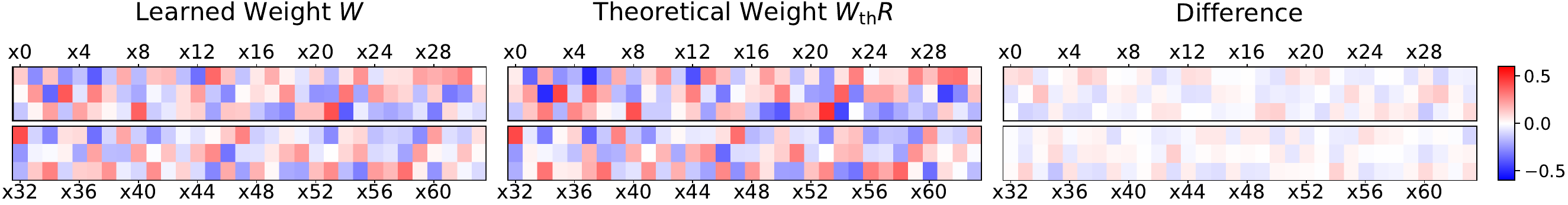}  
    \vspace{-1em}
    \caption{Comparison between the trained Transformer's learned output map and the theoretical finite-window predictor ($n=8$). \textit{Left:} the learned hidden-to-logit map $W \in \mathbb{R}^{3 \times 64}$. \textit{Middle:} the theoretical predictor $RW_{\mathrm{th}}^\top$, where $W_{\mathrm{th}}$ is the CE-optimal weight on the length-$8$ observation window and $R$ is the linear probe from hidden states to $n$-gram features. \textit{Right:} their difference, which is near zero, indicating that the learned readout closely recovers the theoretical finite-window predictor.}
    \label{fig:weight_visualization}
\end{figure}

\textbf{The trained Transformer recovers the theoretical finite-window predictor almost exactly.}
Figure~\ref{fig:weight_visualization} directly compares the learned and theoretical weights for window size $n=8$, using hidden states after attention. The left panel shows the learned weight $W$, the middle panel shows the theoretical weight $RW_{\mathrm{th}}^\top$, and the right panel shows their difference, which is small relative to the scale of the weights. Together with the probing results above, this alignment supports a coherent two-step picture of the Transformer's computation: attention constructs a linearly decodable finite-window representation of recent observations, and the output projection applies an approximately optimal linear predictor to that representation. Additional ablations over HMM configurations and window sizes are provided in Appendices~\ref{app:hmm_config} and~\ref{app:window_size}; notably, the gap between learned and theoretical weights decreases monotonically with $n$ and converges around $n = 6$.

%% file: section_llm_internal.tex
\newcommand{\alg}{PAP}


Section~\ref{sec:construction_ansatz} shows each ansatz \emph{can} be implemented in a Transformer, and our small-Transformer experiments confirm Linear $n$-gram implementations arise in practice. We now analyze pre-trained LLM activations and ask which algorithm best explains their in-context learning predictions.

\subsection{Methods}

Transformers process information through a sequence of layers that read from and write to a shared residual stream, with the activation at layer $\ell$ aggregating all computation up to that depth~\citep{elhage2021mathematical}. We adopt this layer-wise view as our unit of analysis: rather than localizing computation to specific heads or neurons, we ask what each layer computes. Linear probing is a natural fit, since task-relevant quantities are often encoded linearly in the residual stream~\citep{alain2017understanding,jiang2024on}. But probing the full stream tends to detect the presence of a quantity rather than its computational role. To restrict the probe to features the model uses, we project activations onto their top principal components~\citep{nanda2023progress,piotrowski2025constrained}.

\paragraph{Principal Activations Probe (\alg{}).}
Let $\mathbf{o}_{1:T} = (o_1, \ldots, o_T)$ be a sequence of observations. We map each observation to an abstract token, feed the sequence into the LLM, and record the residual stream activation $x_{\ell, t} \in \mathbb{R}^{d_{\text{model}}}$ at every layer $\ell \in \{1, \ldots, L\}$ and position $t \in \{1, \ldots, T\}$. In parallel, we compute a target \emph{algorithmic representation} $r_t \in \mathbb{R}^{d_r}$, a sequence- and position-specific quantity produced by some candidate algorithm applied to $\mathbf{o}_{1:t}$; concrete choices for $r_t$ are deferred to Section~\ref{sec:probe_alg_representations}.

\alg{} requires that the train and test sequences be sampled from a single distribution. Since $x_{\ell, t}$ and $r_t$ are both deterministic functions of $\mathbf{o}_{1:t}$, this ensures the activation--representation pairs $(x_{\ell, t}, r_t)$ are identically distributed across the two splits, so that test performance reflects whether $r_t$ is decodable from $x_{\ell, t}$ rather than distribution shift between them. \alg{} then fits two stages independently per layer on the training split. First, PCA on centered activations yields a top-$k$ orthonormal basis $U_\ell \in \mathbb{R}^{d_{\text{model}} \times k}$ and mean $\bar{x}_\ell$, with $k \ll d_{\text{model}}$. Second, ridge regression maps the $k$ leading PCA coordinates to $r_t$, yielding $M_\ell \in \mathbb{R}^{d_r \times k}$ and $c_\ell \in \mathbb{R}^{d_r}$.

The probe's prediction is
\begin{equation}
\hat{r}_{\ell, t} \;=\; M_\ell U_\ell^{\!\top} (x_{\ell, t} - \bar{x}_\ell) + c_\ell.
\label{eq:pap_probe}
\end{equation}
The rank $k$ is a hyperparameter we vary in our experiments.
 
\paragraph{Probe evaluation metrics.}
We report two complementary metrics: mean squared error, $\mathbb{E}\|\hat{r}_{\ell, t} - r_t\|^2$, which measures absolute reconstruction error, and the coefficient of determination $R^2 = 1 - \text{MSE}/\text{Var}(r_t)$, which normalizes by the target's variance. The two are needed together because $R^2$ degenerates when $r_t$ is nearly constant at a given time step --- a perfect probe can yield low or undefined $R^2$ despite near-zero MSE. Both metrics are computed and reported on the test split.

\paragraph{Causal Intervention with \alg{}.}
A successful probe shows that $r_t$ is \emph{decodable}; it does not show that the model \emph{uses} this representation downstream~\citep{belinkov2022probing}. To test causal use, we perform activation patching at layer $\ell$ and position $t$~\citep{zhang2023towards}. Given a source sequence $\mathbf{o}_{1:t}^{\mathrm{src}}$ and a target sequence $\mathbf{o}_{1:t}^{\mathrm{tgt}}$, we run the model on the source sequence while intervening on the activation at $(\ell,t)$ using information derived from a forward pass on the target sequence. If the intervened representation is part of the computation that determines the next-token prediction, then the patched output should shift toward the prediction implied by $r_t^{\mathrm{tgt}}$ rather than remain aligned with $r_t^{\mathrm{src}}$.
 
We use three patches of increasing specificity. The \textbf{full residual patch} substitutes the residual vector at layer $\ell$ and position $t$ with its counterpart from the target run,
$x_{\ell,t}^{\mathrm{src}} \leftarrow x_{\ell,t}^{\mathrm{tgt}}$.
While this intervention generally affects downstream predictions, it provides evidence for causal use only if the resulting prediction shifts toward that implied by $r_t^{\mathrm{tgt}}$ rather than remaining aligned with $r_t^{\mathrm{src}}$. We therefore treat it as the least restrictive intervention at $(\ell,t)$, against which more targeted patches can be compared. The \textbf{PCA-subspace patch} replaces only the top-$k$ subspace,
\(
x_{\ell, t}^{\text{src}} \leftarrow U_\ell U_\ell^{\!\top} \big(x_{\ell, t}^{\text{tgt}} - \bar{x}_\ell\big) + \big(\mathbf{I} - U_\ell U_\ell^{\!\top}\big)\big(x_{\ell, t}^{\text{src}} - \bar{x}_\ell\big) + \bar{x}_\ell,
\)
testing whether the subspace \alg{} flags as informative is also causally sufficient. The \textbf{probe-inverse patch} applies the smallest perturbation $\Delta x$ that makes the probe read out $r_t^{\text{tgt}}$; the minimum-norm solution is
\begin{equation}
\Delta x = U_\ell\, M_\ell^{+} \big(r_t^{\text{tgt}} - \hat{r}_{\ell, t}^{\text{src}}\big),
\label{eq:pap_inverse_patch}
\end{equation}
where $M_\ell^{+}$ is the right pseudoinverse. By construction, $\Delta x$ lies in the row space of $M_\ell U_\ell^{\!\top}$ --- the directions the probe actually reads from --- giving the strictest test of causal relevance.

\paragraph{Intervention metrics.}
Both metrics are computed against the expected next-token distribution $p_{t+1}^{\text{tgt}}$ under $r_t^{\text{tgt}}$. \emph{KL divergence} $\mathrm{KL}(p_{t+1}^{\text{tgt}} \,\|\, y_{t+1}^{\text{patched}})$ captures distributional agreement and is sensitive to diffuse predictions; \emph{interchange intervention accuracy} (IIA)~\citep{geiger2021causal,pmlr-v236-geiger24a} reports the fraction of pairs on which $\arg\max y_{t+1}^{\text{patched}} = \arg\max p_{t+1}^{\text{tgt}}$, reflecting whether the patch shifts the model's discrete prediction.
 
The components above: layer-wise linear probing~\citep{alain2017understanding}, PCA projection, and activation patching~\citep{zhang2023towards,geiger2021causal} are individually standard, and probing internal representations against postulated hidden states has been used extensively in the world-model literature. Our contribution lies in probing on candidate algorithmic representations drawn from the predictor families of Section~\ref{sec:construction_ansatz}. This lets us ask not just whether a hidden quantity is encoded, but which algorithm's intermediate state best explains and causally drives the model's predictions.

\subsection{Algorithmic Representations in LLMs on HMM Tasks}
\label{sec:probe_alg_representations}

\begin{wraptable}{r}{0.46\textwidth}
\vspace{-1em}
\centering
\footnotesize
\caption{The representations to be probed for each ansatz algorithm class.}
\label{tab:alg_representations}
\begin{tabular}{lcc}
\toprule
\text{Algorithm } 
& Belief $r^{state}_t$
& Operator $r^{op}_t$ \\
\midrule
\texttt{Oracle} & $\mathbb P(h_t|\mathbf{o}_{1:t},\boldsymbol{\lambda})$ & $-$ \\
\texttt{Soft $n$-gram}  & $W_t z_t$ & $W_t$ \\
\texttt{Spectral} & $b_t$ & $B_t$ \\
\bottomrule
\end{tabular}
\vspace{-1em}
\end{wraptable}

\paragraph{Representations for HMM baseline algorithms.}


\begin{figure}
  \centering
  \vspace{-1em}
  \includegraphics[width=\textwidth]{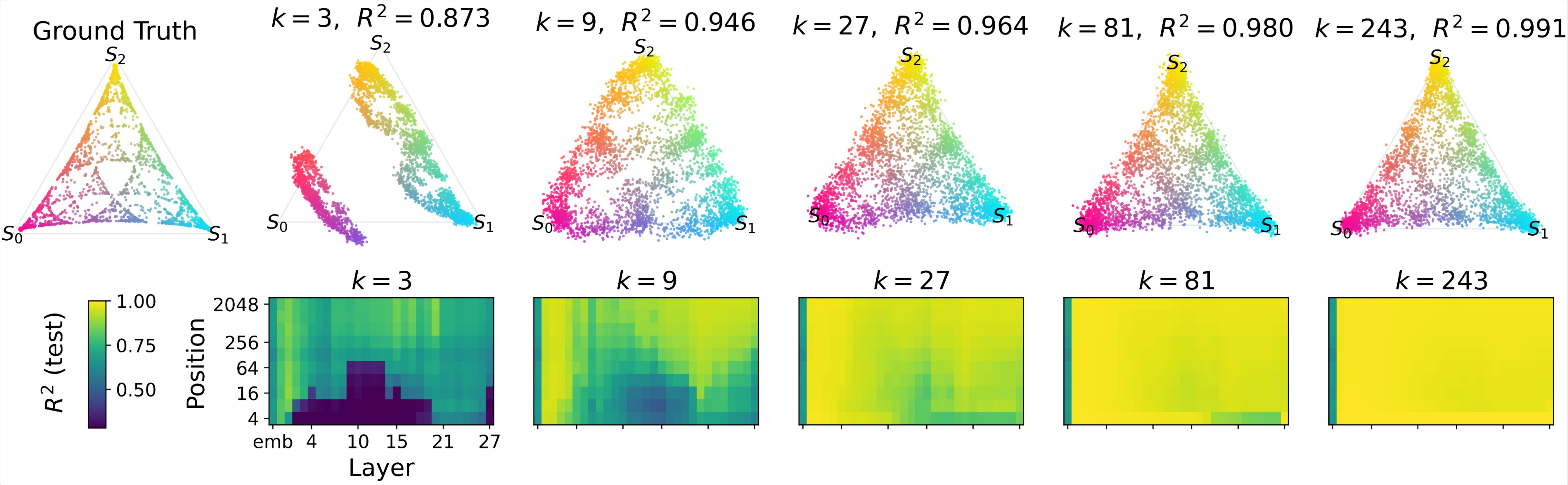}
  \caption{Quality of the belief-geometry mapping for an HMM with $M=3$, $N=3$, $\mathcal{E}(\vA)=0.5$ ($\lambda_2=0.75$), and $\mathcal{E}(\vB)=0.75$. \emph{(top)} Probe reconstruction of the belief simplex at Layer~20, position 2048, for varying PCA rank $k$. \emph{(bottom)} Probe quality measured by $R^2$ on the test set, with each subplot corresponding to a different PCA rank $k$.}
  \label{fig:quality_geometry}
  \vspace{-1.5em}
\end{figure}

Sections~\ref{sec:llm_emp} and~\ref{sec:construction_ansatz} narrow the ansatzes to three classes: \texttt{Linear $n$-gram}, \texttt{Non-Linear $n$-gram}, and \texttt{Spectral}. To probe which computation the Transformer implements, we associate each class with two representations of the observation history $\mathbf{o}_{1:t}$ (Table~\ref{tab:alg_representations}); the two $n$-gram variants share representational structure and are grouped as \texttt{Soft $n$-gram}.

The two representation types capture different aspects of the algorithm's behavior over the observation sequence. The \emph{belief} $r^{state}_t$ tracks the algorithm's current estimate of the hidden state and shifts at every position as new observations arrive. The \emph{operator} $r^{op}_t$, by contrast, encodes the estimated parameters underlying the prediction rule; it is updated incrementally but converges toward the ground-truth values as the sequence length grows. The belief at any position thus reflects both the accumulated observations and the current accuracy of the operator: as the operator converges, the resulting belief becomes more precise. Notably, the belief is meaningful even in the single-HMM setting where the parameters are known and no in-context learning is required, whereas the operator is specific to the ICL setting where the algorithm must estimate the underlying HMM parameters from the observed sequence. 

For the \texttt{Oracle}, the belief is the Bayes-optimal posterior $\mathbb{P}(h_t \mid \mathbf{o}_{1:t}, \boldsymbol{\lambda})$, and no operator is needed since the parameters are known. For \texttt{Soft $n$-gram}, the belief is the predicted logit $W_t z_t$, where $z_t \in \{\psi_t^{(n)}, \phi_t^{(n)}\}$ is the finite-window feature from Lemma~\ref{lemma:kron_prediction_main} and $W_t \in \mathbb{R}^{N \times \dim(z_t)}$ is the in-context learned weight; $W_t$ itself serves as the operator. For \texttt{Spectral} methods~\citep{hsu2012spectral}, the belief is the observation-operator state $b_t \in \mathbb{R}^M$, updated recursively as $b_t \propto B_{o_t} b_{t-1}$, where $B_x \in \mathbb{R}^{M \times M}$ is the observation operator for symbol $x$ estimated from empirical bigram and trigram statistics (see Appendix~\ref{app:bench:spectral}); the operator representation is the estimated $B_t := \{B_x\}_{x \in \mathcal{O}}$.

\vspace{-0.7em}
\paragraph{Experimental Setup.} We use Qwen-3-1.7B (28 layers, $d_{\mathrm{model}}=2048$) and reuse the 75 HMM configurations defined in Section \ref{sec:llm_emp}. Per HMM $\boldsymbol{\lambda} = (\boldsymbol{\pi}, \vA, \vB)$, we sample 4,096 sequences of length 2,049 with an 80/20 train-test split by sequence at random, and at $t\in\{4,8,\ldots,2048\}$ compute $r_t$ for each algorithm and record the residuals $x_{\ell,t}\in\mathbb R^{2048}$. We sweep PCA rank $k\in\{M, M^2, M^3, \ldots\}$, fit the probe by ridge regression on train, and evaluate on test. For causal intervention, we sample 1,024 pairs of observation histories $(\mathbf{o}_{1:t}^{\mathrm{src}},\mathbf{o}_{1:t}^{\mathrm{tgt}})$ whose algorithmic states at position $t$ disagree in their most probable value, $\arg\max r^{src}_t\neq \arg\max r^{tgt}_t$. We run the model on the source history $\mathbf{o}_{1:t}^{\mathrm{src}}$ and intervene on the representation at $(\ell,t)$ using information derived from the target history, then measure whether the resulting next-token prediction shifts toward that implied by $r_t^{\mathrm{tgt}}$.

\subsection{Results}

\begin{figure}
  \centering
  \vspace{-1em}
  \includegraphics[width=\textwidth]{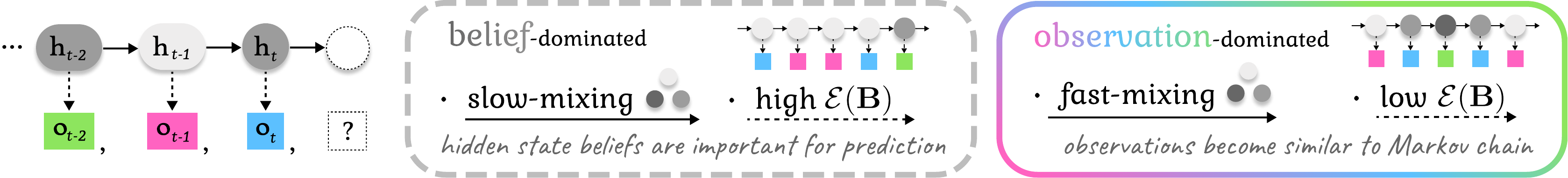}
  \caption{Two HMM regimes: \textit{belief}-dominated (slow mixing, high $\mathcal{E}(\mathbf{B})$), where hidden state beliefs drive prediction, and \textit{observation}-dominated (fast mixing, low $\mathcal{E}(\mathbf{B})$), where observations behave like a Markov chain.}
  \label{fig:belief_observation}
  \vspace{-1.2em}
\end{figure}

\paragraph{The ground-truth belief is linearly decodable from the residual stream.}
Figure~\ref{fig:quality_geometry} probes the \texttt{Oracle} belief on a 3-state 
HMM. At low rank ($k=9$), $R^2$ is high at early (0--3) and late (20--26) layers, 
with longer sequences shifting the high-quality region toward earlier layers. At 
sufficient rank ($k=243$), the belief is near-perfectly decodable from layer~1 
onward. Following~\citet{shai2024transformers}, we visualize the recovered belief 
simplex alongside $R^2$: the same fractal-like structure emerges in our pre-trained 
LLM despite differences in scale, training data, and objective---though the simplex 
remains visibly noisy even at high $R^2$, indicating the probe captures the dominant 
geometry but not every fine-grained detail.

\begin{wrapfigure}{r}{0.68\textwidth}
  \centering
  \vspace{-1em}
  \includegraphics[width=0.68\textwidth]{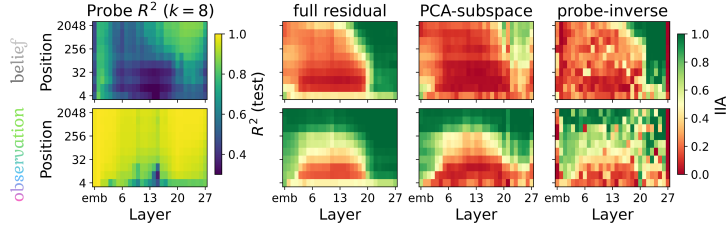}
  \caption{\texttt{Oracle} belief probing ($R^2$, left) and causal patching (IIA, right) across layers and positions ($k{=}8$), for the belief-dominated (top) and observation-dominated (bottom) regimes. Three patches of increasing specificity: full residual, PCA-subspace, probe-inverse (Eq.~\ref{eq:pap_inverse_patch}). $R^2{\uparrow}$, IIA${\uparrow}$.}
  \label{fig:h_vs_m_gt_belief}
  \vspace{-1em}
\end{wrapfigure}

We compare two HMM regimes ($M=N=4$). The \emph{belief-dominated} regime has slow 
mixing and noisy emissions ($\mathcal{E}(\vA)=0.25$, $\lambda_2=0.9$, 
$\mathcal{E}(\vB)=0.75$), so the optimal predictor must integrate observations over 
long contexts to maintain a meaningful belief over hidden states. The 
\emph{observation-dominated} regime has fast mixing and low-entropy emissions 
($\mathcal{E}(\vA)=0.75$, $\lambda_2=0.5$, $\mathcal{E}(\vB)=0.25$), so the most 
recent observation nearly identifies the hidden state and a bigram suffices for 
near-optimal prediction.

\vspace{-0.3em}
\paragraph{Decodability does not imply causal use.} 
Figure~\ref{fig:h_vs_m_gt_belief} shows results at PCA rank $k{=}8$. In both 
regimes, all three patching methods---full residual, PCA-subspace, and 
probe-inverse---successfully steer the model's output, confirming that the probed 
subspace is causally sufficient. The regimes diverge in one qualitative respect: at 
early layers (0--2) of the belief-dominated regime, probe $R^2$ is already high yet 
patching fails, even with a full residual patch. This early-layer divergence 
illustrates why decodability and causal use must be tested 
separately~\citep{belinkov2022probing}.

\vspace{-0.3em}
\paragraph{The model recruits different representations across regimes.}
Under the same comparison (Figure~\ref{fig:alg_reps}), we investigate all algorithmic 
belief representations in Table~\ref{tab:alg_representations}, and additionally include 
\texttt{Bigram} which is the Bayes-optimal predictor for simple Markov chains and is well understood as a baseline. 
\texttt{Bigram} is highly decodable in both regimes at $k{=}8$. Yet causal interventions tell the opposite story: \texttt{Bigram} patching has essentially no effect in the 
belief-dominated regime, while in the observation-dominated regime it steers the output nearly as effectively as intervening on the \texttt{Oracle}. This dissociation suggests the model bypasses bigram-like statistics in favor of a richer representation when genuine 
belief tracking is required. Among the candidates tested, \texttt{Soft $n$-gram} emerges as the most causally viable across both regimes.

\vspace{-0.5em}
\paragraph{\texttt{Soft $n$-gram} is the most causally effective algorithmic 
representation.}

\begin{wrapfigure}{r}{0.6\textwidth}
  \centering
  \vspace{-1em}
  \includegraphics[width=0.6\textwidth]{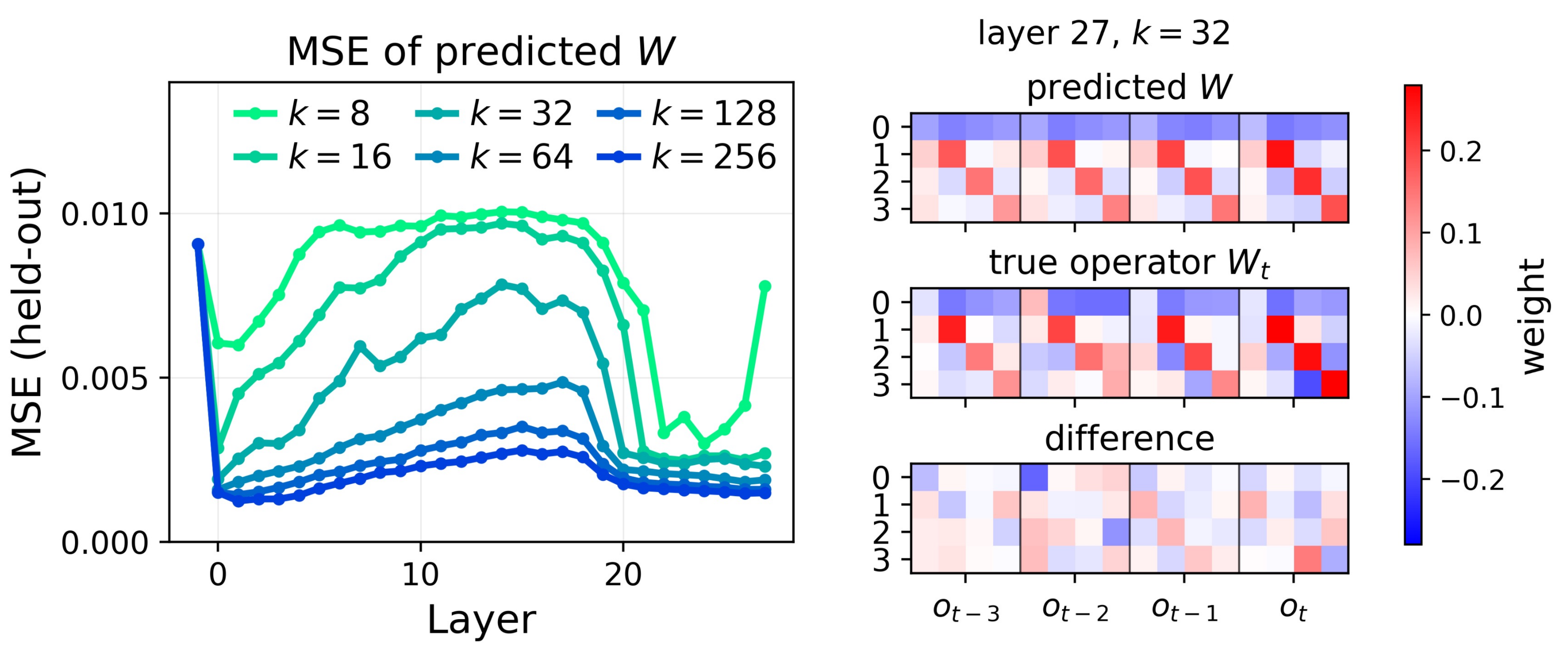}
  \caption{Probing \texttt{Linear $n$-gram} operator $W_t$ in LLM.}
  \label{fig:LLM_W}
  \vspace{-1em}
\end{wrapfigure}

Across both regimes, the \texttt{Soft $n$-gram} belief exerts a stronger causal 
effect on the model's output than \texttt{Bigram} or \texttt{Spectral}, motivating 
a closer look at how it is realized internally. We probe for the \texttt{Linear 
$n$-gram} operator $W_t$ in the belief-dominated regime (Figure~\ref{fig:LLM_W}). 
Probe MSE decreases monotonically with PCA rank. Across depth, the error follows a 
characteristic U-shape: low at early layers, rising through the middle layers 
(1--17), then falling and converging at late layers (18--27). The right panel 
visualizes the recovered operator weights at a representative layer.

\begin{figure}
  \centering
  \includegraphics[width=\textwidth]{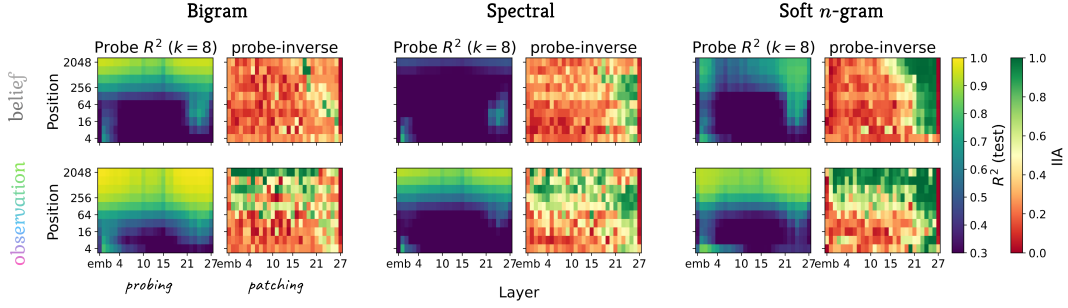}
  \caption{Algorithmic belief representations and their causal effects in the belief-dominated vs. the observation-dominated regime. The \texttt{Soft $n$-gram} column uses \texttt{Linear $n$-gram} with $n=4$.}
  \label{fig:alg_reps}
  \vspace{-2em}
\end{figure}

\paragraph{Causal representation quality predicts LLM empirical performance.}

\begin{wrapfigure}{l}{0.42\textwidth}
  \centering
  \vspace{-1em}
  \includegraphics[width=0.42\textwidth]{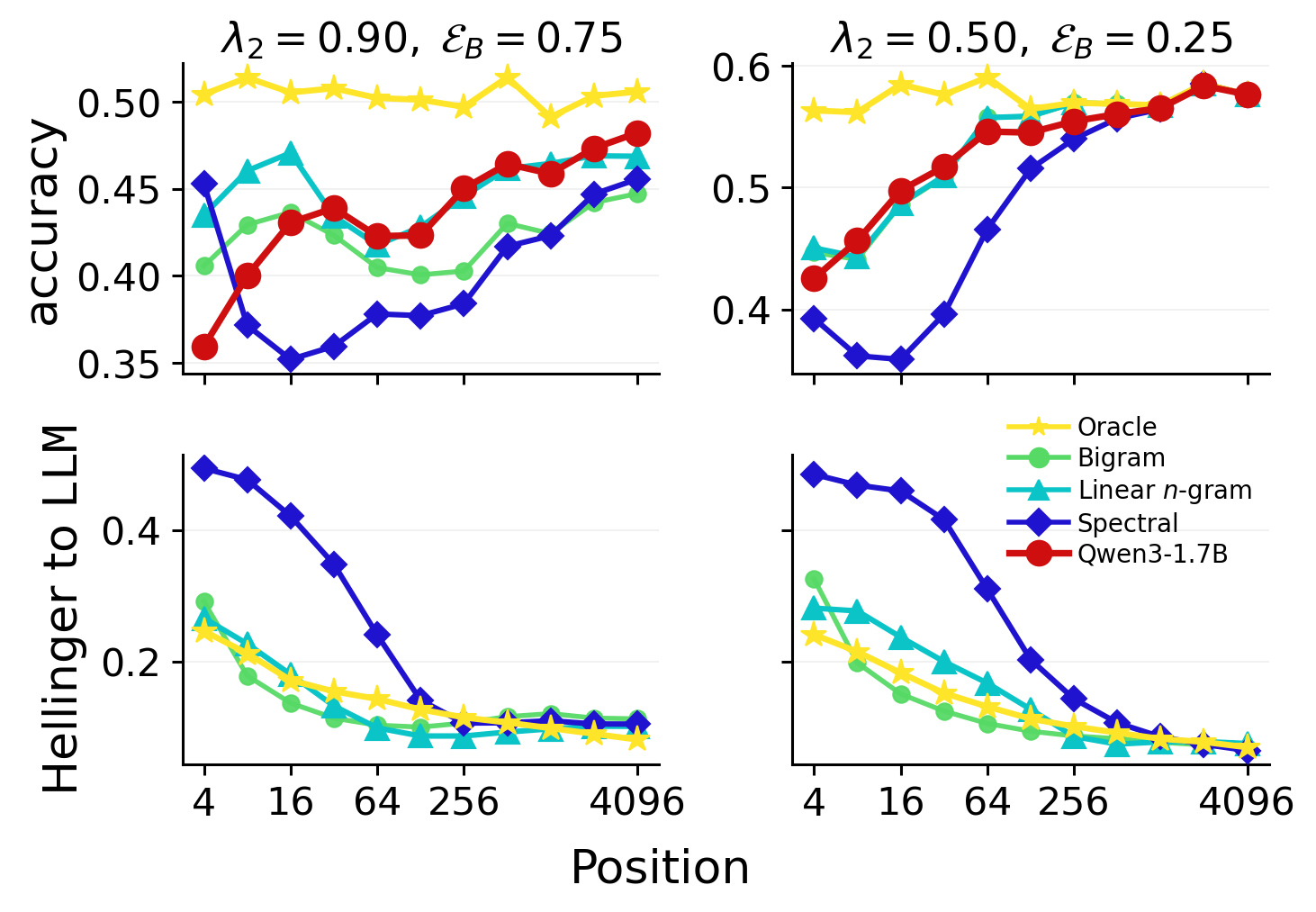}
  \caption{LLM empirical performance in the two HMM regimes.}
  \label{fig:internal_to_emp}
  \vspace{-1em}
\end{wrapfigure}

Returning to the two regimes (Figure~\ref{fig:alg_reps}): in the belief-dominated regime, \texttt{Linear $n$-gram} has a distinctly stronger causal intervention 
effect than \texttt{Bigram} or \texttt{Spectral}; in the observation-dominated 
regime, all three have strong effects. This difference tracks empirical performance 
(Figure~\ref{fig:internal_to_emp}): in the belief-dominated regime, the LLM converges to \texttt{Linear $n$-gram} by position 64 while the other two trail through position 4096; in the observation-dominated regime, all three converge to the \texttt{Oracle} by position 1024. The association between causal intervention effect and LLM predictive accuracy holds consistently across settings 
(Appendix~\ref{app:llm_emp_perf}).

The same pattern appears contrastively in a model that performs suboptimally on the 
task. OLMo-2-1B admits a high-$R^2$ probe on the \texttt{Oracle} belief (Appendix~\ref{app:llm_emp_perf}), yet causal intervention fails entirely even 
under a full residual patch (Figure~\ref{fig:olmo_probe}; contrast with 
Figure~\ref{fig:h_vs_m_gt_belief})---a further instance of decodability without causal use.

\begin{wrapfigure}{r}{0.45\textwidth}
  \centering
  \vspace{-1em}
  \includegraphics[width=0.45\textwidth]{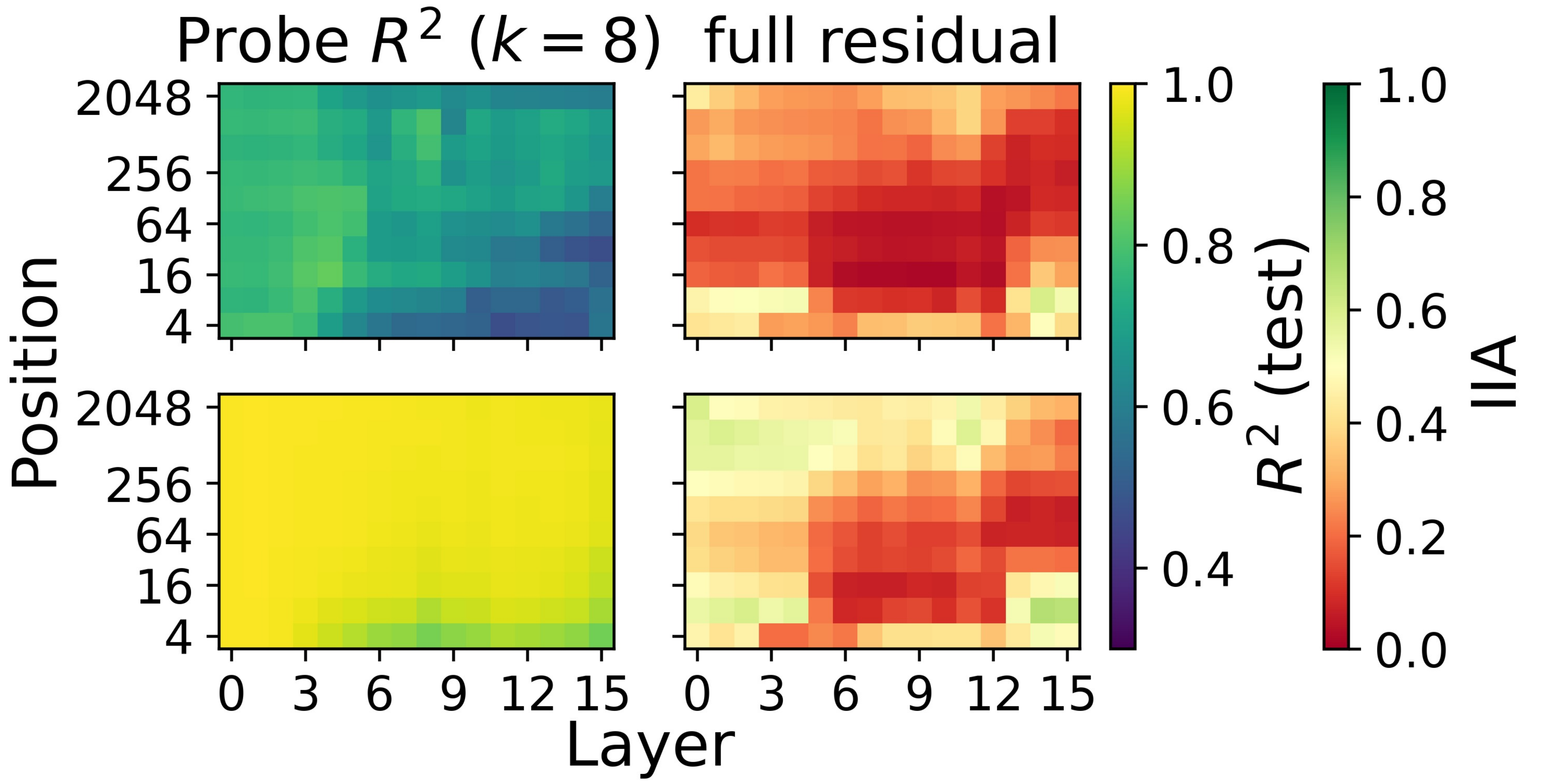}
  \caption{OLMo-2-1B probing and causal patching on the two HMM regimes: high $R^2$ but failed intervention, contrasting with Figure~\ref{fig:h_vs_m_gt_belief}.}
  \label{fig:olmo_probe}
  \vspace{-1em}
\end{wrapfigure}




%% file: related_works.tex


\paragraph{Algorithmic Identification in LLMs.}
Modern interpretability research has investigated which algorithms underlie in-context learning (ICL), finding that the mechanisms employed are sensitive to network depth and input noise~\cite{akyurek2023learning}. Complementary work demonstrates that transformers can implement a diverse repertoire of ICL algorithms
at inference time~\cite{bai2023transformers, garg2022transformers, akram2026transformers}. However, these works largely infer algorithms from input--output behaviour or via theoretical weight
constructions, rather than probing model internals directly. A notable exception is \citet{akyurek2023learning}, which partially probes hidden activations to test whether an implicit model is encoded and updated across layers.
At a finer mechanistic level, \citet{nanda2023progress} go furthest in probing internals: they fully
reverse-engineer the learned algorithm for modular addition by directly analysing weights and activations and performing ablations in Fourier space. Similarly, \citet{dangelo2026transformers} argue on theoretical grounds that LLMs learn to predict in-context via mirror descent within a latent mixture-model framework. Closer to our own approach, \citet{li2025how} identify which of two concrete state-tracking algorithms
a language model has learned for permutation composition, an associative scan or a parity-based
heuristic refinement, and show that the choice of mechanism can be predicted and controlled via
intermediate training tasks.

\paragraph{Mechanistic Interpretability and Probing.}
A foundational line of work initiated by \citet{alain2017understanding} uses linear probes to characterize the representations learned at intermediate layers of neural networks~\cite{adi2017fine, ijcai2018p796, conneau2018cram}.
\citet{li2023emergent} extended this paradigm to transformers trained on synthetic tasks, and \citet{nanda2023emergent} subsequently demonstrated linear decodability of an emergent world model under specific parametrisations, a result later generalised to Chess and additional variables such as player rating~\cite{karvonen2024emergent}.
More recent work combines activation patching with probing to show that pre-trained and fine-tuned LLMs employ identifiable state-tracking strategies and recover structured world models~\cite{vafa2025foundation, liu2026kepler}.

%% file: appendix.tex
\begin{center}\LARGE
\textbf{Appendices}
\end{center}

\section*{Table of Contents}
\begin{itemize}
\item Appendix~\ref{app:background}: Additional HMM Background
\item Appendix~\ref{app:bench}: Benchmark Algorithmic Details
\item Appendix~\ref{app:$n$-gram_prediction}: $n$-gram Prediction Constructions
\item Appendix~\ref{app:trained_transformer}: Details of the Small Trained Transformer
\item Appendix~\ref{app:llm_emp_perf}: LLMs Empirical Performance
\item Appendix~\ref{app:pap_results}: Principal Activation Probe (PAP) and Results
\end{itemize}

\section{Additional HMM Background}
\label{app:background}
\input{appendix_hmm_background}
\clearpage

\input{appendix_benchmark_methods}
\clearpage

\input{appendix_nGram_construction}
\clearpage
\section{Spectral Learning for Prediction}\label{app:spectral}

\input{appendix_approximation}
\clearpage

\section{Small Trained Transformer}
\label{app:trained_transformer}
\input{appendix_small_trained_transformer}
\clearpage

\section{LLMs Empirical Performance}
\label{app:llm_emp_perf}
\input{appendix_llm_emp}

\clearpage

\section{Principal Activations Probe (\alg{}) and Results}
\label{app:pap_results}
\input{appendix_principal_probe}
\clearpage


%% file: appendix_hmm_background.tex
In this section, we present the HMM setting in detail, along with the conditions under which the underlying Markov chain converges to a unique stationary distribution. Recall that an HMM is characterized by the Markov chain's \emph{initial state distribution} and \emph{state transitions}, together with the \emph{emission probabilities} of observations given hidden states. We focus on the finite-alphabet setting, taking states in $\mathcal{X} = \{1, 2, \ldots, M\}$ and observations in $\mathcal{O} = \{1, 2, \ldots, N\}$ without loss of generality. The model parameters comprise an initial distribution $\boldsymbol{\pi} \in \mathbb{R}^M$, with $\pi_j$ the probability of starting in state $j$; a transition matrix $\vA \in \mathbb{R}^{M \times M}$, with $a_{ij}$ the probability of moving from state $i$ to state $j$; and an emission matrix $\vB \in \mathbb{R}^{M \times N}$, with $b_{jk}$ the probability of emitting observation $k$ from state $j$. The triple $\lambda = (\boldsymbol{\pi}, \vA, \vB)$ fully parameterizes a finite-alphabet HMM.

Let $\{X_1, X_2, \ldots\}$ be a discrete-time Markov chain on $\mathcal{X}$ with transition matrix $\vA$, and let $p_{ij}^{(n)} = \mathbb{P}(X_{t+n} = j \mid X_t = i)$ denote the $n$-step transition probability from $i$ to $j$. State $j$ is \textit{accessible} from state $i$ if $p_{ij}^{(n)} > 0$ for some $n \geq 1$, and a subset $\mathcal{C} \subseteq \mathcal{X}$ is \textit{irreducible} if every pair of states in $\mathcal{C}$ is mutually accessible. The \textit{period} of state $i$ is $c(i) = \gcd\{n \geq 1 : p_{ii}^{(n)} > 0\}$, the greatest common divisor of its possible return times; state $i$ is \textit{aperiodic} when $c(i) = 1$. A Markov chain is \textit{geometrically ergodic} if it is irreducible and aperiodic, in which case it converges to a unique \textit{stationary distribution} $\boldsymbol{\mu} \in \mathbb{R}^M$ satisfying $\boldsymbol{\mu} = \boldsymbol{\mu} \vA$. Its \textit{mixing rate} is the smallest $\rho \in [0, 1)$ for which $|p_{ij}^{(n)} - \mu_j| \leq C \rho^n$ holds with some constant $C \geq 0$, uniformly over $i, j \in \mathcal{X}$ and $n \geq 1$; for a finite-alphabet HMM, $\rho$ equals $\lambda_2$, the second-largest eigenvalue of $\vA$ in modulus. To avoid dependence on the initial state, the bulk of our HMMs use ergodic transitions, though we also include a handful of non-ergodic cases in our experiments.

The \textit{entropy} of a discrete random variable $X$ is $H(X) = -\sum_{x \in \mathcal{X}} p(x) \log p(x)$. A fundamental property is that conditioning reduces uncertainty: $H(X \mid Y) \leq H(X)$ for any random variables $X$ and $Y$, with equality if and only if they are independent \citep{10.5555/1146355}. The chain rule then expresses the joint entropy of a stochastic process as $H(X_1, X_2, \ldots, X_n) = \sum_{i=1}^n H(X_i \mid X_{i-1}, \ldots, X_1)$. For a Markov chain with stationary distribution $\boldsymbol{\mu}$, the \textit{entropy rate} simplifies to $H(\mathcal{X}) = \lim_{n \to \infty} \tfrac{1}{n} H(X_1, \ldots, X_n) = -\sum_{i,j} \mu_i a_{ij} \log a_{ij}$, depending only on the transition matrix $\vA$. By analogy, we define the entropy of the emission matrix as $-\sum_{j,k} \mu_j b_{jk} \log b_{jk}$, which captures the average uncertainty in observations given the underlying state. The entropy rate of the observation process in an HMM admits no closed form, but it is sandwiched between $H(O_n \mid O_{n-1}, X_{n-1}, \ldots, O_1, X_1) \leq H(\mathcal{O}) \leq H(O_n \mid O_{n-1}, \ldots, O_1)$. Since $\vA$ governs the transitions $X_t \to X_{t+1}$ and $\vB$ governs the emission of $O_t$ from $X_t$, the entropies of $\vA$ and $\vB$ jointly control a lower bound on the entropy of the sampled HMM sequence.

%% file: appendix_benchmark_methods.tex
\section{Benchmark Algorithm Details}
\label{app:bench}

This appendix details the implementation of each benchmark predictor used in Section~\ref{sec:alg_descriptions} and Section~\ref{sec:llm_emp}. Throughout, $u_t = e_{o_t} \in \mathbb{R}^N$ denotes the one-hot encoding of observation $o_t$, 
and $q_{t+1} \in \Delta^{N-1}$ denotes the predicted next-observation distribution.

\subsection{Oracle: Forward Algorithm}
\label{app:bench:forward}
Given the true HMM parameters $\boldsymbol{\lambda} = (\boldsymbol{\pi}, \vA, \vB)$, 
the Bayes-optimal predictor is computed by the standard forward recursion. The belief 
state $b_t \in \Delta^{M-1}$ is updated as
\[
b_t = \frac{\operatorname{diag}(\vB u_t)\,\vA^\top b_{t-1}}
           {\onenorm{\operatorname{diag}(\vB u_t)\,\vA^\top b_{t-1}}},
\qquad b_0 = \boldsymbol{\pi},
\]
and the next-observation distribution is read out as $q_{t+1} = \vB^\top \vA^\top b_t$.

\subsection{Spectral Methods}
\label{app:bench:spectral}

Spectral methods estimate empirical low-order moments $\hat{P}_1 \in \mathbb{R}^N$, 
$\hat{P}_2 \in \mathbb{R}^{N \times N}$, and $\hat{P}_3(o) \in \mathbb{R}^{N \times N}$ 
for each $o \in \mathcal{O}$ from observation triples in the prefix $\mathbf{o}_{1:t}$, then form an observation-operator predictor that bypasses iterative parameter estimation. The observation operator for symbol $o$ is $\hat{B}_o := \hat{P}_2^{+} \hat{P}_3(o)$, 
so the belief state $b_t \in \mathbb{R}^M$ evolves recursively as $b_t \propto \hat{B}_{o_t} b_{t-1}$, with $b_0 \propto \hat{P}_2^{+} \hat{P}_1$. 
The resulting predictor takes the form
\[
\mathbb{P}(O_{t+1} = o \mid \mathbf{o}_{1:t}) \;\propto\;
\hat{P}_1^\top\, \hat{B}_{o_t} \cdots \hat{B}_{o_1}\, \hat{P}_2^{+} \hat{P}_1,
\]
where the operator $\hat{B}_t := \{\hat{B}_o\}_{o \in \mathcal{O}}$ converges to its ground-truth value as more observations are used to estimate the moments.

\paragraph{\texttt{Spectral}.} Standard implementation following \citet{hsu2012spectral}, using the rank-$M$ truncated pseudoinverse of $\hat{P}_2$ via SVD as above.

\paragraph{\texttt{Spectral-Norm}.} Replaces $\hat{P}_2^{+}$ with $\hat{P}_2^\top$, 
defining $\hat{B}_o = \hat{P}_2^\top \hat{P}_3(o)$:
\[
\mathbb{P}(O_{t+1} = o \mid \mathbf{o}_{1:t}) \;\propto\;
\hat{P}_1^\top\, \hat{P}_2^\top \hat{P}_3(o_t) \cdots 
\hat{P}_2^\top \hat{P}_3(o_1)\, \hat{P}_2^\top \hat{P}_1.
\]
This avoids ill-conditioning from small singular values of $\hat{P}_2$ at the cost of biased moment matching.

\subsection{Linear $n$-gram Methods}
\label{app:bench:linear_ngram}
Let $\psi_t^{(n)} := [u_t^\top, u_{t-1}^\top, \ldots, u_{t-n+1}^\top]^\top \in 
\mathbb{R}^{nN}$ be the stacked one-hot feature of the $n$ most recent observations. 
We fit a linear predictor $W \in \mathbb{R}^{N \times nN}$ in three variants, corresponding to \texttt{Gradient CE}, \texttt{Gradient MSE}, and \texttt{Ridge MSE} 
in Table~\ref{tab:bench}.

\paragraph{\texttt{Gradient CE}.} Online gradient descent on the cross-entropy loss
\[
\ell_{\mathrm{CE}}(W; t) = -u_{t+1}^\top \log \operatorname{softmax}(W \psi_t^{(n)}),
\]
with predictor $q_{t+1} = \operatorname{softmax}(W_t \psi_t^{(n)})$ and update 
$W_{t+1} = W_t - \eta\,(\operatorname{softmax}(W_t \psi_t^{(n)}) - u_{t+1})\,
(\psi_t^{(n)})^\top$.

\paragraph{\texttt{Gradient MSE}.} Online gradient descent on the squared error
\[
\ell_{\mathrm{MSE}}(W; t) = \tfrac{1}{2}\|W \psi_t^{(n)} - u_{t+1}\|_2^2,
\]
with $q_{t+1} = W_t \psi_t^{(n)}$ projected onto the simplex.

\paragraph{\texttt{Ridge MSE}.} At each evaluation position, we solve the closed-form ridge regression
\[
\hat W_t = \arg\min_W \sum_{s=n}^{t-1} \|W \psi_s^{(n)} - u_{s+1}\|_2^2 
           \;+\; \lambda \|W\|_F^2,
\]
and predict $q_{t+1} = \hat W_t \psi_t^{(n)}$ projected onto the simplex. Regularization $\lambda$ is selected on a held-out validation split.

For the \texttt{Soft $n$-gram} entry of Table~\ref{tab:alg_representations} 
(Section~\ref{sec:llm_internal}), we use \texttt{Ridge MSE} weights at position $t$: the operator representation is $W_t$ and the belief representation is 
$W_t \psi_t^{(n)}$.

\subsection{Non-linear $n$-gram Methods}
\label{app:bench:nonlinear_ngram}
Let $\phi_t^{(n)} := u_t \otimes u_{t-1} \otimes \cdots \otimes u_{t-n+1} \in 
\mathbb{R}^{N^n}$ be the Kronecker feature. By Lemma~\ref{lemma:kron_prediction_main}, $q_{t+1} = \mathcal{G}_{t,n}\,\phi_t^{(n)}$ for some $\mathcal{G}_{t,n} \in \mathbb{R}^{N \times N^n}$.

\paragraph{\texttt{Kron}.} Ridge regression of the next observation on $\phi_t^{(n)}$. Because $\phi_t^{(n)}$ is an indicator for a single $n$-gram (exactly one entry equals one, the rest zero), ridge regression reduces to bucketing the prefix by its $n$-gram suffix and tabulating smoothed empirical conditional distributions.

\paragraph{\texttt{Kernel}.} Kernel ridge regression on the $n$-observation window, capturing the same nonlinear interactions as $\phi_t^{(n)}$ without materializing the $\mathcal{O}(N^n)$ feature vector, which permits larger $n$.

\subsection{Baum--Welch}
\label{app:bench:bw}
We run the standard Baum--Welch EM algorithm~\citep{baum1970maximization} with $M$ assumed known. Rows of $\vA$ and $\vB$ are initialized by drawing uniformly from the 
simplex (i.e., symmetric Dirichlet with concentration $1$); we run to convergence in negative log-likelihood (tolerance $10^{-6}$, max $200$ iterations). At each evaluation position, the converged parameters are plugged into the forward recursion of Appendix~\ref{app:bench:forward}.

%% file: appendix_nGram_construction.tex
\section{$n$-gram Prediction}\label{app:$n$-gram_prediction}
\subsection{Nonlinear Features for $n$-gram Prediction}
Recall that, an HMM is specified by the triple $\boldsymbol{\lambda} = (\boldsymbol{\pi}, \vA, \vB)$, where $\boldsymbol{\pi} \in \mathbb{R}^M$ is the initial state distribution ($\pi_j$ the probability of starting in state $j$), $\vA \in \mathbb{R}^{M \times M}$ is the transition matrix ($a_{ij}$ the probability of moving from state $i$ to $j$), and $\vB \in \mathbb{R}^{M \times N}$ is the emission matrix ($b_{jk}$ the probability of emitting $k$ from state $j$). 
Let $u_t$ denote the one-hot vector in $\R^{N}$ corresponding to the observation $o_t$, that is, $u_t = e_{o_t}$, where $e_1,\dots,e_N$ denote the standard basis vectors of $\mathbb{R}^N$. Let $q_t \in \R^N$ denote the probability distribution of next observation given the history of observations, that is, $q_{t+1}[i] :=  \P\left(O_{t+1} = i \bgl O_{1:t} = o_1,o_2, \dots, o_t \right)$. Another important quantity is the belief state $b_t \in \R^M$ which is defined as the posterior probability of the HMM being at each state given the history of observations, that is, $b_{t}[i] :=  \P\left(X_{t+1} = i \bgl O_{1:t} = o_1,o_2, \dots, o_t \right)$. Using Bayes' rule, we have
\begin{equation}
\begin{aligned} \label{eqn:belief_update_orig}
    b_t &= \frac{\operatorname{diag}(\vB u_t) \vA^\top b_{t-1}}{\onenorm{\operatorname{diag}(\vB u_t) \vA^\top b_{t-1}}},  \\
    q_{t+1} &= \vB^\top \vA^\top b_t
\end{aligned}
\end{equation}
Opening up the recursion till $t-n$, and defining the scaling $\alpha_{t,n} {:=} \prod_{i=1}^n \onenorm{\operatorname{diag}(\vB u_{t-i+1}) \vA^\top b_{t-i}}$, we have 
\begin{align}
    q_{t+1} = \vB^\top \vA^\top  \cdot \operatorname{diag}(\vB u_t) \vA^\top  \cdot \operatorname{diag}(\vB u_{t-1}) \vA^\top \cdots \operatorname{diag}(\vB u_{t-n+1}) \vA^\top b_{t-n}/\alpha_{t,n} \label{eqn:output_prediction_orig}
\end{align}
Let $\vB[:,k]$ denote the $k$-th column of the emission matrix $\vB$, and let $\vA_k := \operatorname{diag} \left( \vB[:,k] \right)\vA^\top$. Then \eqref{eqn:belief_update_orig}, and \eqref{eqn:output_prediction_orig} can be alternately expressed as as a bilinear dynamical equation~\cite{sattar2025finite} as follows,
\begin{equation}
\begin{aligned} \label{eqn:belief_update_bilinear}
    b_t &= \operatorname{diag}\left(\sum_{k=1}^N u_t[k] \vB[:,k]\right) \vA^\top b_{t-1}/\alpha_{t,1} = \sum_{k=1}^N u_t[k] \operatorname{diag}\left(\vB[:,k]\right) \vA^\top b_{t-1}/\alpha_{t,1} \\   
    & = \sum_{k=1}^N u_t[k] \vA_k b_{t-1}/\alpha_{t,1}\\
    q_{t+1} &= \vB^\top \vA^\top b_t
\end{aligned}
\end{equation}
Again, opening up the recursion~\eqref{eqn:belief_update_bilinear} till $t-n$, we have
\begin{align}
    q_{t+1} = \vB^\top \vA^\top \left( \sum_{k=1}^N u_t[k] \vA_k \right) \left( \sum_{k=1}^N u_{t-1}[k] \vA_k \right)   \cdots \left( \sum_{k=1}^N u_{t-n+1}[k] \vA_k \right) b_{t-n}/\alpha_{t,n} \label{eqn:output_prediction_bilinear}
\end{align}

Using \eqref{eqn:output_prediction_bilinear}, we can write the next observation prediction as a function of a nonlinear feature vector as follows. 
\begin{lemma}[$n$-gram prediction] \label{lemma:kron_prediction}
Consider an HMM specified by $\boldsymbol{\lambda} = (\boldsymbol{\pi}, \vA, \vB)$. Let $\vA_k := \operatorname{diag} \left( \vB[:,k] \right)\vA^\top$ for all $k\in[N]$, where $\vB[:,k]$ denotes the $k$-th column of the emission matrix $\vB$. Let 
\begin{align}
\phi_t^{(n)}
:=
u_t\otimes u_{t-1}\otimes \cdots \otimes u_{t-n+1}\in\mathbb{R}^{N^n}, \label{eqn:kron_features}
\end{align}
denote a nonlinear feature of $n$ past observations. Let $b_t$ denotes the belief state at time $t$, and $\alpha_{t,0} >0$ be a normalizing scalar. 
Then, there exists a matrix $\Gc_{t,n}\in \mathbb{R}^{N\times N^n}$ such that
\begin{align}
q_{t+1} = \Gc_{t,n}\,\phi_t^{(n)},
\end{align}
where the column of $\Gc_{t,n}$ indexed by $(k_0,\dots,k_{n-1}) \in[N]^n$ is given by
\begin{align}
\Gc_{t,n}[:,(k_0,\dots,k_{L-1})]
=
\vB^\top \vA^\top \vA_{k_0} \vA_{k_1} \cdots \vA_{k_{n-1}}\, b_{t-n}/\alpha_{t,n}.
\end{align}
\end{lemma}

\begin{proof}
From equation \eqref{eqn:output_prediction_bilinear}, we have
\[
 q_{t+1} = \vB^\top \vA^\top \left( \sum_{k=1}^N u_t[k] \vA_k \right) \left( \sum_{k=1}^N u_{t-1}[k] \vA_k \right)   \cdots \left( \sum_{k=1}^N u_{t-n+1}[k] \vA_k \right) b_{t-n}/\alpha_{t,n}.
\]
Distributing the product over the sums yields
\[
q_{t+1}
{=}
\sum_{k_0=1}^N \sum_{k_1=1}^N \cdots \sum_{k_{n-1}=1}^N
\left(\vB^\top \vA^\top \vA_{k_0} \vA_{k_1}\cdots \vA_{k_{n-1}} \frac{b_{t-n}}{\alpha_{t,n}}\right)
\left(u_t[k_0]u_{t-1}[k_1]\cdots u_{t-n+1}[k_{n-1}]\right).
\]
Next, note that the coordinate of $ \phi_t^{(n)} = u_t \otimes u_{t-1}\otimes \cdots \otimes u_{t-n+1}$ indexed by $(k_0,\dots,k_{n-1})$ is precisely $
\phi_t^{(n)}[(k_0,\dots,k_{n-1})] = u_t[k_0] u_{t-1}[k_1]\cdots u_{t-n+1}[k_{n-1}]$. Therefore, defining $\Gc_{t,n}$ column-wise by
\begin{align}
\Gc_{t,n}[:,(k_0,\dots,k_{n-1})]
=
\vB^\top \vA^\top \vA_{k_0} \vA_{k_1} \cdots \vA_{k_{n-1}}\, b_{t-n}/\alpha_{t,n}
\end{align}
we obtain $q_{t+1} = \Gc_{t,n}\phi_t^{(n)}$. This completes the proof. 
\end{proof}

\begin{remark}
Note that, since each $u_{t-j}$ is a one-hot vector, the tensor product $u_t\otimes u_{t-1}\otimes \cdots \otimes u_{t-n+1}$ is also one-hot. Hence $q_{t+1}$ is obtained by selecting exactly one column of $\Gc_{t,n}$. 
\end{remark}
One of the drawbacks of the nonlinear feature $\phi_t^{(n)}$ is that the size of the unknown matrix/operator $\Gc_{t,n}$ grows exponentially with the history length $n$. Hence, the sample complexity of learning $\Gc_{t,n}$ also scales with exponentially with $n$. In order to mitigate this issue, we note that the column space of $\Gc_{t,n}$ lies in a low dimensional subspace space, because the effective dimensionality of the unknown parameter space is $M^2 + MN$ (instead of $N^{n+1}$).
\subsection{Linear Features for $n$-gram Prediction}
In this section, we show that, we can learn an efficient $n$-gram prediction function for $q_{t+1}$ by exploiting the special structure of the one-hot nonlinear feature $\phi_t^{(n)}$ in \eqref{eqn:kron_features}. Specifically, we will show how the nonlinear features $\phi_t^{(n)}$ encodes the linear features studied in \cite{hao2025transformers}.
\begin{lemma}[Features relation]\label{lemma:lin_kron_conversion}
    Consider the same setup of Lemma~\ref{lemma:kron_prediction}. Furthermore, let $\psi_t^{(n)} := [u_{t}^\top~~u_{t-1}^\top~~\cdots~~u_{t-n+1}^\top]^\top \in \R^{nN}$ denote a linear feature of $n$ past observations. Then, there exists matrices $\vW_{lin} \in \mathbb{R}^{nN \times N^n}$, $\vW_{kron} \in \mathbb{R}^{N^n \times nN}$, and a vector $v \in \R^{N^n}$ such that,
    \begin{align}
        \psi_t^{(n)} = \vW_{lin} \phi_t^{(n)}, \qquad \text{and} \qquad \phi_t^{(n)} = \mathrm{ReLU}\left( \vW_{kron} \psi_t^{(n)} + v\right)
    \end{align}
    where $\mathrm{ReLU}(x):= x^+ = \max(0,x)$, and is applied entry-wise to vectors.
\end{lemma}
\begin{proof}
We first show the existence of a matrix $\vW_{lin} \in \mathbb{R}^{nN \times N^n}$ such that, $\psi_t^{(n)} = \vW_{lin} \phi_t^{(n)}$. To begin, let $e_1,\dots,e_N$ denote the standard basis vectors of $\mathbb{R}^N$. If we set the column of $\vW_{lin}$ indexed by $(k_0,\dots,k_{n-1}) \in[N]^n$ to be
\begin{align}
    \vW_{lin}[:,(k_0,\dots,k_{n-1})] = \begin{bmatrix}
e_{k_0}\\
e_{k_1}\\
\vdots\\
e_{k_{n-1}}
\end{bmatrix}
\in \mathbb{R}^{nN} \label{eqn:W_lin_def}
\end{align}
then, it is easy to see that $\phi_t^{(n)}$ will pick the column of $\vW_{lin}$ which is exactly $\psi_t^{(n)}$, that is, we have $\psi_t^{(n)} = \vW_{lin} \phi_t^{(n)}$. Next, we show the existence of a matrix $\vW_{kron} \in \mathbb{R}^{N^n \times nN}$, and a vector $v \in \R^{N^n}$ such that, $ \phi_t^{(n)} = \mathrm{ReLU}\left( \vW_{kron} \psi_t^{(n)} + v\right)$. Let $\varphi_t^{(n)}$ be a vector in $\R^{N^n}$ such that, its entries indexed by $(k_0,\dots,k_{n-1}) \in[N]^n$ are given by
\begin{align}
    \varphi_t^{(n)}[(k_0,\dots,k_{n-1})] = \sum_{j=0}^{n-1} u_{t-j}[k_j].
\end{align}

Since $\{u_t\}_{t\geq 0}$ are one-hot vectors, we have
\begin{align}
     \varphi_t^{(n)}[(k_0,\dots,k_{n-1})] \quad 
     \begin{cases}
        =n,& \text{if } u_t {=} e_{k_0},u_{t-1} {=} e_{k_1},\dots,u_{t-n+1} {=} e_{k_{n-1}},\\
        < n,& \text{otherwise.}
    \end{cases}
\end{align}
Applying ReLU function after subtracting $n-1$ from each entry of $\varphi_t^{(n)}$, we have
\begin{align}
    \mathrm{ReLU}\left(\varphi_t^n[(k_0,\dots,k_{n-1})] - n +1 \right) &= 
     \begin{cases}
        1,& \text{if } u_t {=} e_{k_0},u_{t-1} {=} e_{k_1},\dots,u_{t-n+1} {=} e_{k_{n-1}},\\
        0,& \text{otherwise.}  
    \end{cases} \label{eqn:ReLU_on_varphi}
\end{align}
Comparing \eqref{eqn:ReLU_on_varphi} with \eqref{eqn:kron_features}, we have
\begin{align}
    \mathrm{ReLU}\left(\varphi_t^{(n)} -(n-1)\cdot \mathbf{1} \right) =  u_t \otimes u_{t-1}\otimes \cdots \otimes u_{t-n+1} = \phi_t^{(n)}, \label{eqn:phi_equals_ReLU_on_varphi}
\end{align}
where $\mathbf{1}$ is a vector of all ones in $\R^{N^n}$. To complete the proof, we need to show the existence of a matrix $\vW_{kron} \in \mathbb{R}^{N^n \times nN}$ such that $\varphi_t^{(n)} = \vW_{kron} \psi_t^{(n)}$. If we set the row of $\vW_{kron}$ indexed by $(k_0,\dots,k_{n-1}) \in[N]^n$ to be
\begin{align}
\vW_{kron}[(k_0,\dots,k_{n-1}),:] = \begin{bmatrix}
e_{k_0}^\top &
e_{k_1}^\top &
\hdots &
e_{k_{n-1}}^\top
\end{bmatrix}
\in \mathbb{R}^{nN} \label{eqn:W_kron_def}
\end{align}
then, it is easy to see that $\vW_{kron}\psi_t^{(n)} = \varphi_t^{(n)}$. Combining this with \eqref{eqn:phi_equals_ReLU_on_varphi}, we have
\begin{align}
    \mathrm{ReLU}\left(\vW_{kron}\psi_t^{(n)} -(n-1)\cdot \mathbf{1} \right) =  u_t \otimes u_{t-1}\otimes \cdots \otimes u_{t-n+1} = \phi_t^{(n)}. \label{eqn:phi_equals_ReLU_on_psi}
\end{align}
This completes the proof.
\end{proof}
\begin{remark}
    From the Proof of Lemma~\ref{lemma:lin_kron_conversion}, it is easy to see that $\vW_{kron} = \vW_{lin}^\top$.
\end{remark}
Lemma~\ref{lemma:lin_kron_conversion} shows that the nonlinear features $\phi_t^{(n)}$ can be obtained from the linear features $\psi_t^{(n)}$ after passing it through a ReLU network. However, this increases the dimension of the feature space exponentially. In the following, we will identify conditions, under which linear features $\psi_t^{(n)}$ can be used to approximate the prediction $q_{t+1}$. 

\subsection{When does Linear Features $\psi_t^{(n)}$ provide accurate prediction?}
\label{app:when_linear_acc}
First, we show that the matrix $\vW_{lin}$ in Lemma~\ref{lemma:lin_kron_conversion} is not invertible (rank-deficient). Hence, we cannot construct the nonlinear features $\phi_t^{(n)}$ from the linear features $\psi_t^{(n)}$ by simple matrix inversion.
\begin{proposition}[Linear approximation]\label{prop:linear_approx}
    Consider the same setup of Lemmas~\ref{lemma:kron_prediction}, and \ref{lemma:lin_kron_conversion}. We have, (a) $\operatorname{rank}(\vW_{lin}) < nN$; and (b) The linear features $\psi_t^{(n)}$ can be used to approximate the prediction $q_{t+1}$, if there exists a matrix/operator $\Kc_{t,n} \in \R^{N \times nN}$ such that $\Gc_{t,n} \approx \Kc_{t,n} \vW_{lin}$.
\end{proposition}
\begin{proof}
Since every column of $\vW_{lin}$ contains exactly one $1$ in each of the $n$ blocks, the block sums are identical. Thus there are $n-1$ linear dependencies among the rows. This implies that
\[
\operatorname{rank}(\vW_{lin}) = \operatorname{rank}(\vW_{lin}\vW_{lin}^\top) < nN
\qquad \text{for } n>1.
\]
To prove the second statement of Proposition~\ref{prop:linear_approx}, recall from Lemma~\ref{lemma:kron_prediction}, that we have, $q_{t+1} = \Gc_{t,n}\,\phi_t^{(n)}$. Hence, if there exists a matrix/operator $\Kc_{t,n} \in \R^{N \times nN}$ such that $\Gc_{t,n} \approx \Kc_{t,n} \vW_{lin}$, then we have
\begin{align}
   q_{t+1} \approx \Kc_{t,n} \vW_{lin}\,\phi_t^{(n)} = \Kc_{t,n} \,\psi_t^{(n)}, 
\end{align}
where we get the last equality from Lemma~\ref{lemma:lin_kron_conversion}. This completes the proof. 
\end{proof}

\subsection{Transformer Construction for $\phi_t^{(n)}$ and $\psi_t^{(n)}$}
In this section, we can show that we can construct a transformer architecture with either $\log(n)$ layers or a single layer with $n$ attention heads to extract the nonlinear feature $\phi_t^{(n)}$ (hence the linear feature $\psi_t^{(n)}$) from the one-hot vectors $u_t, u_{t-1}, \dots, u_{0}$.

\begin{lemma}[Transformer realization of $\phi_t^{(n)}$]
\label{prop:transformer-realizes-phi}
Fix the history length $n \ge 1$. Suppose the lagged observations (one-hot) $u_t,u_{t-1},\dots,u_{0}$ are provided to a transformer as separate tokens together with positional encoding that uniquely identify the lags $0,1,\dots,t$. Let $\phi_t^{(n)}$, $\psi_t^{(n)}$ be as defined in Lemma~\ref{lemma:lin_kron_conversion}. Then, there exists, either (a) finite-depth multi-head transformer ($n$ heads), or (b) $\mathcal{O}(\log(n))$-depth transformer, with feed-forward MLP sub-layers that computes $\phi_t^{(n)}$ (and $\psi_t^{(n)}$ ) exactly.
\end{lemma}

\begin{proof}
First, we provide the proof of claim (a). Our main observation is that attention can gather the lagged symbols into $\psi_t^{(n)}$ as follows: Introduce a dedicated summary token $s$.
Using the positional encoding, assign one attention head to each lag $j\in\{0,\dots,n-1\}$. Choose the query of head $j$ at the summary token so that it attends only to the token carrying lag $j$.
Choose the value projection of that head to copy the content vector of that token into a reserved $N$-dimensional slot of the summary token. After concatenating the outputs of the $n$ heads, the summary token contains $\psi_t^{(n)} := [u_{t}^\top~~u_{t-1}^\top~~\cdots~~u_{t-n+1}^\top]^\top \in \R^{nN}$. Combining this with Lemma~\ref{lemma:lin_kron_conversion}, it is evident that, we need one feed-forward MLP layer to map $\psi_t^{(n)}$ to $\phi_t^{(n)}$. 
Thus, one multi-head attention layer suffices to recover the stacked lagged history exactly from the raw tokens, provided the positions are identifiable. Moreover, one multi-head attention layer with one feed-forward (ReLU) layer suffices to recover the Kronecker products of lagged history exactly.

To provide the proof of claim (b), we construct $\phi_t^{(n)}$ by a binary tree as follows: For simplicity, first we assume that $n=2^r$ for some integer $r>0$. This is without loss of generality and changes the depth only by a constant. At layer $0$, the Transformer stores the $n$ vectors
\[
    v_j^{(0)}:=u_{t-j}\in \mathbb{R}^N,
    \qquad j=0,1,\ldots,n-1.
\]
Assuming that a single attention head followed by a feed-forward MLP later can perform a tensor product (note that $u \otimes v = \operatorname{vec}(u v^\top)$) of two adjacent tokens, at layer $1$, it performs pairwise tensor products
\[
    v_j^{(1)}
    :=
    v_{2j}^{(0)}\otimes v_{2j+1}^{(0)}
    \in \mathbb{R}^{N^2},
    \qquad
    j=0,\ldots,n/2-1.
\]
Similarly, at layer $2$, it performs pairwise tensor products
\[
    v_j^{(2)}
    :=
    v_{2j}^{(1)}\otimes v_{2j+1}^{(1)}
    \in \mathbb{R}^{N^4}, \qquad
    j=0,\ldots,n/4-1
\]
Continuing in this way, 
after $r=\log_2(n)$ layers, there is a single vector
\[
    v_0^{(r)}
    =
    u_t\otimes u_{t-1}\otimes\cdots\otimes u_{t-n+1}
    =
    \phi_t^{(n)}.
\]

Lastly, we will prove our claim that a Transformer layer can perform pairwise tensor products as follows: Given two adjacent tokens $u, v$, from Lemma~\ref{lemma:lin_kron_conversion} a single ReLU-layer maps $(u,v)$ to $u \otimes v$. Hence, using an ReLU network in each layer of the transformer, we can guarentee that each Transformer layer can perform pairwise tensor products. This proves our claim that an $\mathcal{O}(\log(n))$-depth transformer, with feed-forward MLP sub-layers (specifically ReLU) can compute $\phi_t^{(n)}$ (hence, $\psi_t^{(n)}$ through a linear map) exactly. This completes the proof.
\end{proof}

\begin{remark}[Role of attention vs. MLP]
The attention layer performs routing, that is, it collects the lagged one-hot vectors into a common summary representation.
The nonlinear interaction across positions is carried out by the feed-forward ReLU layer.
In particular, the exact computation of $\phi_t^{(n)}$ is not a purely additive single-head attention phenomenon, it relies on higher-order cross-position interactions.
\end{remark}

\subsection{Transformer can Implement Gradient Descent on the Cross-Entropy Loss}
In the previous section, we showed that both linear and nonlinear features can be realized by Transformers. In this section, we will build on this tow show that Transformers can also implement Gradient Descent algorithm to minimize the cross-entropy loss of predicting $q_{t+1}$ from $\phi_t^{(n)}$ or $\psi_t^{(n)}$ via a linear map.

\begin{theorem}[Log-depth construction for CE-GD]\label{thrm_CE_GD}
Consider the same setup of Lemmas~\ref{lemma:kron_prediction}, and \ref{lemma:lin_kron_conversion}. Furthermore, consider a dataset
    \begin{align}
        \mathcal{D}_m=\{(\mathfrak{h}_t,y_t)\}_{t=n}^{n+m-1},\qquad
        \mathfrak{h}_t=(u_{t},u_{t-1},\ldots,u_{t-n+1}),
        \qquad y_i\in \{e_1,\ldots,e_N\}.
    \end{align}
     Let $\vW\in \mathbb{R}^{N\times N^n}$ be the parameter of multinomial logistic regression, $q_\vW(\cdot\mid \phi)=\operatorname{softmax}(\vW\phi)$, 
    trained by full-batch gradient descent on the empirical cross-entropy loss
    \[
        \mathcal{L}_m(\vW)
        :=
        \frac{1}{m}\sum_{t=n}^{n+m-1}
        \ell(\vW\phi_{t}^{(n)},y_t),
        \qquad
        \ell(z,y):=-y^\top \log \operatorname{softmax}(z),
    \]
    If the number of GD iterations satisfies $T=\mathcal{O}(\log(n))$, then there exists a
    Transformer of depth $\mathcal{O}(\log(n))$ that emulates $T$ steps of full-batch GD on the CE loss of the logistic-regression model $q_\vW(\cdot\mid \phi)=\operatorname{softmax}(\vW\phi)$.
\end{theorem}

\begin{proof} First, from Lemma~\ref{prop:transformer-realizes-phi}, we know that Transformers can realize both $\phi_t^{(n)}$ and $\psi_t^{n}$. Next, we consider logistic regression on these features. Let $\phi_t^{(n)} \in \R^{N^n}$ is indexed by $\alpha_t := (k_0,\dots,k_{n-1}) \in [N]^n$. Then, using the fact that $\phi_t^{(n)}$ is a one-hot vector, we have
\begin{align}
    \vW\phi_t^{(n)}
    =
    \vW[:, (k_0,\dots,k_{n-1})] = \vW[:, \alpha_t].
\end{align}
Now group the samples according to their active history index $\alpha_t$, that is, collecting together all samples whose nonlinear feature equals the same basis vector $e_{\alpha_t}$, we have
\begin{align}
    \mathcal{L}_m(\vW)
    =
    \frac{1}{m}
    \sum_{t=n}^{n+m-1}
    \ell(\vW[:, \alpha_t],y_t)
    =
    \frac{1}{m}
    \sum_{\alpha\in[M]^L}
    \sum_{t=n}^{n+m-1} \mathbf{1}\{\alpha_t=\alpha\}
    \ell(\vW[:, \alpha],y_t).
\end{align}
Thus the GD dynamics are column-wise decoupled. For a single sample,
\begin{align}
    \nabla_\vW \ell(\vW\phi_t^{(n)},y_t)
    =
    \left(
        \operatorname{softmax}(\vW\phi_t^{(n)})-y_t
    \right)
    \left(\phi_t^{(n)}\right)^\top.
\end{align}
Using $\phi_t^{(n)}= e_{\alpha_t}$, this becomes
\begin{align}
    \nabla_\vW \ell(\vW\phi_t^{(n)},y_t)
    =
    \left(
        \operatorname{softmax}(\vW[:,\alpha_t])-y_t
    \right)
    e_{\alpha_t}^\top.
\end{align}
Hence, the sample $t$ updates only the column $\vW[:,\alpha_t]$. For the query-relevant column $\theta_{\alpha_q}$ (where $q=n+m$), define the matching-history count and label-count
vector
\begin{align}
    N_{\alpha_q}
    :=
    \sum_{t=n}^{n+m-1} \mathbf{1}\{\alpha_t=\alpha_q\},
    \qquad
    c_{\alpha_q}
    :=
   \sum_{t=n}^{n+m-1} \mathbf{1}\{\alpha_t=\alpha_q\} y_t
    \in \mathbb{R}^N.
\end{align}
Then the full-batch GD recursion for the query-relevant column is
\begin{align}
    \vW[:,\alpha_q]^{(s+1)}
    =
    \vW[:,\alpha_q]^{(s)}
    -
    \frac{\eta}{m}
    \left(
        N_{\alpha_q}\operatorname{softmax}(\vW[:,\alpha_q]^{(s)})
        -
        c_{\alpha_q}
    \right),
    \quad s=0,1,\ldots,T-1.
\end{align}
Moreover,
    $\operatorname{softmax}\!\left(\vW^{(T)}\phi_q^{(n)}\right)
    =
    \operatorname{softmax}\!\left(\vW^{(T)}e_{\alpha_q}\right)
    =
    \operatorname{softmax}\!\left(\vW[:,\alpha_q]^{(T)}\right)$. 
Thus to emulate GD for the query, the Transformer does not need to update all
$N^n$ columns of $\vW$. It only needs to compute the sufficient statistics
\begin{align}
    (N_{\alpha_q},c_{\alpha_q})
\end{align}
for the history class matching the query, and then simulate the $N$-dimensional recursion above. In the following, we will show that a Transformer can compute these sufficient statistics in $\mathcal{O}(\log (n))$ depth. For each training example $t$, the Transformer compares its history $\mathfrak{h}_t=(u_{t},\dots,u_{t-n+1})$
with the query history $\mathfrak{h}_q=(u_{q},\dots,u_{q-n+1})$. For every lag $j \in [n-1]$,
\begin{align}
    u_{t-j}^\top u_{q-j}
    =
    \begin{cases}
    1,
    \quad \text{iff} \quad
    u_{t-j}=u_{q-j} \\
    0, \quad \text{otherwise}
    \end{cases}
\end{align}
Therefore, we have
\begin{align}
    r_t
    :=
    \sum_{j=0}^{n-1} u_{t-j}^\top u_{q-j}
    \quad
    \begin{cases}
    =n,
    \quad \text{iff} \quad
    h_{t}=h_{q} \\
    < n, \quad \text{otherwise}
    \end{cases}
\end{align}
The sum defining $r_t$ can be computed by a parallel binary reduction
over the $n$ lags, using $\mathcal{O}(\log (n))$ layers. A thresholding feed-forward map then computes
\begin{align}
    \
    \mu_t
    :=
    \mathbf{1}\{\mathfrak{h}_t=\mathfrak{h}_q\}
    =
    \operatorname{ReLU}(r_t-n+1),
\end{align}
because $r_t$ is an integer in $\{0,1,\dots,n-1\}$. Once the selectors $\mu_t$ are available, attention aggregation computes
\begin{align}
    c_{\alpha_q}
    =
    \sum_{t=n}^{n+m-1} \mu_t y_t,
    \qquad
    N_{\alpha_q}
    =
    \sum_{t=n}^{n+m-1} \mu_t.
\end{align}
This aggregation can be performed by a summary token attending to the context examples with values
$(y_t,1)$. In the idealized hard-attention model this is exact; in the standard softmax
model it can be approximated arbitrarily well on the finite input domain. Finally, define the one-step GD update map
\begin{align}
    F(\vW[:, \alpha],c,N)
    :=
    \vW[:, \alpha]
    -
    \frac{\eta}{m}
    \left(
        N\operatorname{softmax}(\vW[:, \alpha])-c
    \right).
\end{align}
The map $F$ is smooth in $\vW[:, \alpha]$ and continuous in $(\vW[:, \alpha],c,N)$. Hence a feed-forward MLP layer can approximate it uniformly. In the idealized hard-attention model, the block computes it exactly. Applying this block recurrently for $T$ Transformer layers, and applying softmax to the final output layer gives
\begin{align}
    \operatorname{softmax}(\vW[:, \alpha_q]^{(T)})
    =
    \operatorname{softmax}\!\left(\vW^{(T)}\phi_q^{(n)}\right).
\end{align}

The history matching and nonlinear feature construction require $\mathcal{O}(\log (n))$ depth, while the unrolled GD simulation requires $\mathcal{O}(T)$ depth. Therefore the total depth is $\mathcal{O}(\log (n)+T)$. In particular, when $T=\mathcal{O}(\log (n))$, the total depth is $\mathcal{O}(\log (n))$. This completes the proof.
\end{proof}

%% file: appendix_approximation.tex
Recall from Appendix~\ref{app:bench:spectral} that the spectral learning algorithm
estimates low-order moment statistics from the observation prefix and forms observation
operators to recursively update a belief state. Following \citet{hsu2012spectral}, define
the uni-, bi-, and trigram moment matrices
\begin{align}
    [P_1]_i &:= \Pr[O_1 = i] \;\in \mathbb{R}^N, \\
    [P_{2,1}]_{ij} &:= \Pr[O_2 = i,\, O_1 = j] \;\in \mathbb{R}^{N \times N}, \\
    [P_{3,o,1}]_{ij} &:= \Pr[O_3 = i,\, O_2 = o,\, O_1 = j] \;\in \mathbb{R}^{N \times N},
    \quad o \in \mathcal{O},
\end{align}
where $N$ denotes the number of distinct observations and $M$ the number of hidden states.
We use the one-hot encodings $u_s := e_{o_s} \in \mathbb{R}^N$ as defined in
Appendix~\ref{app:$n$-gram_prediction}.

Let $U \in \mathbb{R}^{N \times M}$ be the matrix of left singular vectors of $P_{2,1}$
corresponding to its $M$ largest singular values. By Lemma~2 of \citet{hsu2012spectral},
$\operatorname{range}(U) = \operatorname{range}(\vB^\top)$, so $U$ spans the same subspace
as the emission matrix and $U^\top \vB^\top$ is invertible. The observable representation
is then defined by
\begin{align}
    \widetilde{b}_1 &:= U^\top P_1 \;\in \mathbb{R}^M, \\
    \widetilde{b}_\infty &:= (P_{2,1}^\top U)^+ P_1 \;\in \mathbb{R}^M, \\
    B_o &:= (U^\top P_{3,o,1})(U^\top P_{2,1})^+ \;\in \mathbb{R}^{M \times M},
    \quad o \in \mathcal{O},
\end{align}
where $(\cdot)^+$ denotes the Moore--Penrose pseudoinverse. By Lemma~3 of
\citet{hsu2012spectral}, the belief state $\widetilde{b}_t \in \mathbb{R}^M$ evolves as
$\widetilde{b}_{t+1} \propto B_{o_t} \widetilde{b}_t$, and the conditional predictive
satisfies
\begin{align}\label{eq:spectral}
    \Pr(O_{t+1} = o \mid O_{1:t} = o_1, \dots, o_t)
    \;\propto\;
    \widetilde{b}_\infty^\top B_{o_t} \cdots B_{o_1}\, \widetilde{b}_1.
\end{align}
In-context, $P_1$, $P_{2,1}$, and $P_{3,o,1}$ are replaced by their empirical estimates
from the observation prefix, and $U$ is estimated as the top-$M$ left singular vectors of
$\widehat{P}_{2,1}$.

The core computation in~\eqref{eq:spectral} decomposes into three primitive operations on
the observation sequence $o_{1:t}$:
\begin{itemize}
    \item $f_{\mathrm{uni}}(o_{1:t}) := U^\top P_1 \in \mathbb{R}^M$, a linear function of
          the empirical unigram frequencies.
    \item $f_{\mathrm{bi}}(o_{1:t},\, v) := (U^\top P_{2,1})^{-1} v \in \mathbb{R}^M$,
          applying the projected inverse bigram matrix to $v \in \mathbb{R}^M$.
    \item $f_{\mathrm{tri}}(o_{1:t},\, v,\, o) := U^\top P_{3,o,1}\, v \in \mathbb{R}^M$,
          applying the projected trigram operator for symbol $o$ to $v \in \mathbb{R}^M$.
\end{itemize}
Note that $B_o = f_{\mathrm{bi}}(o_{1:t},\, f_{\mathrm{tri}}(o_{1:t},\, \cdot,\, o))$,
so each belief-state update $\widetilde{b}_{t+1} \propto B_{o_t}\widetilde{b}_t$ is one
application of $f_{\mathrm{tri}}$ followed by $f_{\mathrm{bi}}$. The full
prediction~\eqref{eq:spectral} is thus a sequential composition of $\mathcal{O}(t)$ such
pairs, followed by an $\ell_1$ normalization. Explicitly:
\begin{enumerate}
    \item Initialize $\widetilde{b}_1 = f_{\mathrm{uni}}(o_{1:t})$.
    \item For $k = 1, \dots, t-1$:
    \begin{itemize}
        \item Set $\widetilde{b}_k' = f_{\mathrm{tri}}(o_{1:t},\, \widetilde{b}_k,\, o_k)$.
        \item Set $\widetilde{b}_{k+1} = f_{\mathrm{bi}}(o_{1:t},\, \widetilde{b}_k')$.
    \end{itemize}
    After this loop, $\widetilde{b}_t$ is the updated belief after processing
    $o_1, \dots, o_{t-1}$.
    \item For each $o \in \mathcal{O}$:
    \begin{itemize}
        \item Set $\widetilde{b}_{t+1}^o = f_{\mathrm{tri}}(o_{1:t},\, \widetilde{b}_t,\, o)$.
        \item Set $c^o = f_{\mathrm{bi}}(o_{1:t},\, \widetilde{b}_{t+1}^o)$.
    \end{itemize}
    \item Output for symbol $o$: $\widetilde{b}_\infty^\top c^o$, where
          $\widetilde{b}_\infty = (P_{2,1}^\top U)^+ P_1$ is the fixed terminal vector
          from the observable representation defined above. Normalize over $o \in \mathcal{O}$.
\end{enumerate}

We now argue that each primitive and their composition can be approximated by a Transformer.

\paragraph{Approximating $f_{\mathrm{uni}}$.}
Since $P_1 = \mathbb{E}[u_t]$ and $U$ is a fixed matrix (estimated once from the prefix
via SVD of $\widehat{P}_{2,1}$), the map $o_{1:t} \mapsto U^\top \widehat{P}_1$ is a
linear function of the empirical unigram counts, implementable by a single attention layer
that averages the token embeddings.

\paragraph{Approximating $f_{\mathrm{bi}}$.}
We need to implement $v \mapsto (U^\top P_{2,1})^{-1} v$ in-context. By the Neumann
series identity, $(U^\top P_{2,1})^{-1} v \approx (I - U^\top P_{2,1})^k v$ for
appropriately normalized $P_{2,1}$ and sufficiently large $k$, so it suffices to implement
$k$ repeated multiplications by $U^\top P_{2,1}$~\cite{akyurek2023learning}. The map
$v \mapsto U^\top P_{2,1} v$ amounts to an empirical bigram-weighted sum projected by $U$,
which can be implemented by a Transformer attention layer
following~\citet{ekbote2025one}.

\paragraph{Approximating $f_{\mathrm{tri}}$.}
We need to implement $v \mapsto U^\top P_{3,o,1}\, v$ in-context for a general belief
vector $v \in \mathbb{R}^M$. We proceed in two steps.

First, consider the case $v = u_s$ for some one-hot observation vector $u_s \in
\mathbb{R}^N$. Following the construction of~\citet{varre2025learning} (Section~4.1), a
trigram estimator for this case can be realized by a two-layer Transformer with $n-1$
attention heads and no residual connection (disentangled architecture). Concretely, their
construction computes
\[
    \frac{\displaystyle\sum_{\ell=3}^{t}
        u_{\ell-1}^\top u_t \cdot u_{\ell-2}^\top u_{t-1} \cdot u_\ell}
    {\displaystyle\sum_{\ell=3}^{t}
        u_{\ell-1}^\top u_t \cdot u_{\ell-2}^\top u_{t-1}},
\]
where $u_{\ell-2}^\top u_{t-1}$ acts as a $0/1$ indicator that selects context positions
whose second lag matches the query's second lag. This implements $U^\top P_{3,o,1}\, v$
when $v = u_{t-1}$ is itself a one-hot vector from the observed sequence.

Second, for a general belief vector $v \in \mathbb{R}^M$ already present in the residual
stream, we observe that $U^\top P_{3,o,1}\, v$ is \emph{linear} in $v$. Writing $v =
\sum_i v_i e_i$, we have $U^\top P_{3,o,1}\, v = \sum_i v_i\, U^\top P_{3,o,1}\, e_i$.
Since each $e_i$ is a one-hot vector, each term $U^\top P_{3,o,1}\, e_i$ is computable by
the construction above. The linear combination against the coefficients $v_i$ can then be
formed by a feed-forward layer using the belief vector already in the residual stream. Thus
$f_{\mathrm{tri}}$ is approximable by a Transformer for any $v \in \mathbb{R}^M$.

\paragraph{Approximating the composition.}
The full algorithm requires composing $\mathcal{O}(t)$ alternating applications of
$f_{\mathrm{tri}}$ and $f_{\mathrm{bi}}$, each acting on the current $M$-dimensional
belief vector $\widetilde{b}_k \in \mathbb{R}^M$. Crucially, the operators
$B_{o_k} = f_{\mathrm{bi}} \circ f_{\mathrm{tri}}(\,\cdot\,, o_k)$ are never explicitly
materialized as $M \times M$ matrices; instead, the binary tree construction
of~\citet{hu2024limitation} (Appendix~D, eq.~D.12 onwards) composes the $M$-dimensional
vector maps $\widetilde{b} \mapsto f_{\mathrm{bi}}(f_{\mathrm{tri}}(\widetilde{b}, o_k))$
directly, with the in-context bigram and trigram statistics reused at each node of the
tree. This yields $\mathcal{O}(\log t)$ depth for a length-$t$ composition. With a context
of $m$ in-context sequences each of history length $n$, the total prefix length satisfies
$t = \mathcal{O}(n + m)$, giving an overall Transformer depth of $\mathcal{O}(\log(n+m))$,
as stated in Remark~\ref{lemma:spectral}.

%% file: appendix_small_trained_transformer.tex
Given a HMM configuration ($A$, $B$, $\pi$), we generate 8k sequences with length $2{,}049$, train on 4k sequences, and validate on the rest 4k sequences.

\subsection{Ablation on Model Config}
\label{app:1_hmm_transformer}

We first fix an HMM configuration with 3 states, 12 observations, $0.25$ $A$ entropy, and $0.75$ $B$ entropy. We report the training loss, validation loss, and accuracy on location $\{2, 4, 8, ..., 1024, 2048\}$.

\begin{figure}[!htb]
  \centering
    \includegraphics[width=\linewidth]{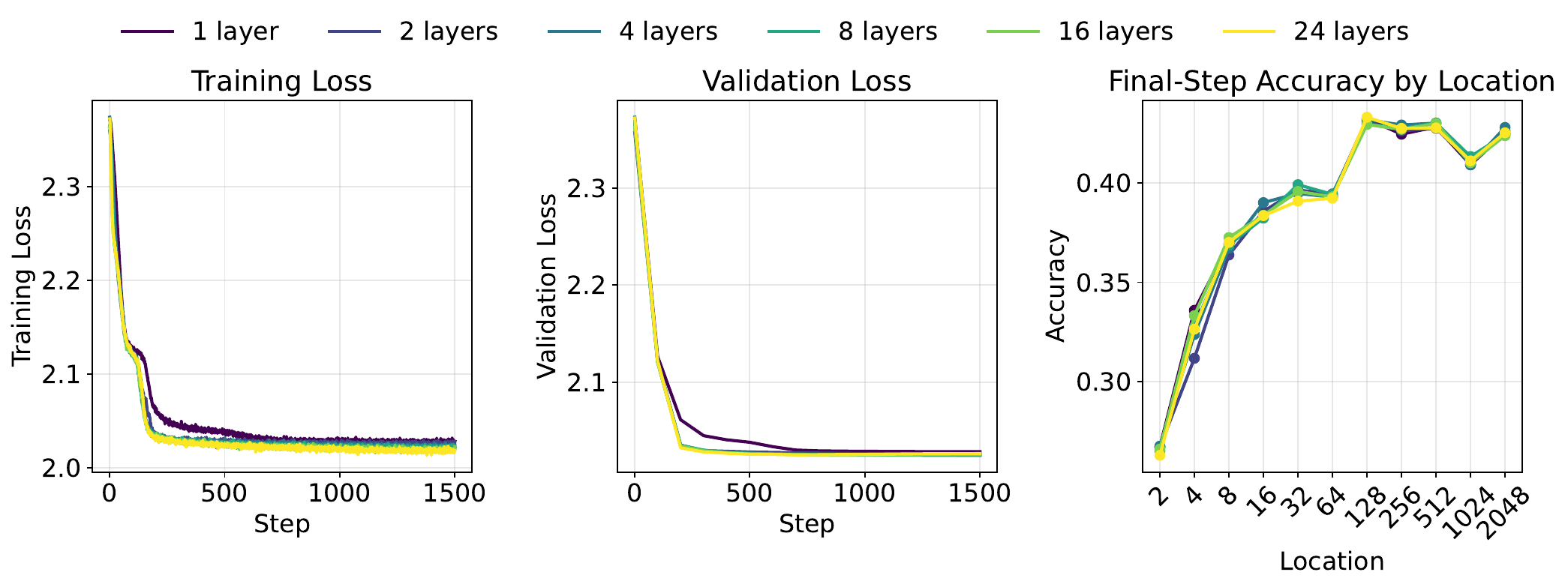}
  \caption{Training loss, validation loss, and accuracies with different number of layers.}
  \label{fig:layer_metrics}
\end{figure}

\textbf{More number of layers does not lead to better performance.} We fix the model dimension to $64$, feedforward dimension to $256$, and number of heads to $4$, and ablate the number of layers from $1$ to $24$. The results are shown in Figure~\ref{fig:layer_metrics}. Across all runs, the training loss, validation loss, and accuracies converge to a similar point.

\begin{figure}[!htb]
  \centering
    \includegraphics[width=\linewidth]{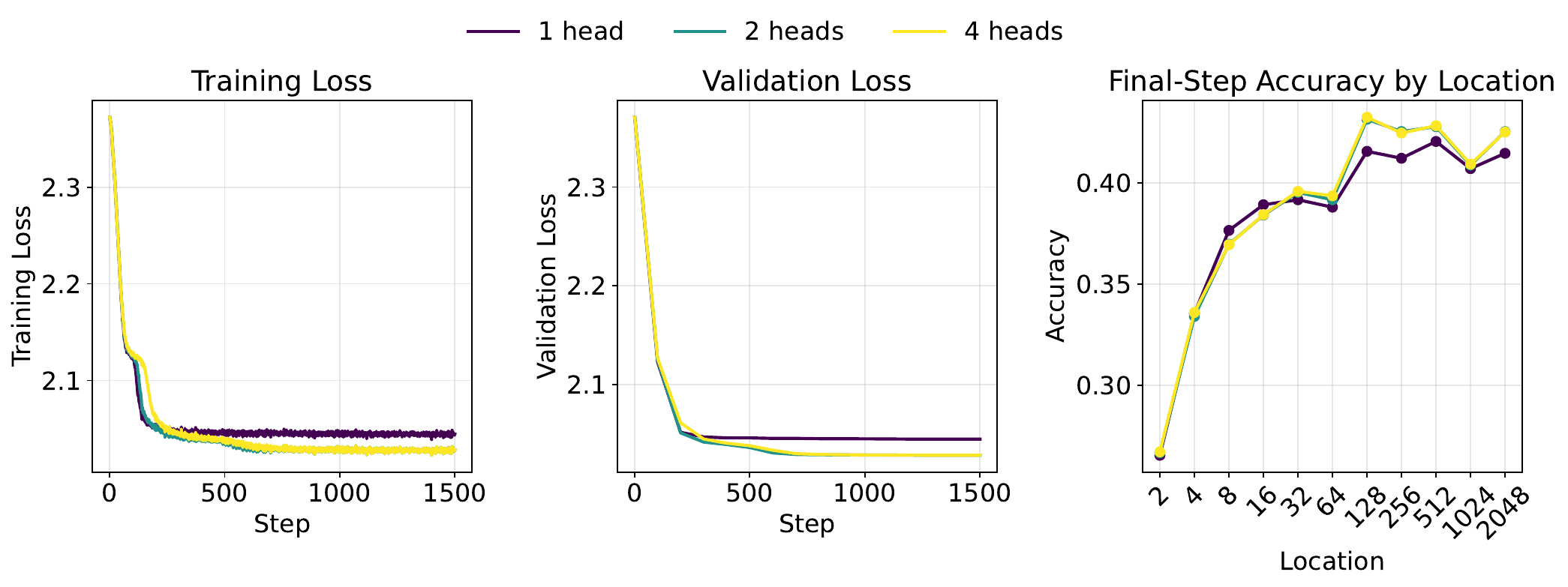}
  \caption{Training loss, validation loss, and accuracies with different number of heads.}
  \label{fig:heads_metrics}
\end{figure}

\textbf{More number of heads give slight performance improvement.} We fix the model dimension to $64$, feedforward dimension to $256$, and number of layers to $1$, and ablate the number of heads from $1$ to $4$. The results are shown in Figure~\ref{fig:heads_metrics}. Increasing number of heads above $1$ would have slightly improvement in  training / validation losses and accuracies.

\begin{figure}[!htb]
  \centering
    \includegraphics[width=\linewidth]{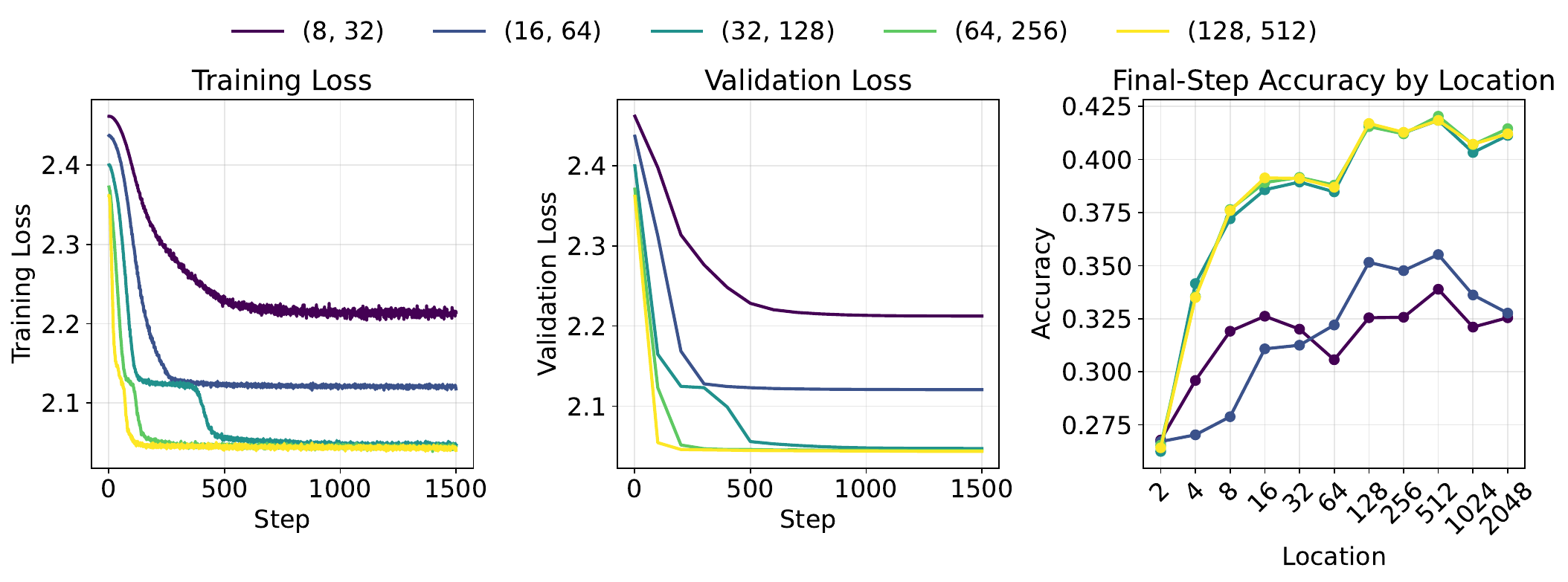}
  \caption{Training loss, validation loss, and accuracies with different model dimensions.}
  \label{fig:dim_metrics}
\end{figure}

\textbf{Smaller dimension leads to underfitting, while larger dimension does not lead to better performance.} We fix the number of layers to $1$ and number of heads to $1$, and ablate the (model dimension, feedforward dimension) from ($8$, $32$) to ($128$, $512$). The results are shown in Figure~\ref{fig:dim_metrics}. For size ($8$, $32$) and ($16$, $64$), we can see underfitting based on training and validation loss. We need at least ($32$, $128$), and increasing beyond ($64$, $256$) does not lead to better performance.

\subsection{Ablation on HMM Config}
\label{app:hmm_config}

Based on the previous section, we fix the model config to 1 head, 1 layer, 64 model dimension, and 256 feedforward dimension.

\begin{figure}[!htb]
  \centering
    \includegraphics[width=\linewidth]{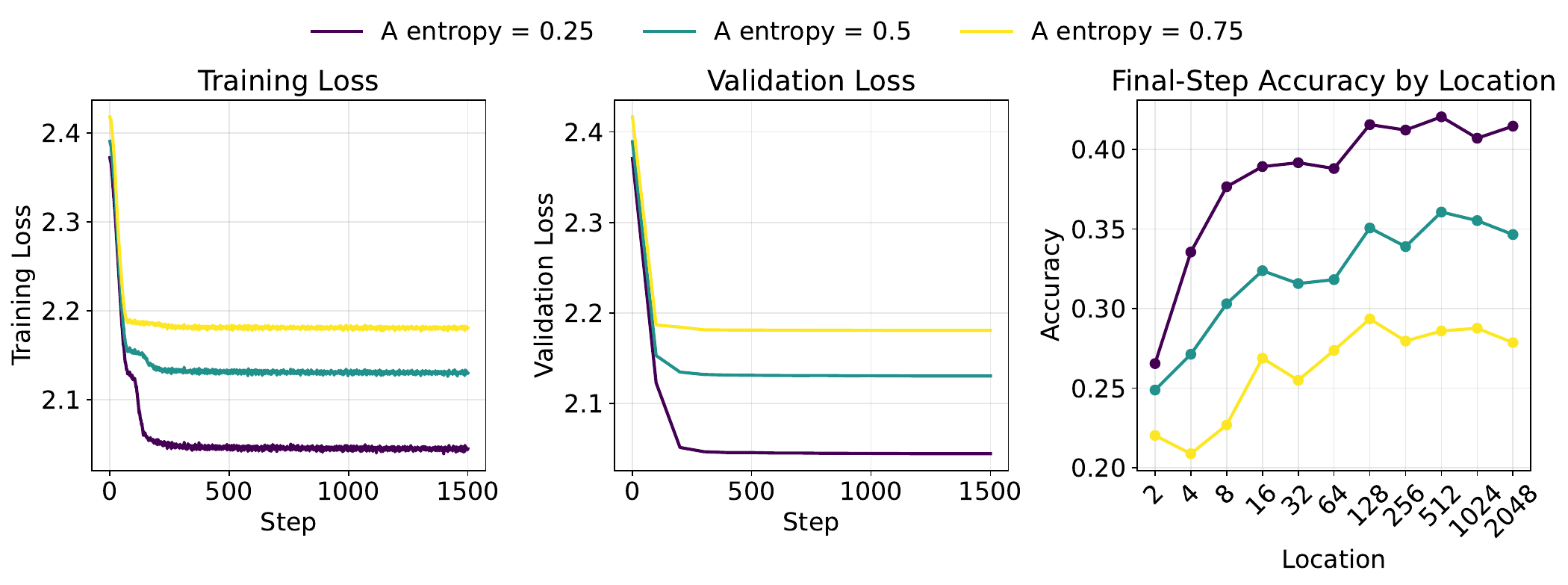}
  \caption{Training loss, validation loss, and accuracies with different A entropies.}
  \label{fig:A_entropy_metrics}
\end{figure}

\begin{figure}[!htb]
    \centering
    \includegraphics[width=\linewidth]{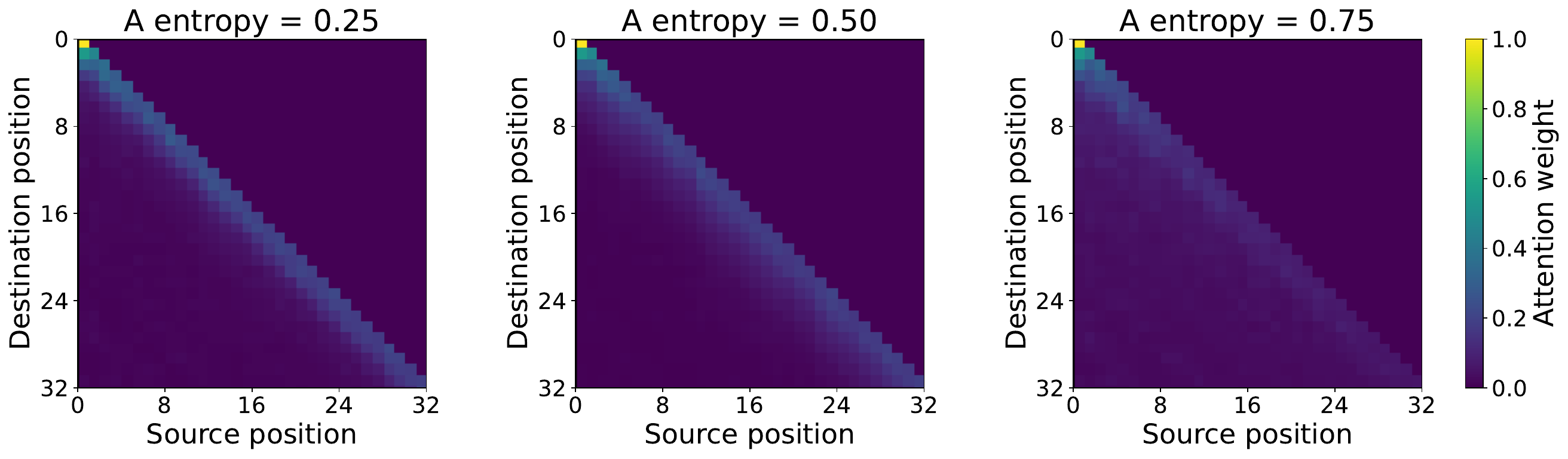}
    \caption{Attention pattern for different A entropies.}
    \label{fig:A_entropy_attention}
\end{figure}

\textbf{Slower mixing leads to longer attention window.} We fix the HMM config to 3 states and 12 observations with B entropy 0.75. We ablate the A entropy, $\{ 0.25, 0.5, 0.75\}$. The results are shown in Figure~\ref{fig:A_entropy_metrics}. As we increase A entropy, the training and validation losses become worse which is expected as the current task becomes harder. Therefore, the accuracy also decreases with larger A entropy. The attention visualization is provided in Figure~\ref{fig:A_entropy_attention}. From the attention visualization, we can see that, as A entropy increases (i.e. slower mixing), the attention window also expands. When the HMM has slower mixing, the model needs to incorporate a longer context in order to make accurate predictions.

\textbf{Larger B entropy leads to longer attention window.} We fix the HMM config to 3 states and 12 observations with A entropy 0.25. We ablate the B entropy, $\{ 0.25, 0.5, 0.75\}$. The results are shown in Figure~\ref{fig:B_entropy_metrics}. As we increase B entropy, again the training and validation losses become worse as the current task becomes harder. The accuracy also decreases with larger B entropy. The attention visualization is provided in Figure~\ref{fig:B_entropy_attention}. From the attention visualization, we can see that, as B entropy increases, the attention window also expands.

\begin{figure}[!htb]
  \centering
    \includegraphics[width=\linewidth]{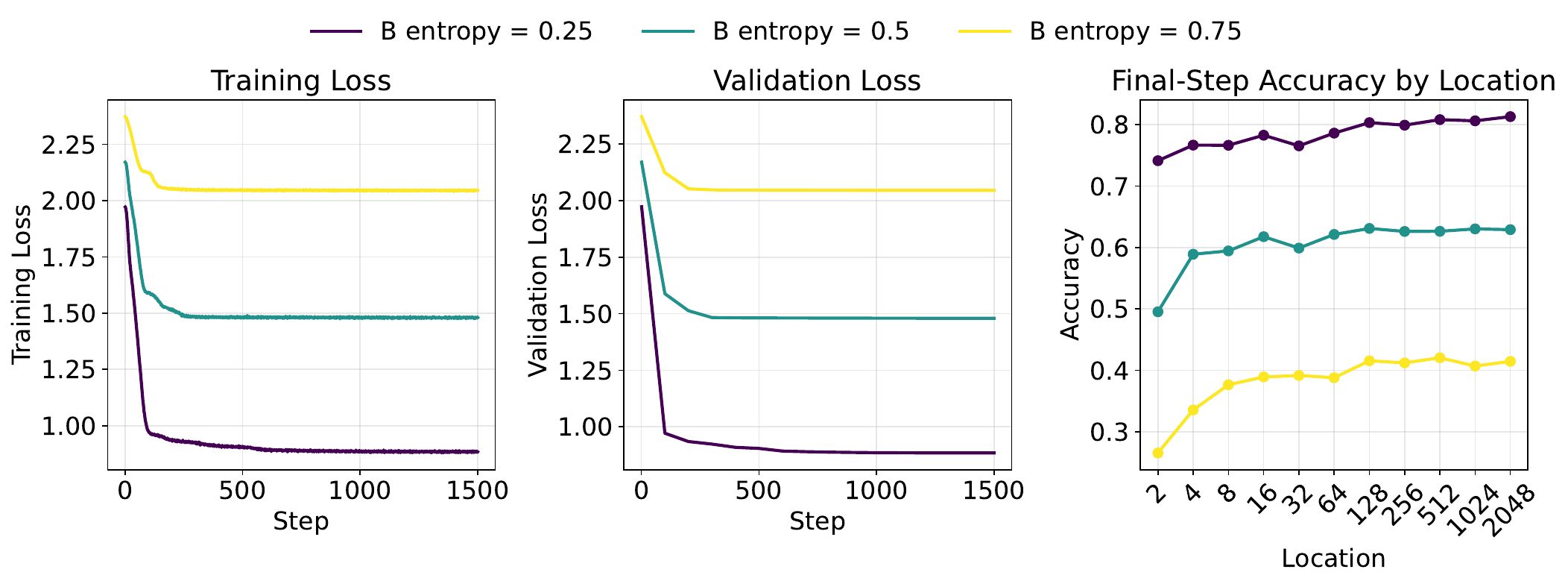}
  \caption{Training loss, validation loss, and accuracies with different B entropies.}
  \label{fig:B_entropy_metrics}
\end{figure}

\begin{figure}[!htb]
    \centering
    \includegraphics[width=\linewidth]{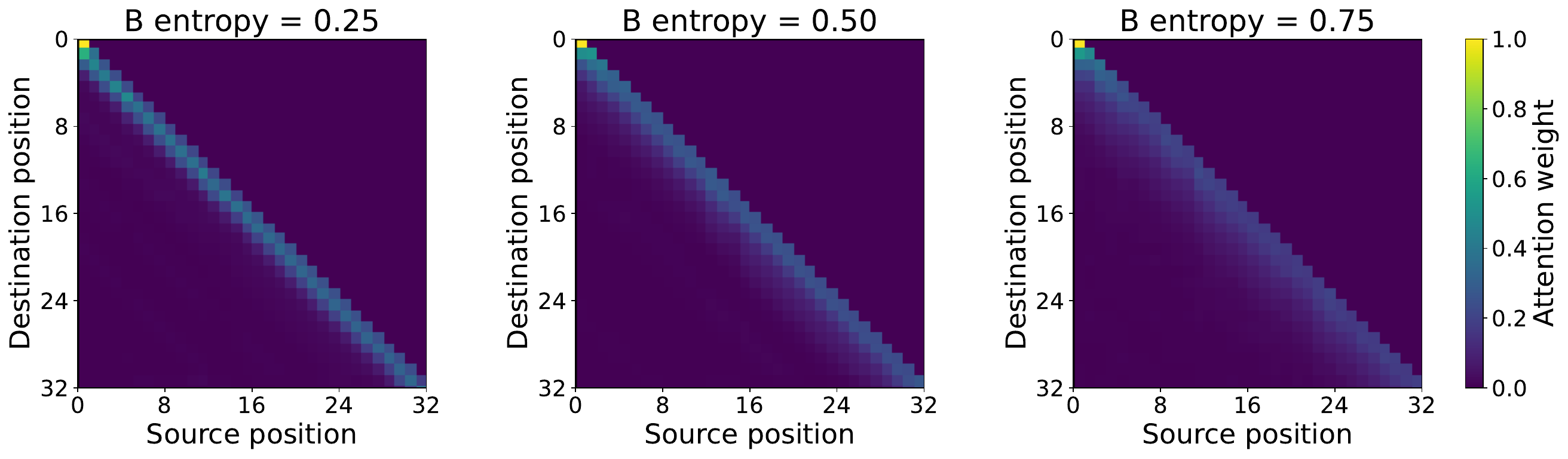}
    \caption{Attention pattern for different B entropies.}
    \label{fig:B_entropy_attention}
\end{figure}

\newpage
\textbf{Number of observations does not affect attention window.} We fix the HMM config to 3 states with A entropy 0.25 and B entropy 0.75. We ablate the number of observations, $\{3, 6, 12\}$. The results are shown in Figure~\ref{fig:num_obs_metrics}. As we increase the number of observations, again the training and validation losses become worse as the current task becomes harder. The accuracy also decreases with more number of observations. However, the attention visualizations show in Figure~\ref{fig:num_obs_attention} for different number of observations remain consistent. Attention pattern could be highly correlated with the entropies but not the number of observations.

\begin{figure}[h]
  \centering
    \includegraphics[width=\linewidth]{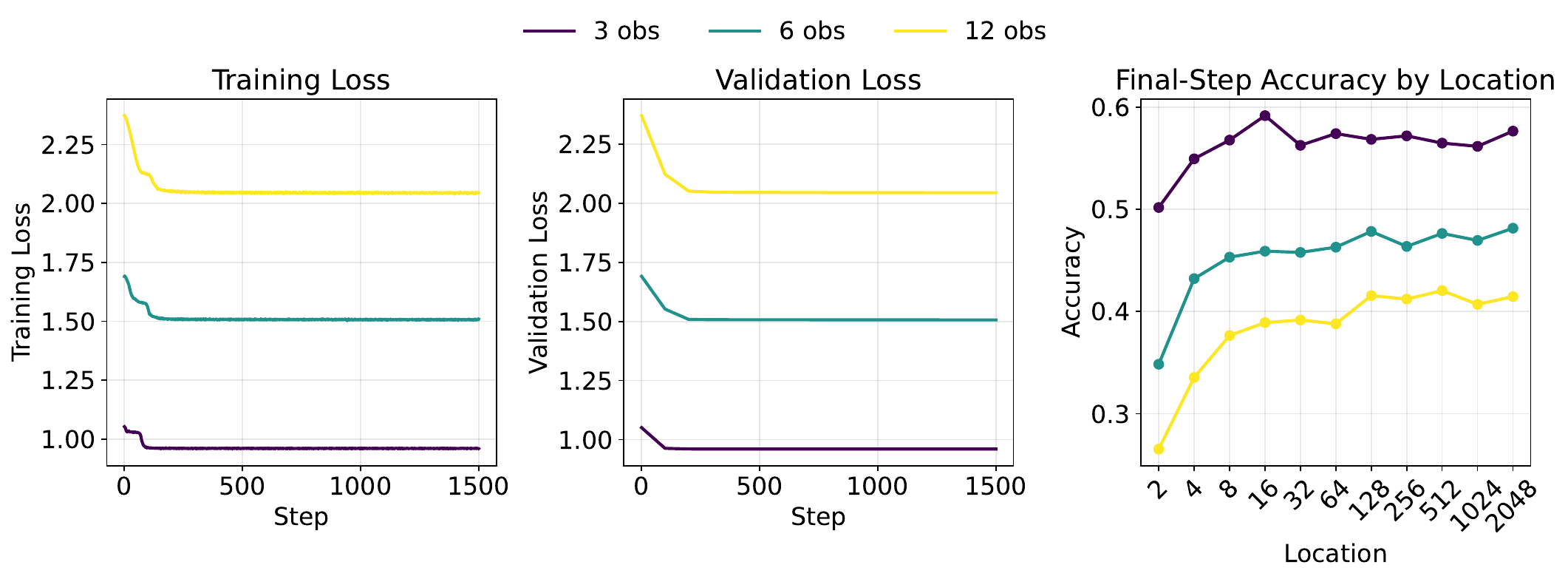}
  \caption{Training loss, validation loss, and accuracies with different number of observations.}
  \label{fig:num_obs_metrics}
\end{figure}

\begin{figure}[h]
    \centering
    \includegraphics[width=\linewidth]{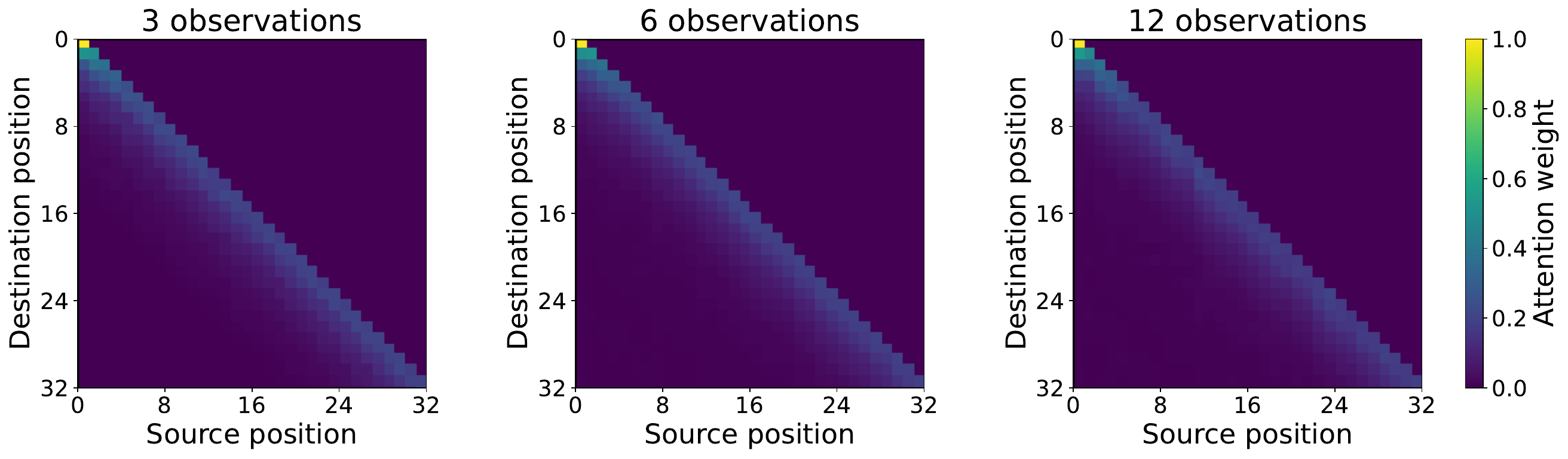}
    \caption{Attention pattern for different number of observations.}
    \label{fig:num_obs_attention}
\end{figure}

\subsection{Ablation on Window Size $n$}
\label{app:window_size}

\begin{figure}[h]
    \centering
    \includegraphics[width=0.4\textwidth]{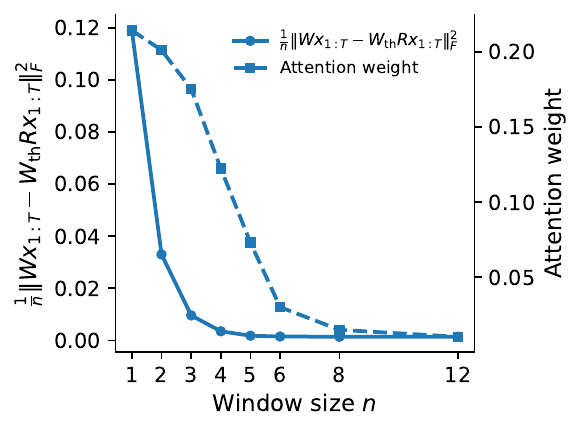}
    \vspace{-1em}
    \caption{Discrepancy between the two mappings and attention weights.}
    \label{fig:window_size_mse}
\end{figure}

\textbf{The effective attention window matches the finite-window predictor.}
We further vary the window size $n$ used by the theoretical finite-window predictor and measure the discrepancy between the direct hidden-to-logit map $Wx_{1:T}$ and the composed theoretical map $W_{\mathrm{th}}Rx_{1:T}$. Figure~\ref{fig:window_size_mse} shows that this discrepancy decreases rapidly as $n$ increases, and converges around $n=6$. This is consistent with the attention pattern, where the attention mass concentrates mostly on the previous $6$ tokens.

\subsection{Linear vs.\ Kronecker $n$-gram Feature Mappings}
\label{app:feature_compare}

\begin{figure}[h]
      \centering
      \begin{subfigure}{0.48\textwidth}
          \centering
          \includegraphics[width=\textwidth]{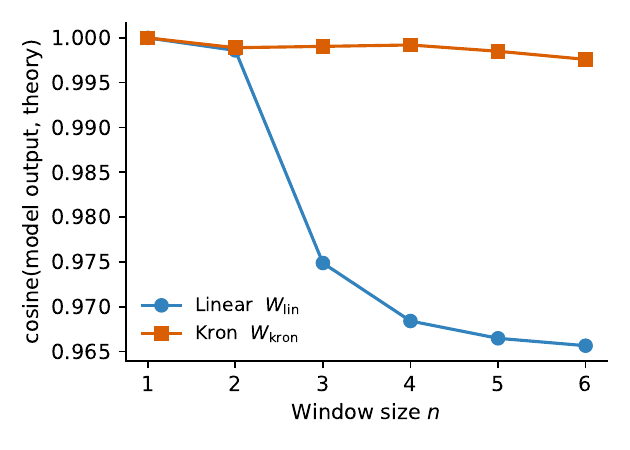}
          \caption{Output alignment}
          \label{fig:output_align}
      \end{subfigure}
      \hfill
      \begin{subfigure}{0.48\textwidth}
          \centering
          \includegraphics[width=\textwidth]{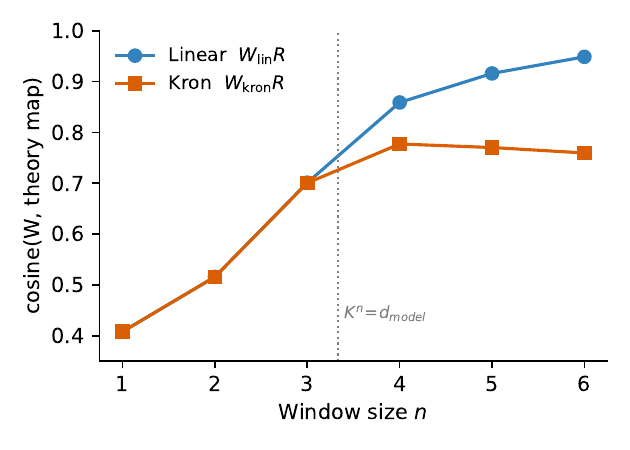}
          \caption{Representation alignment}
          \label{fig:repr_align}
      \end{subfigure}
      \vspace{-0.5em}
      \caption{\textbf{Output vs.\ representation alignment for linear (additive,
      dimension $Kn$) and Kronecker (tabular, dimension $K^n$) $n$-gram features.}
      (a) Cosine similarity between the model's next-token logits and each
      theoretical predictor, aggregated by length-$n$ context. (b) Cosine
      similarity between the learned hidden-to-logit readout $W$ and the
      theory-induced map $W_{\mathrm{th}}R$ at the post-attention residual stream.}
      \label{fig:features_compare}
  \end{figure}

\textbf{The transformer predicts like the tabular model, but its post-attention representation is additive.} We contrast two ways of representing the length-$n$ context: the \emph{linear} features, which stack the one-hot encodings of the past $n$ observations (dimension $Kn$), and the \emph{Kronecker} features, which form their full Cartesian product (dimension $K^n$, one coordinate per distinct context tuple). We probe these features at two levels. For \emph{output alignment} (Figure~\ref{fig:output_align}), we aggregate the model's next-token logits by context and measure the cosine similarity to each theory's per-context prediction ($W_{\mathrm{kron}}$, and the additive $W_{\mathrm{lin}}$ projected onto contexts). For \emph{representation alignment} (Figure~\ref{fig:repr_align}), we fit a linear map $R$ from the post-attention residual stream $x_{1:T}$ to the features and measure the cosine similarity between the directly learned readout $W$ and the composed map $W_{\mathrm{th}}R$. The probe $R$ is fit by ridge regression on every context position of the held-out validation split---$4096$ sequences of length $2049$, i.e.\ $4096\times(2049-n)\approx 8.4\times10^{6}$ state/feature pairs---so the fit is heavily over-determined relative to its effective rank $d_{\mathrm{model}}=64$.

At the output level, the model's predictions match the Kronecker (tabular) predictor almost perfectly at every window (cosine $\approx 1.0$), while the linear predictor's agreement decays with $n$ (down to $\approx 0.97$) as the additive restriction discards increasingly relevant context. At the representation level the ordering reverses: the linear mapping keeps improving (cosine $\to 0.95$), whereas the Kronecker mapping saturates at $\approx 0.76$ once the number of tabular contexts $K^n$ exceeds the model width $d_{\mathrm{model}}$; because the probe is trained on millions of examples, this gap is not data-limited but reflects the capacity of the residual stream.

\textbf{This double dissociation is exactly what our theory predicts.} The output is tabular-optimal (aligned with the Kronecker predictor), yet the post-attention representation is additive (it linearly encodes only the linear features). By Lemma~\ref{lemma:kron_prediction_main}, the linear $n$-gram features cannot be mapped to the Kronecker features---and hence to the tabular-optimal next-token distribution---by a linear transformation alone; a non-linear layer ($\mathrm{ReLU}$) is required. The attention block therefore assembles the additive features into the residual stream, and the subsequent non-linear MLP lifts them to the tabular prediction, which is why alignment is \emph{linear} before the MLP but \emph{Kronecker} at the output. Appendix~\ref{app:when_linear_acc} characterizes the regimes in which the linear-feature approximation already yields accurate output predictions, consistent with the residual ($\approx 0.97$) agreement of the linear predictor in Figure~\ref{fig:output_align}.

%% file: appendix_llm_emp.tex
We present detailed evaluation results for the experiment described in Section \ref{sec:llm_emp}. Recall the setup: for each HMM configuration, we sample observation sequences $\mathbf{o}_{1:T}$ from $\boldsymbol{\lambda} = (\boldsymbol{\pi}, \mathbf{A}, \mathbf{B})$ and evaluate next-observation prediction $o_{t+1}$ given $\mathbf{o}_{1:t}$. We fix number of hidden state $M = 4$ and vary three controls across $75$ configurations: transition entropy $\mathcal{E}(\mathbf{A}) \in \{0,\,0.25,\,0.5,\,0.75,\,1\}$, inducing mixing rates $|\lambda_2(\mathbf{A})| \in \{1, 0.9, 0.75, 0.5, 0\}$; emission entropy $\mathcal{E}(\mathbf{B}) \in \{0,\,0.25,\,0.5,\,0.75,\,1\}$, ranging from injective (when $N \geq M$) to uniform; and alphabet size $N \in \{2, 4, 8\}$, where $N < M$ induces hidden-state aliasing. The initial distribution is set to the stationary distribution, $\boldsymbol{\pi} = \boldsymbol{\mu}$. For each configuration, we sample $4{,}096$ sequences of length $4{,}097$ and evaluate at context lengths $\{4, 8, \ldots, 4096\}$.

We report two metrics averaged over $4{,}096$ samples per HMM configuration. \emph{Accuracy} is the fraction of positions at which the argmax prediction matches the ground truth. We additionally report Hellinger distance between predicted and reference next-token distributions, used in two ways: \emph{Hellinger-to-oracle} measures deviation from the Bayes-optimal posterior; and \emph{Hellinger-to-baseline} measures distributional similarity to each algorithmic baseline. In both cases, smaller values indicate better performance.

We evaluate $12$ open-weight pretrained LLMs drawn from the Qwen, Llama, Gemma, and OLMo families, ranging from $0.6$B to $8$B parameters, without any fine-tuning. Each model receives raw observation tokens via a fixed alphabet-to-vocabulary mapping; at every position, we restrict the next-token logits to the observation vocabulary and renormalize using softmax.

The report is separated into the following subsections:
\begin{itemize}
\item Appendix~\ref{app:llm_family}: Empirical Performance of Different LLM Families
\item Appendix~\ref{app:llm_fail_to_converge}: Cases When LLMs Fail to Converge to Oracle
\item Appendix~\ref{app:baseline_compare}: Baseline Comparison
\item Appendix~\ref{app:more_state}: More Hidden States ($M=12$)
\end{itemize}

\subsection{Empirical Performance of Different LLM Families}
\label{app:llm_family}
For each LLM family, we present three plots corresponding to alphabet sizes $N \in \{2, 4, 8\}$. Each plot contains $25$ subplots arranged over a grid of transition entropies $\mathcal{E}(\mathbf{A}) \in \{0,\,0.25,\,0.5,\,0.75,\,1\}$ and emission entropies $\mathcal{E}(\mathbf{B}) \in \{0,\,0.25,\,0.5,\,0.75,\,1\}$. Within each subplot, the upper panel shows top-1 prediction accuracy and the lower panel shows Hellinger-to-oracle distance, both as a function of context length. The \texttt{oracle} and \texttt{bigram} baselines are included in each subplot for reference.

Among the families, the Qwen models (Qwen2.5-1.5B, Qwen3-0.6B, Qwen3-1.7B, Qwen3-4B, Qwen3-8B) exhibit the most consistent convergence across model sizes, with larger models achieving higher accuracy at longer context lengths and converging closer to the oracle distribution. The Llama and Gemma families generally converge but show a larger gap from the oracle than Qwen under high-entropy conditions, with Llama3.1-8B suffering a sharp performance drop at context length $4096$. The OLMo family (OLMo2-1B, OLMo2-7B) performs the worst overall: it degrades markedly at longer context lengths ($T \in \{1024, 2048, 4096\}$), exhibiting a pronounced U-shaped trend, and underperforms other families when either transition or emission entropies is high.

\newpage
\begin{figure}[H]
    \centering
    \includegraphics[width=\linewidth]{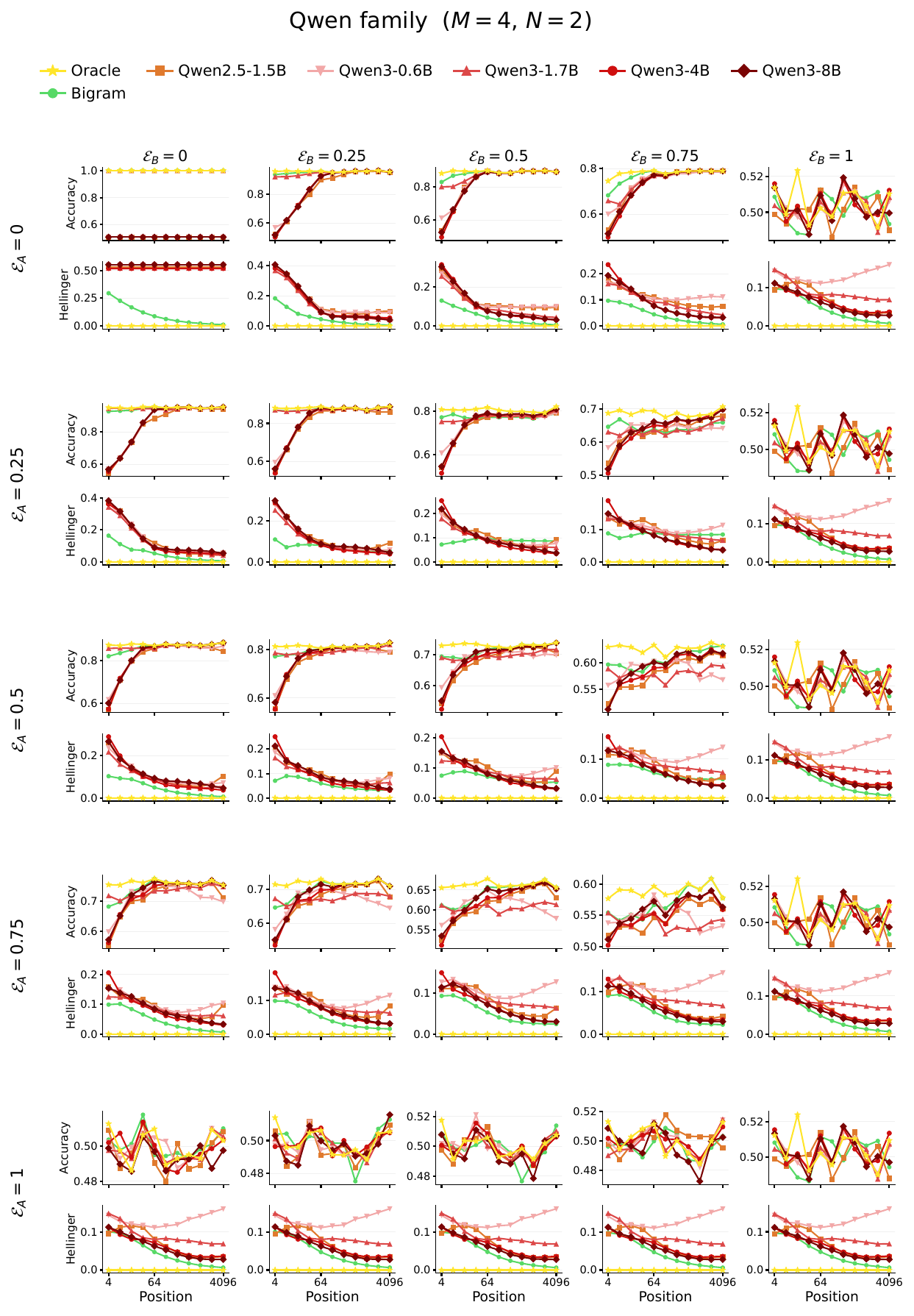}
    \caption{Qwen family LLMs evaluated on HMMs with $M=4$, $N=2$. Subplots vary transition entropy $\mathcal{E}(\mathbf{A})$ (rows) and emission entropy $\mathcal{E}(\mathbf{B})$ (columns); each shows top-1 accuracy (upper) and Hellinger-to-oracle distance (lower) versus context length, with \texttt{oracle} and \texttt{bigram} baselines.}
    \label{fig:qwen_m4_n2}
\end{figure}

\begin{figure}[H]
    \centering
    \includegraphics[width=\linewidth]{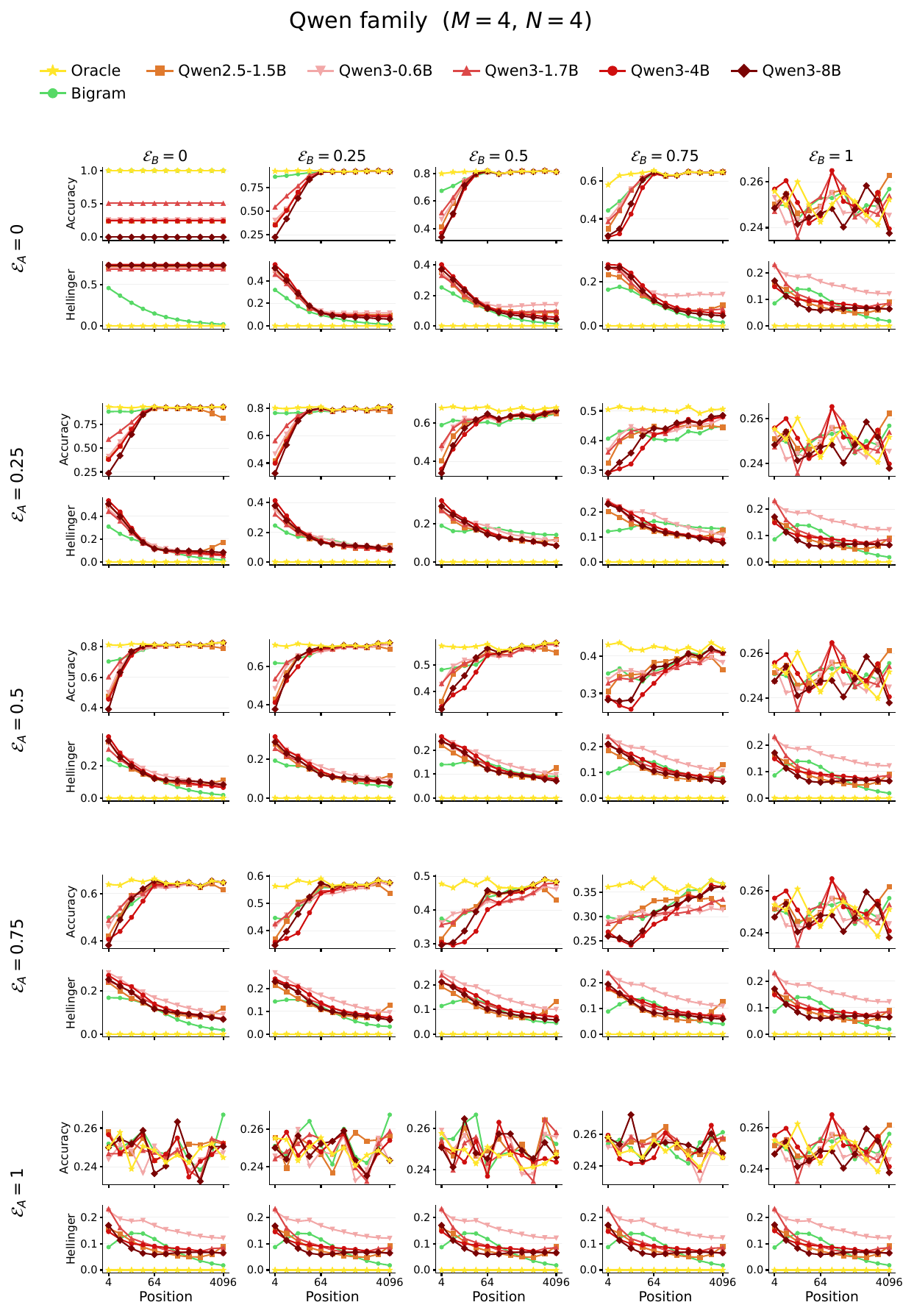}
    \caption{Qwen family LLMs evaluated on HMMs with $M=4$, $N=4$. Subplots vary transition entropy $\mathcal{E}(\mathbf{A})$ (rows) and emission entropy $\mathcal{E}(\mathbf{B})$ (columns); each shows top-1 accuracy (upper) and Hellinger-to-oracle distance (lower) versus context length, with \texttt{oracle} and \texttt{bigram} baselines.}
    \label{fig:qwen_m4_n4}
\end{figure}

\begin{figure}[H]
    \centering
    \includegraphics[width=\linewidth]{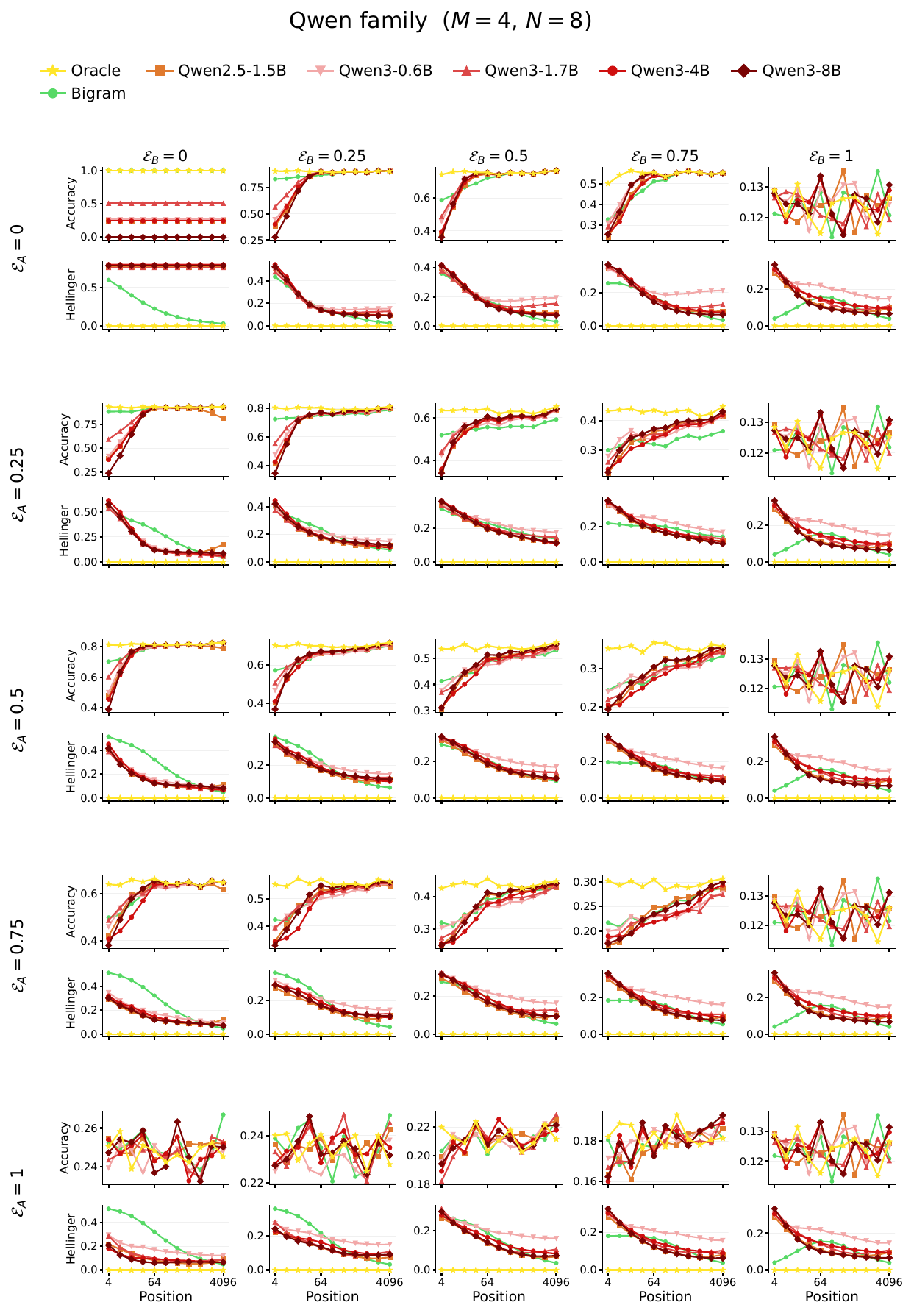}
    \caption{Qwen family LLMs evaluated on HMMs with $M=4$, $N=8$. Subplots vary transition entropy $\mathcal{E}(\mathbf{A})$ (rows) and emission entropy $\mathcal{E}(\mathbf{B})$ (columns); each shows top-1 accuracy (upper) and Hellinger-to-oracle distance (lower) versus context length, with \texttt{oracle} and \texttt{bigram} baselines.}
    \label{fig:qwen_m4_n8}
\end{figure}

\begin{figure}[H]
    \centering
    \includegraphics[width=\linewidth]{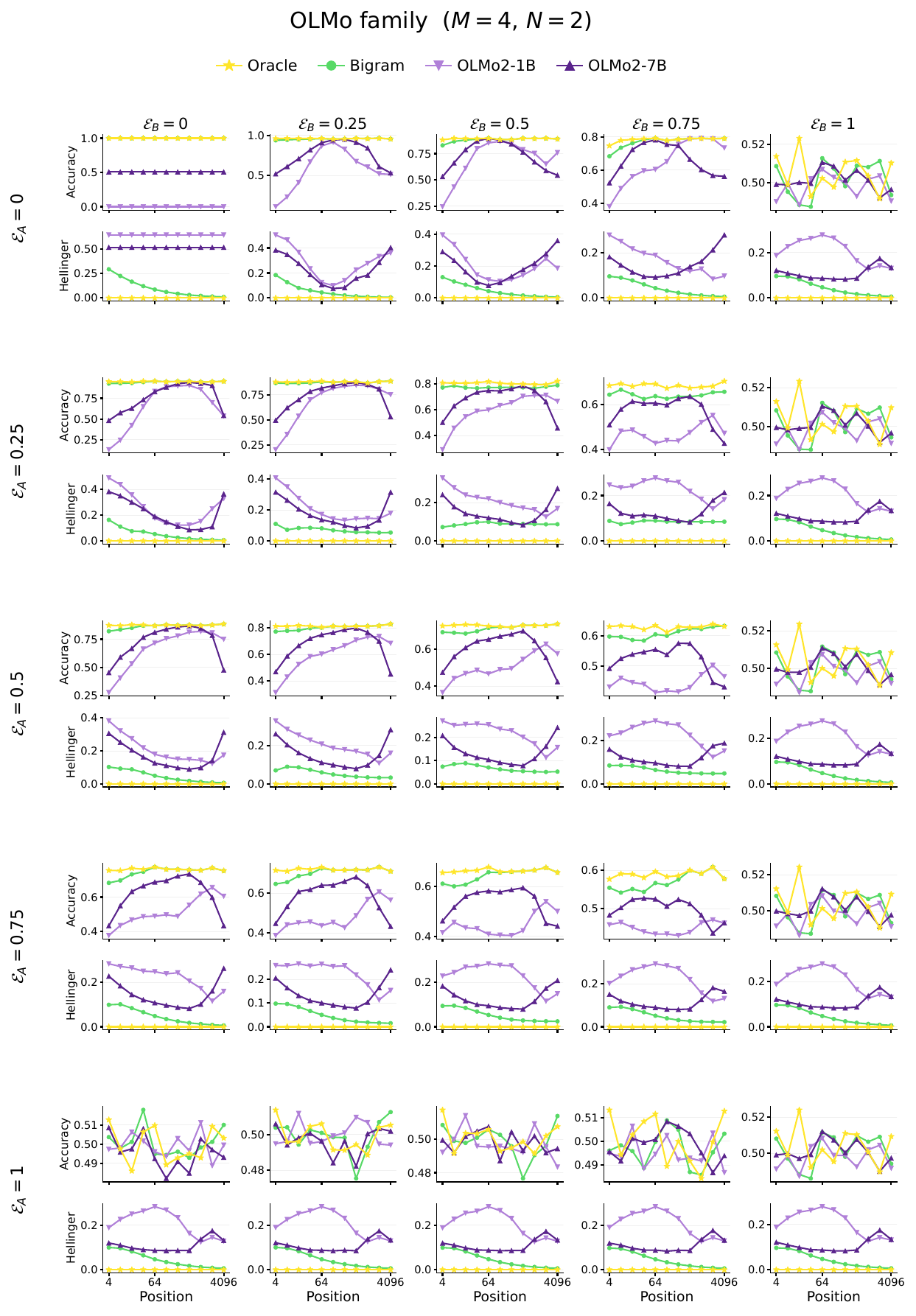}
    \caption{OLMo family LLMs evaluated on HMMs with $M=4$, $N=2$. Subplots vary transition entropy $\mathcal{E}(\mathbf{A})$ (rows) and emission entropy $\mathcal{E}(\mathbf{B})$ (columns); each shows top-1 accuracy (upper) and Hellinger-to-oracle distance (lower) versus context length, with \texttt{oracle} and \texttt{bigram} baselines.}
    \label{fig:olmo_m4_n2}
\end{figure}

\begin{figure}[H]
    \centering
    \includegraphics[width=\linewidth]{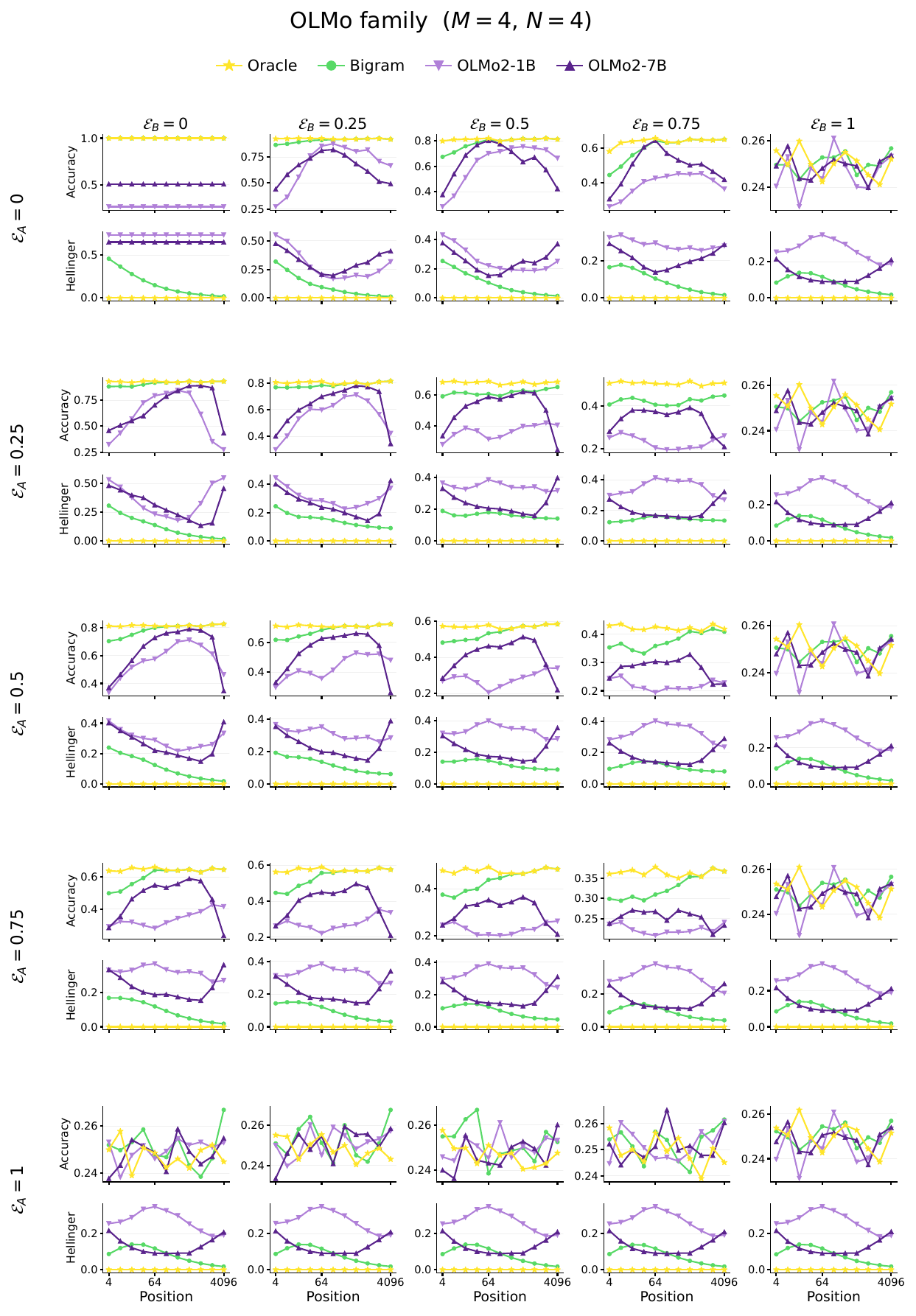}
    \caption{OLMo family LLMs evaluated on HMMs with $M=4$, $N=4$. Subplots vary transition entropy $\mathcal{E}(\mathbf{A})$ (rows) and emission entropy $\mathcal{E}(\mathbf{B})$ (columns); each shows top-1 accuracy (upper) and Hellinger-to-oracle distance (lower) versus context length, with \texttt{oracle} and \texttt{bigram} baselines.}
    \label{fig:olmo_m4_n4}
\end{figure}

\begin{figure}[H]
    \centering
    \includegraphics[width=\linewidth]{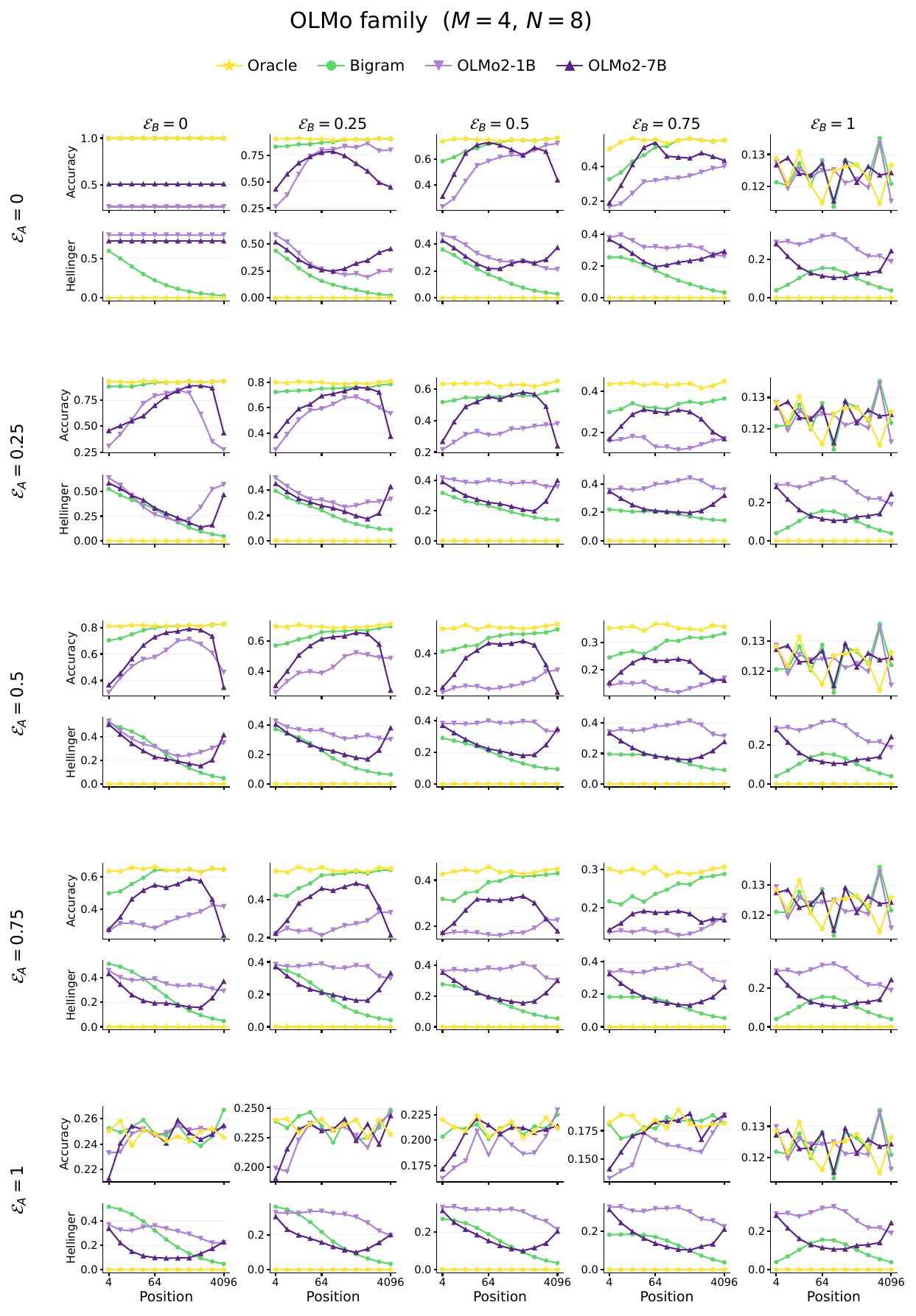}
    \caption{OLMo family LLMs evaluated on HMMs with $M=4$, $N=8$. Subplots vary transition entropy $\mathcal{E}(\mathbf{A})$ (rows) and emission entropy $\mathcal{E}(\mathbf{B})$ (columns); each shows top-1 accuracy (upper) and Hellinger-to-oracle distance (lower) versus context length, with \texttt{oracle} and \texttt{bigram} baselines.}
    \label{fig:olmo_m4_n8}
\end{figure}

\begin{figure}[H]
    \centering
    \includegraphics[width=\linewidth]{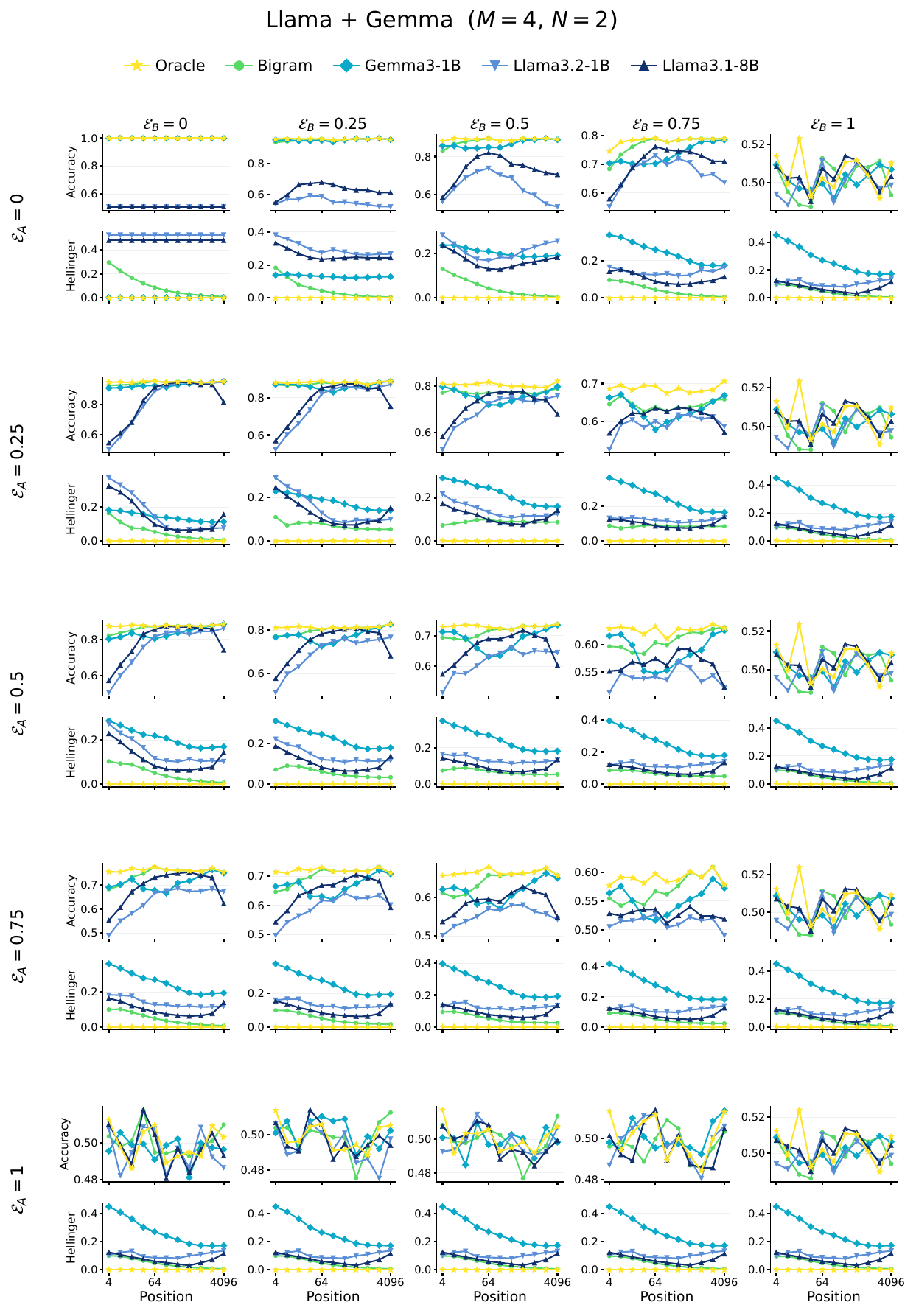}
    \caption{Llama and Gemma family LLMs evaluated on HMMs with $M=4$, $N=2$. Subplots vary transition entropy $\mathcal{E}(\mathbf{A})$ (rows) and emission entropy $\mathcal{E}(\mathbf{B})$ (columns); each shows top-1 accuracy (upper) and Hellinger-to-oracle distance (lower) versus context length, with \texttt{oracle} and \texttt{bigram} baselines.}
    \label{fig:llama_gemma_m4_n2}
\end{figure}

\begin{figure}[H]
    \centering
    \includegraphics[width=\linewidth]{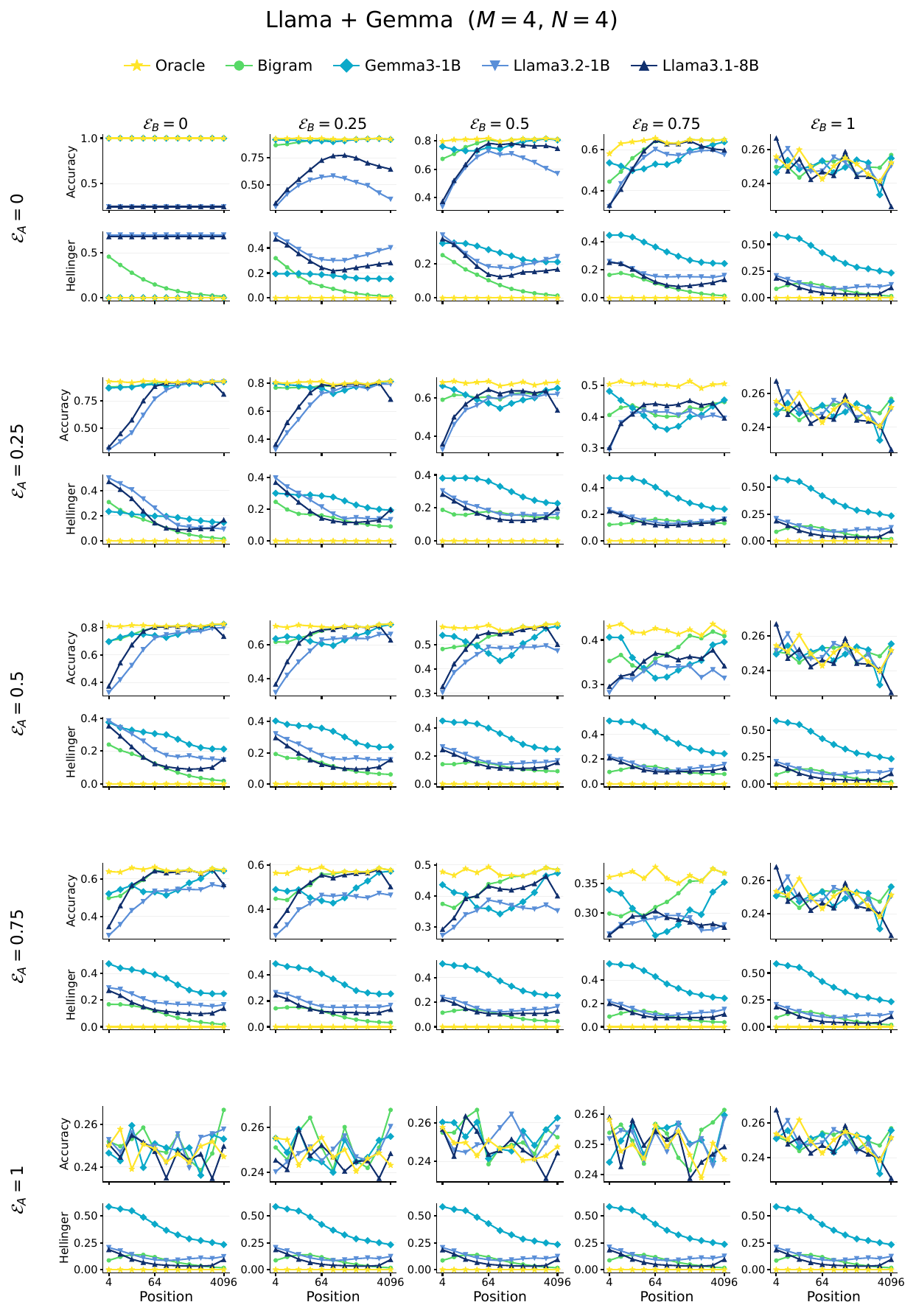}
    \caption{Llama and Gemma family LLMs evaluated on HMMs with $M=4$, $N=4$. Subplots vary transition entropy $\mathcal{E}(\mathbf{A})$ (rows) and emission entropy $\mathcal{E}(\mathbf{B})$ (columns); each shows top-1 accuracy (upper) and Hellinger-to-oracle distance (lower) versus context length, with \texttt{oracle} and \texttt{bigram} baselines.}
    \label{fig:llama_gemma_m4_n4}
\end{figure}

\begin{figure}[H]
    \centering
    \includegraphics[width=\linewidth]{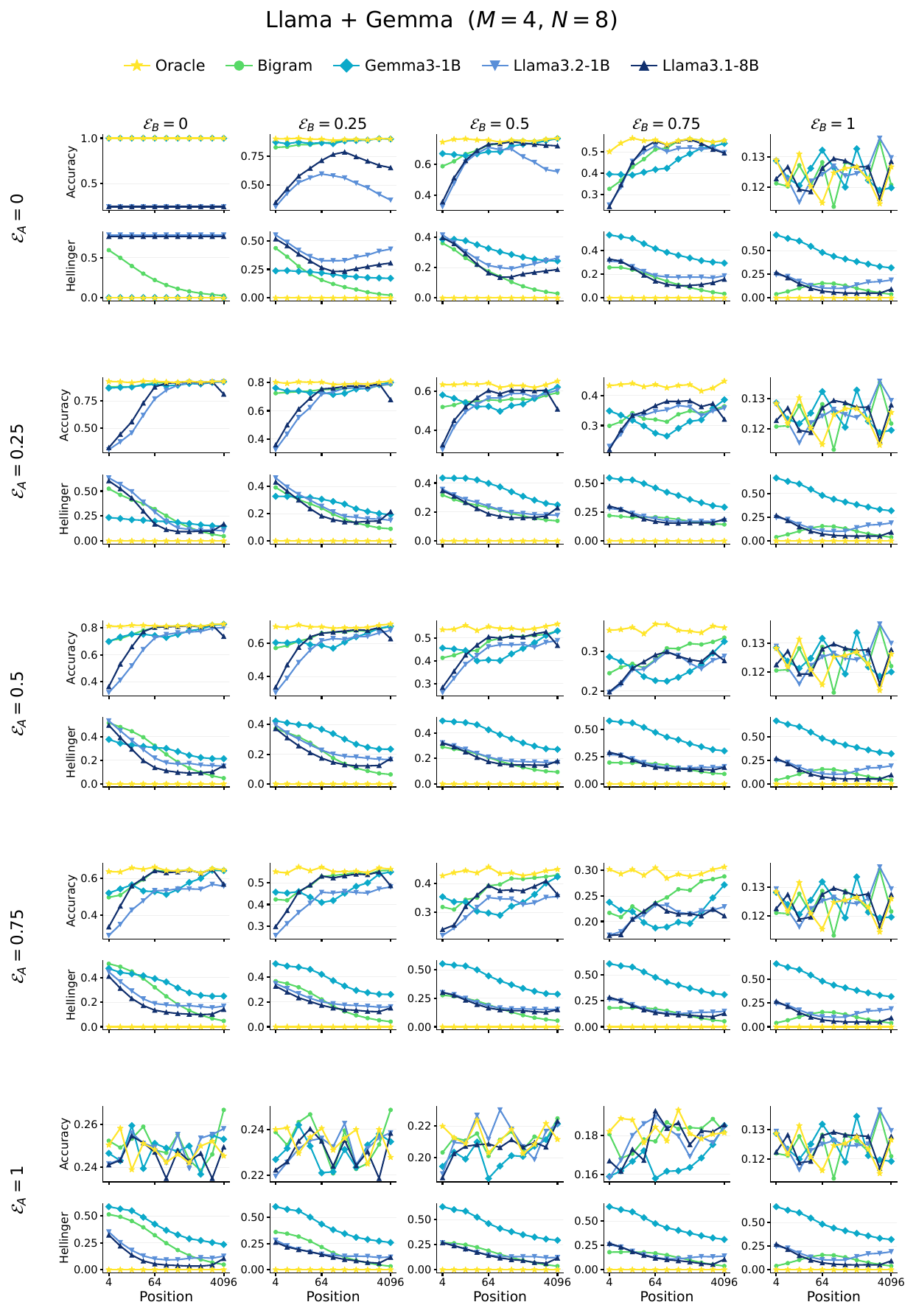}
    \caption{Llama and Gemma family LLMs evaluated on HMMs with $M=4$, $N=8$. Subplots vary transition entropy $\mathcal{E}(\mathbf{A})$ (rows) and emission entropy $\mathcal{E}(\mathbf{B})$ (columns); each shows top-1 accuracy (upper) and Hellinger-to-oracle distance (lower) versus context length, with \texttt{oracle} and \texttt{bigram} baselines.}
    \label{fig:llama_gemma_m4_n8}
\end{figure}

\subsection{Cases When LLMs Fail to Converge to Oracle}
\label{app:llm_fail_to_converge}
We first note an edge case when both transition entropy $\mathcal{E}(\mathbf{A})$ and emission entropy $\mathcal{E}(\mathbf{B})$ are zero. In this setting, the observation sequence reduces to a constant string (e.g., ``AAAA\ldots''), which should be trivially easy to predict. Nevertheless, most models fail to continue predicting the repeated token, likely as an artifact of natural language pretraining \cite{yona2025interpreting}. We exclude this configuration from subsequent analysis.

More broadly, convergence to the \texttt{oracle} can be assessed jointly via top-1 accuracy and Hellinger-to-oracle distance. We report both metrics across the remaining $74$ configurations in Figure~\ref{fig:suboptimal}. As discussed earlier, certain model families exhibit systematic suboptimality. Focusing on a strong performer such as Qwen3-8B, we find no clear correlation between suboptimality and entropy, though the Hellinger gap shows a mild increase as the number of observations grows.

\begin{figure}
    \centering
    \includegraphics[width=1.08\linewidth]{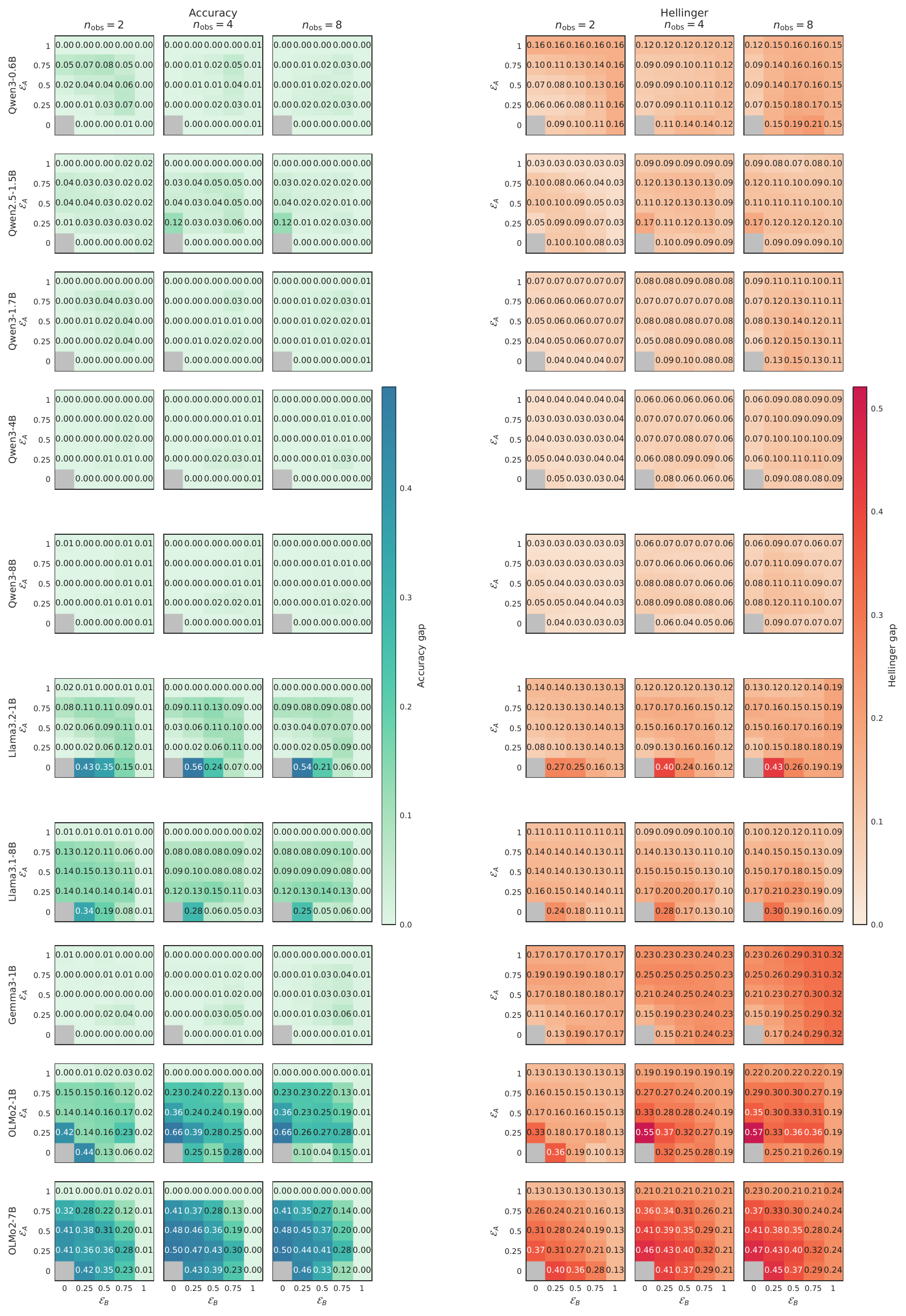}
    \caption{LLM suboptimality vs. the Bayes-optimal \texttt{oracle} at position $t{=}4096$ ($M{=}4$). Per model (rows): accuracy gap (left) and Hellinger distance to the oracle posterior (right); lower is better, $0$ = matches oracle. Columns are $n_{\mathrm{obs}}\in\{2,4,8\}$; each cell is a $5\times5$ sweep of transition entropy $\mathcal{E}_A$ (vertical) vs. emission entropy $\mathcal{E}_B$ (horizontal). Darker color means more performance gap.}
    \label{fig:suboptimal}
\end{figure}

\subsection{Baseline Comparison}
\label{app:baseline_compare}

We compare LLM performance against the baseline algorithms described in Section~\ref{sec:alg_descriptions} and Appendix~\ref{app:bench}, grouping them into the same four classes as in Figure~\ref{fig:llm_empirical}. All baselines are tuned via extensive hyperparameter search to ensure fair comparison. Specifically, Baum-Welch is run with $10$ random restarts and up to $200$ iterations each; the linear and non-linear $n$-gram models are tuned over the order $n$ and L2 regularization strength; and the spectral algorithms apply smoothing to the estimated bigram and trigram statistics.

The classical Baum-Welch algorithm (Figures~\ref{fig:classic_m4n2}, \ref{fig:classic_m4n4}, \ref{fig:classic_m4n8}) performs strongly on this task, but its behavior differs markedly from LLMs. Most notably, it achieves high accuracy even at very short context lengths and maintains stable performance as context grows, whereas LLMs typically require longer context to converge. Under high emission entropy, its accuracy becomes less stable across context lengths. Both observations reflect the iterative nature of Baum-Welch: given the number of hidden states $M$, it fits HMM parameters directly to the observed sequence, so its performance is governed by parameter estimation quality rather than in-context sequence modeling.

The linear and non-linear $n$-gram variants (Figures~\ref{fig:linear_m4n2}, \ref{fig:linear_m4n4}, \ref{fig:linear_m4n8}, \ref{fig:nonlinear_m4n2}, \ref{fig:nonlinear_m4n4}, \ref{fig:nonlinear_m4n8}) exhibit in-context learning behavior qualitatively similar to LLMs, with accuracy steadily improving over longer contexts. There are, however, notable differences among variants: the linear $n$-gram model trained with MSE loss fails to converge to the LLM or oracle performance, whereas the cross-entropy and ridge regression variants do.

The spectral algorithms (Figures~\ref{fig:spectral_m4n2}, \ref{fig:spectral_m4n4}, \ref{fig:spectral_m4n8}) show a somewhat different profile from the $n$-gram variants, tending to exhibit a U-shaped accuracy curve with more pronounced improvement at longer context lengths. This pattern is consistent with the fact that reliable estimation of the underlying bigram and trigram statistics requires a larger number of observations, suggesting that the spectral methods may become more competitive as context grows.

Taken together, none of the baseline algorithms can be confidently identified as matching LLM behavior across all HMM configurations. The $n$-gram variants and spectral methods each resemble LLMs in certain regimes but diverge in others. Baum-Welch, however, can be confidently ruled out: its ability to achieve strong performance from very short contexts reflects a fundamentally different learning mechanism rather than the gradual in-context generalization characteristic of LLMs.

\begin{figure}
    \centering
    \includegraphics[width=\linewidth]{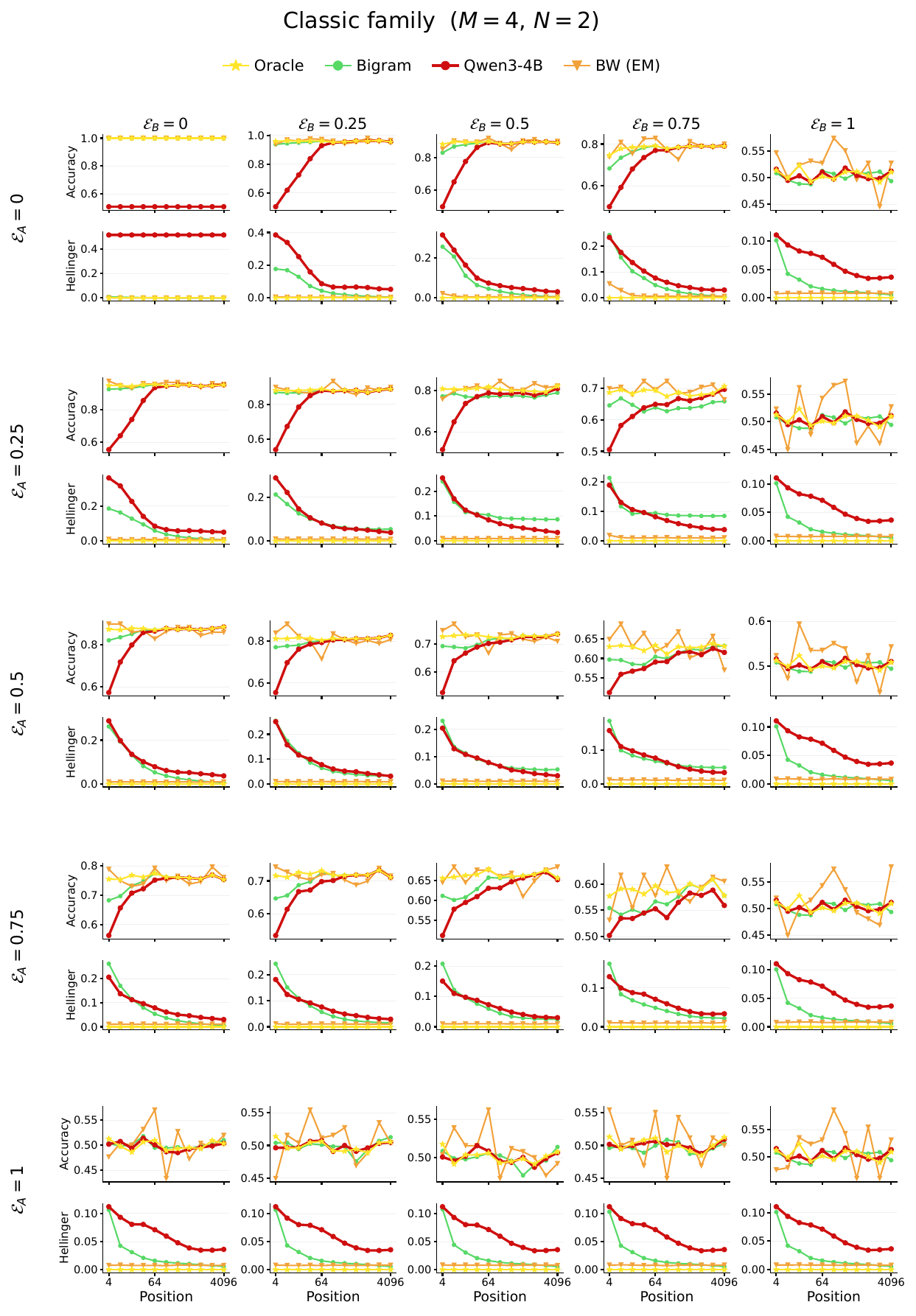}
    \caption{Grid over transition entropy $\mathcal{E}_A$ (rows) and emission entropy $\mathcal{E}_B$ (columns); each cell shows top-1 accuracy (top) and Hellinger distance to the oracle posterior (bottom) as a function of sequence position. Higher accuracy / lower Hellinger is better.}
    \label{fig:classic_m4n2}
\end{figure}

\begin{figure}
    \centering
    \includegraphics[width=\linewidth]{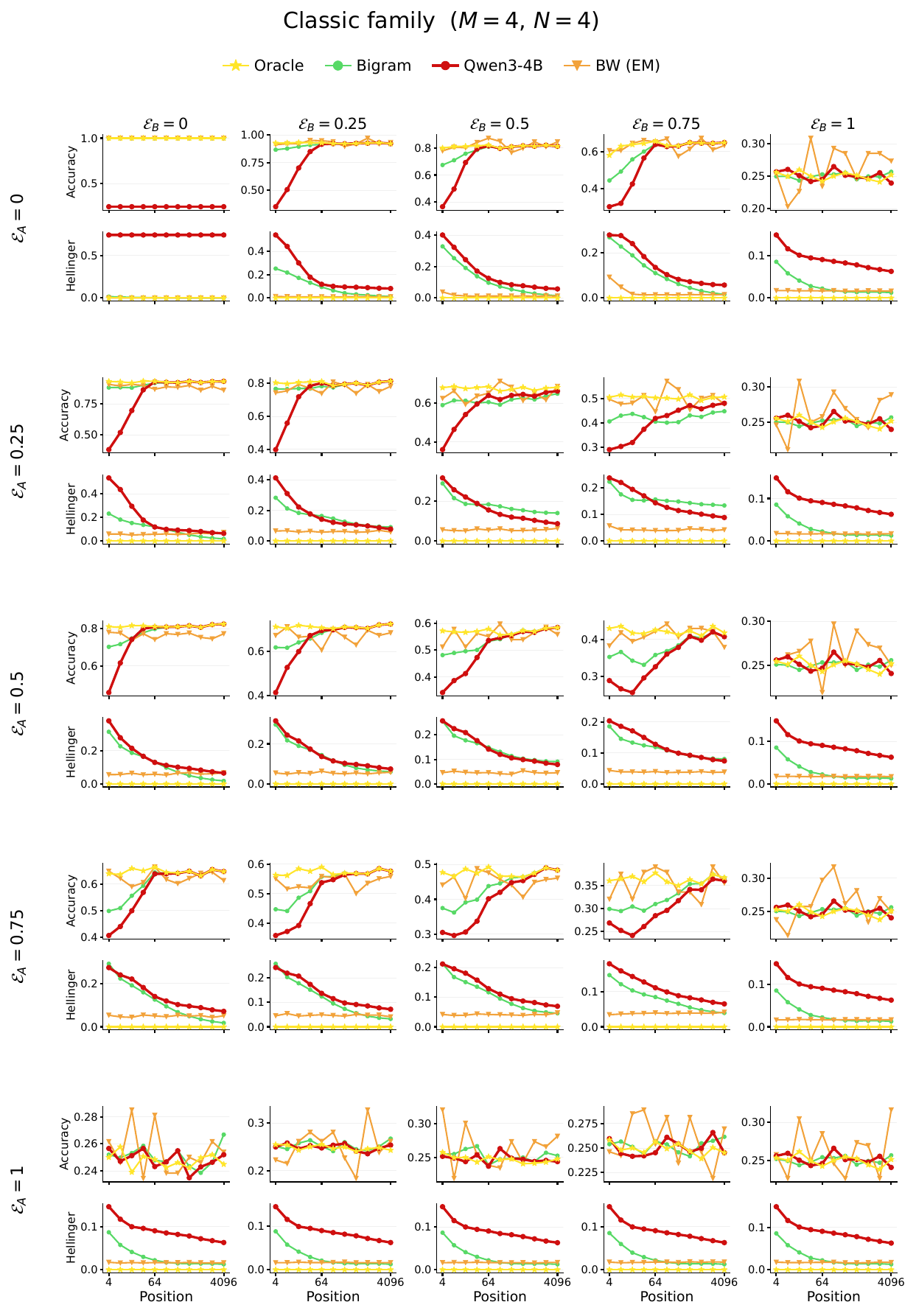}
    \caption{Grid over transition entropy $\mathcal{E}_A$ (rows) and emission entropy $\mathcal{E}_B$ (columns); each cell shows top-1 accuracy (top) and Hellinger distance to the oracle posterior (bottom) as a function of sequence position. Higher accuracy / lower Hellinger is better.}
    \label{fig:classic_m4n4}
\end{figure}

\begin{figure}
    \centering
    \includegraphics[width=\linewidth]{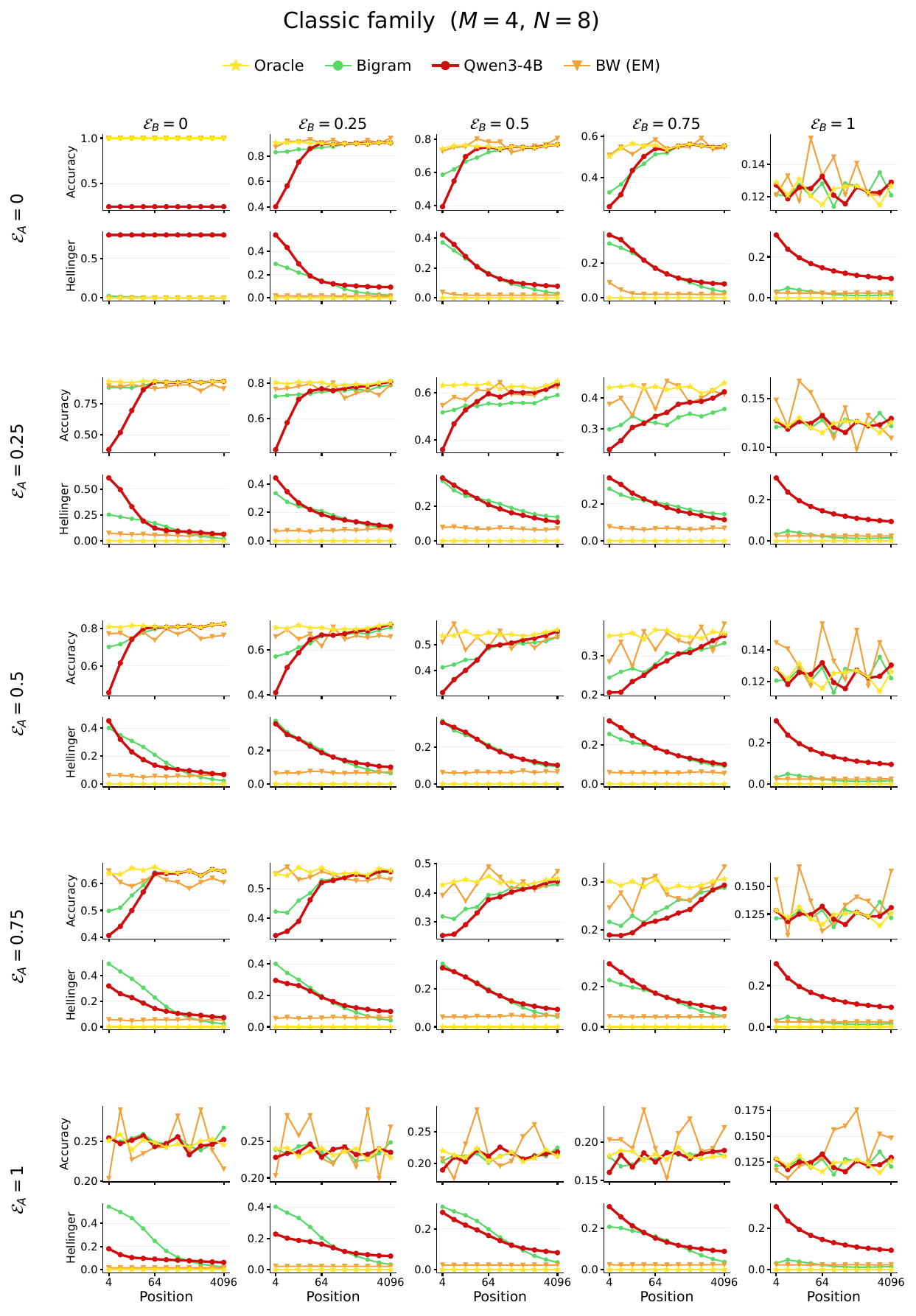}
    \caption{Grid over transition entropy $\mathcal{E}_A$ (rows) and emission entropy $\mathcal{E}_B$ (columns); each cell shows top-1 accuracy (top) and Hellinger distance to the oracle posterior (bottom) as a function of sequence position. Higher accuracy / lower Hellinger is better.}
    \label{fig:classic_m4n8}
\end{figure}

\begin{figure}
    \centering
    \includegraphics[width=\linewidth]{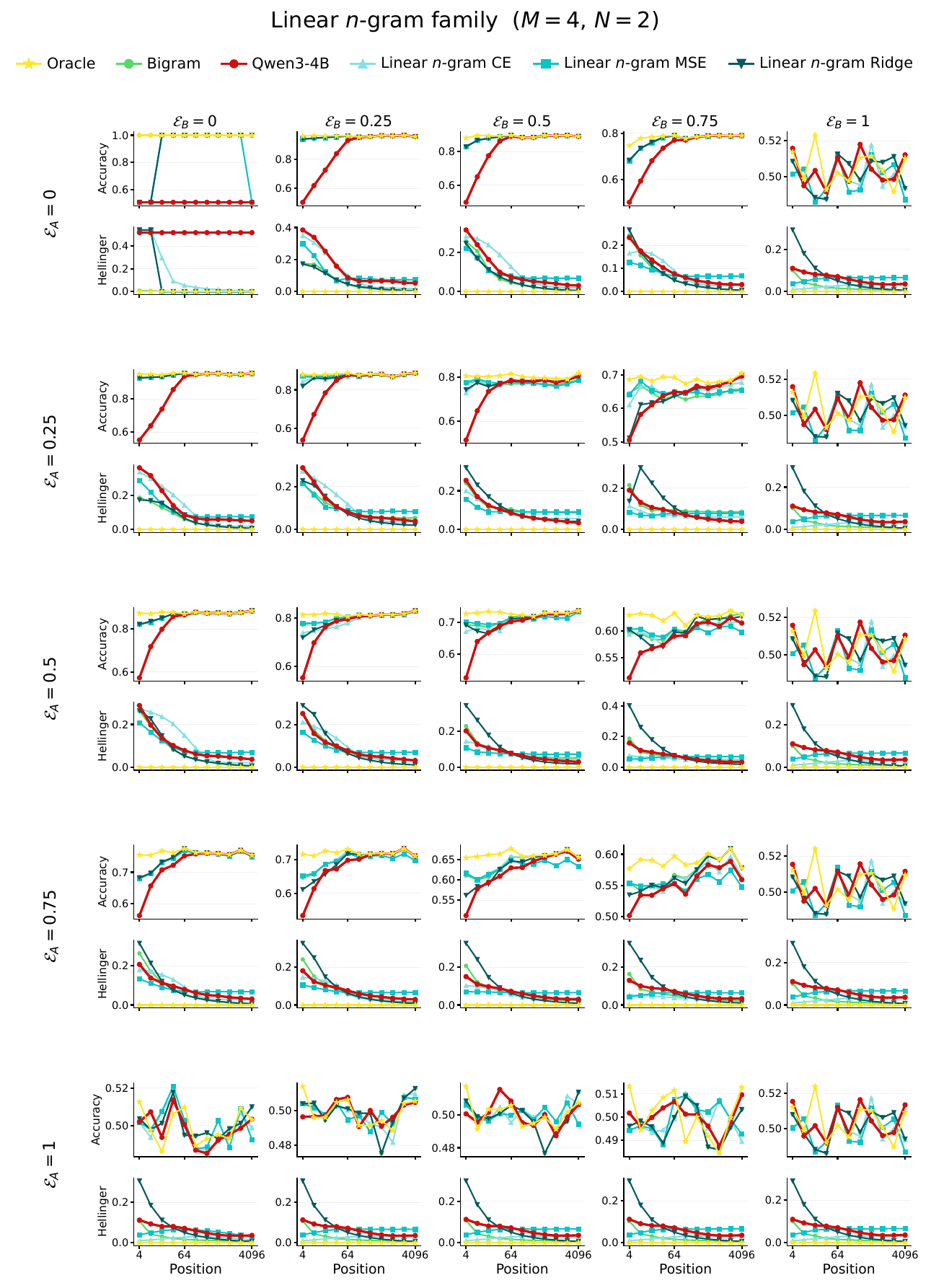}
    \caption{Grid over transition entropy $\mathcal{E}_A$ (rows) and emission entropy $\mathcal{E}_B$ (columns); each cell shows top-1 accuracy (top) and Hellinger distance to the oracle posterior (bottom) as a function of sequence position. Higher accuracy / lower Hellinger is better.}
    \label{fig:linear_m4n2}
\end{figure}

\begin{figure}
    \centering
    \includegraphics[width=\linewidth]{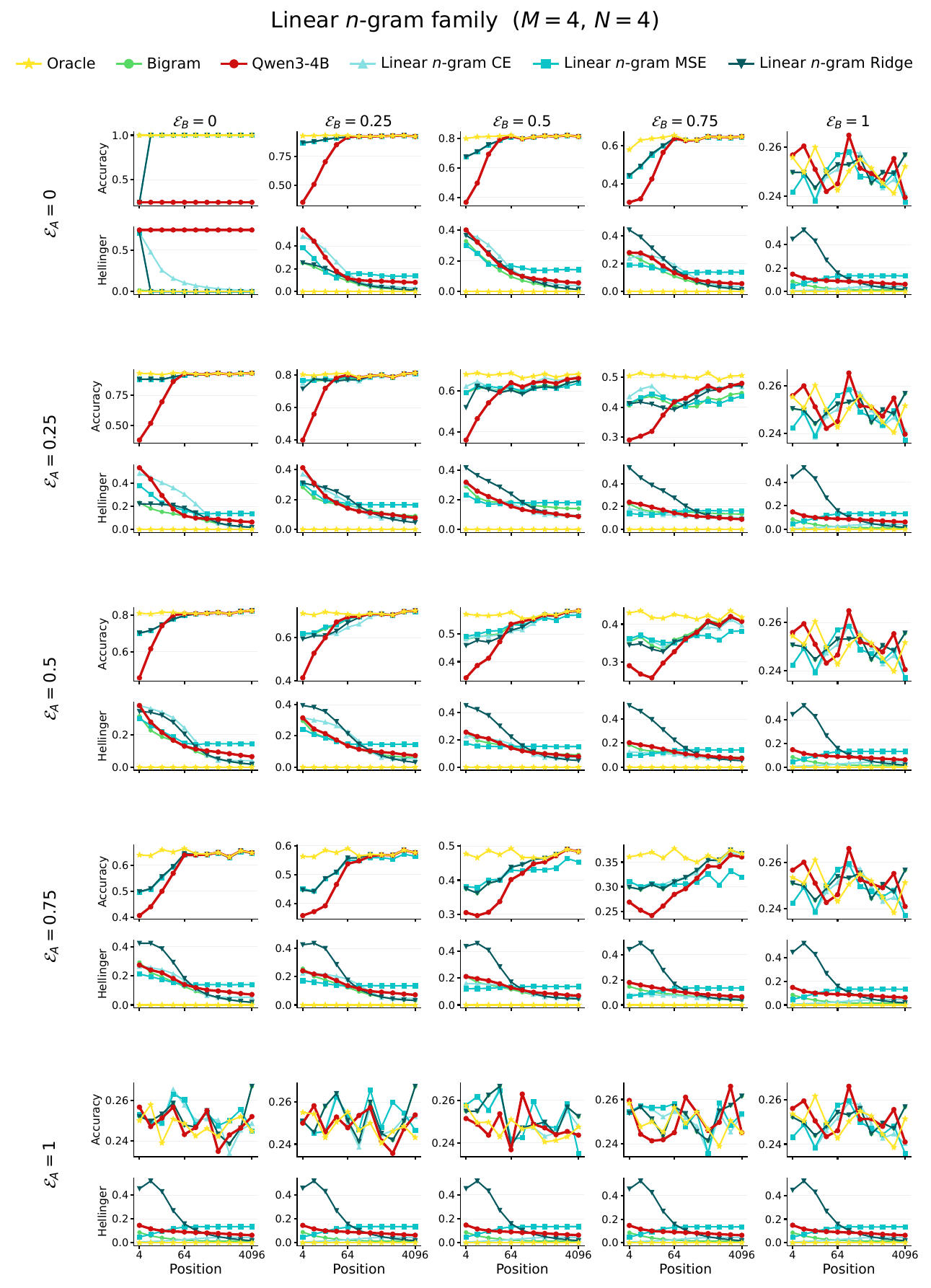}
    \caption{Grid over transition entropy $\mathcal{E}_A$ (rows) and emission entropy $\mathcal{E}_B$ (columns); each cell shows top-1 accuracy (top) and Hellinger distance to the oracle posterior (bottom) as a function of sequence position. Higher accuracy / lower Hellinger is better.}
    \label{fig:linear_m4n4}
\end{figure}

\begin{figure}
    \centering
    \includegraphics[width=\linewidth]{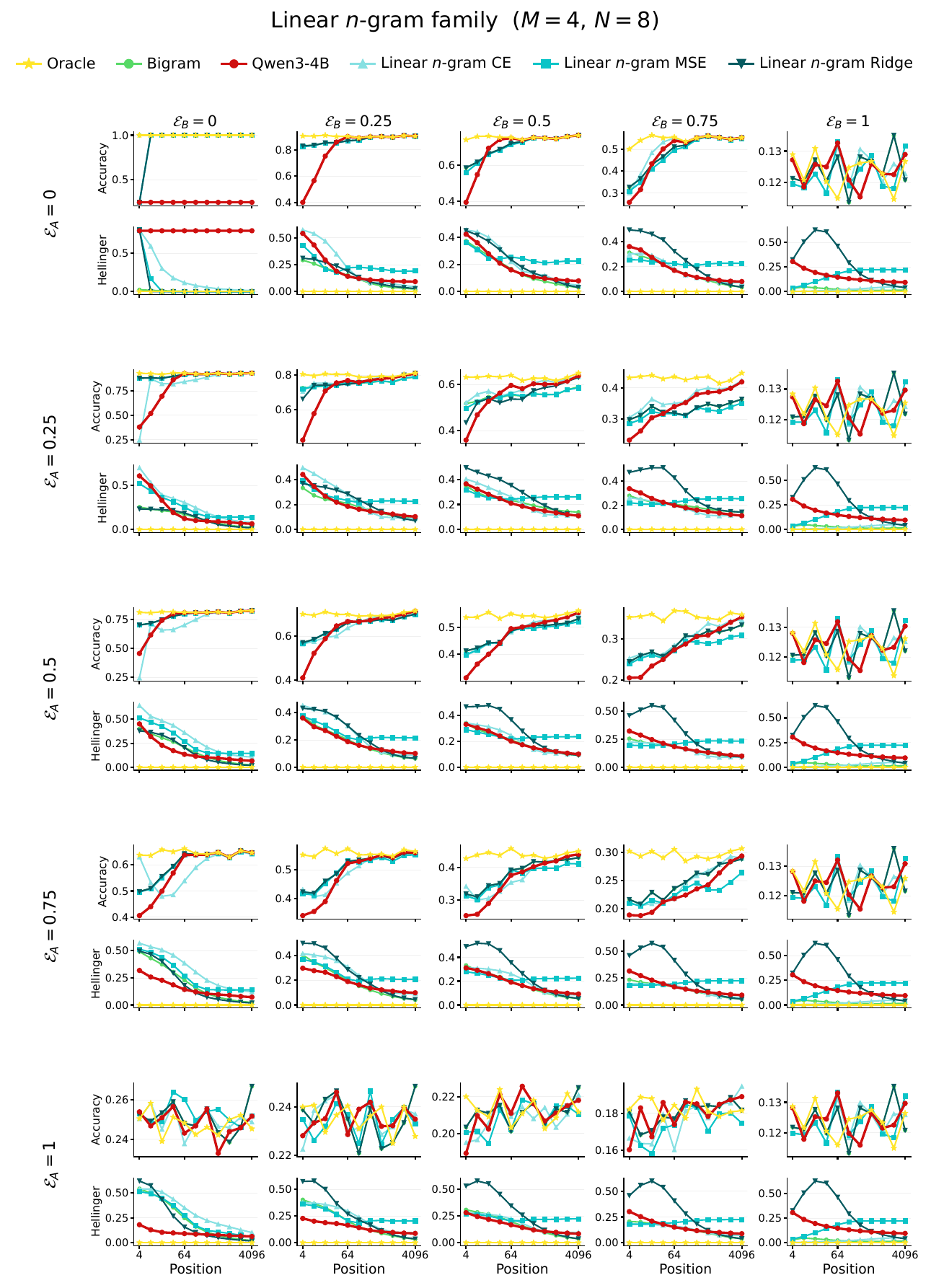}
    \caption{Grid over transition entropy $\mathcal{E}_A$ (rows) and emission entropy $\mathcal{E}_B$ (columns); each cell shows top-1 accuracy (top) and Hellinger distance to the oracle posterior (bottom) as a function of sequence position. Higher accuracy / lower Hellinger is better.}
    \label{fig:linear_m4n8}
\end{figure}

\begin{figure}
    \centering
    \includegraphics[width=\linewidth]{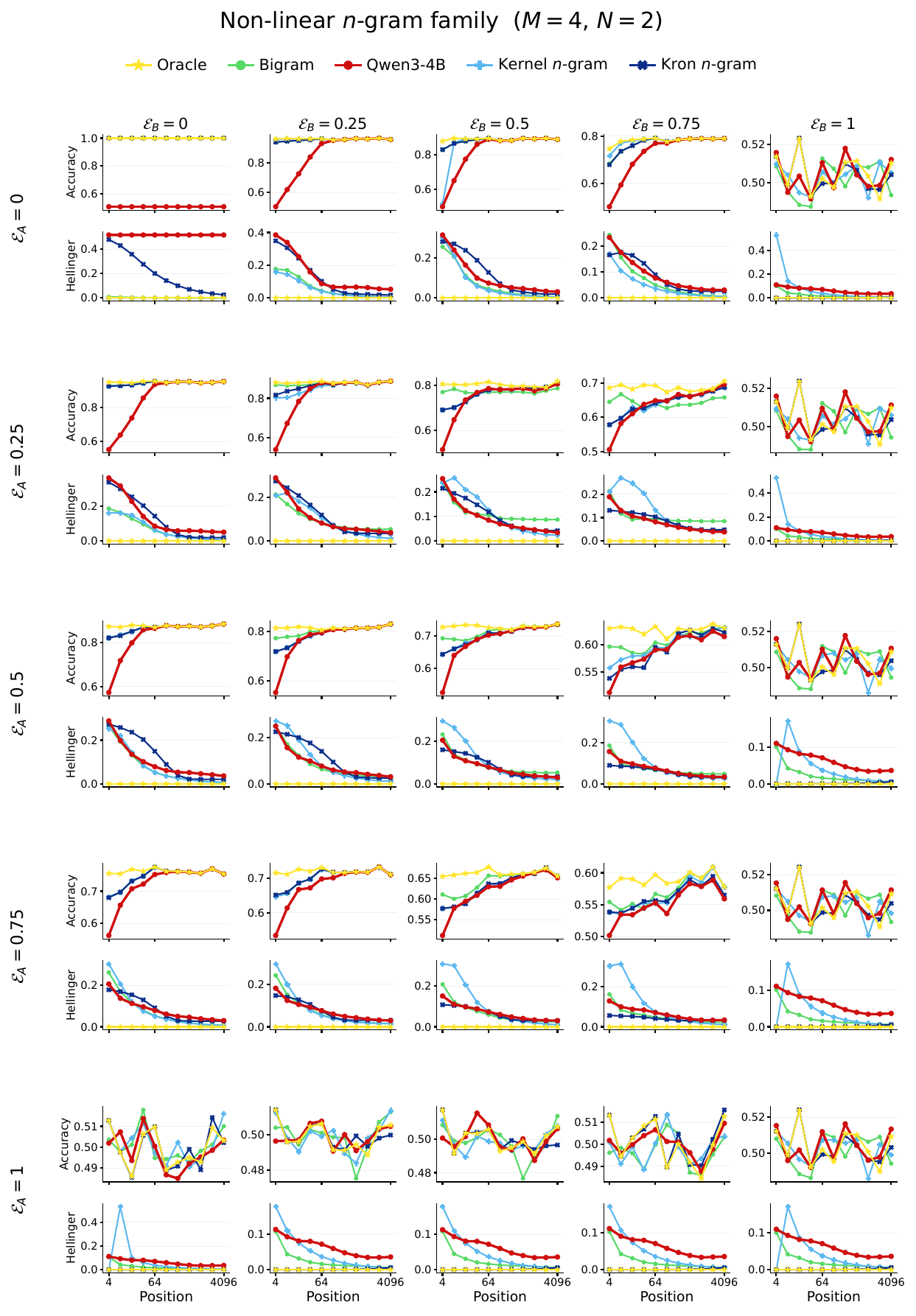}
    \caption{Grid over transition entropy $\mathcal{E}_A$ (rows) and emission entropy $\mathcal{E}_B$ (columns); each cell shows top-1 accuracy (top) and Hellinger distance to the oracle posterior (bottom) as a function of sequence position. Higher accuracy / lower Hellinger is better.}
    \label{fig:nonlinear_m4n2}
\end{figure}

\begin{figure}
    \centering
    \includegraphics[width=\linewidth]{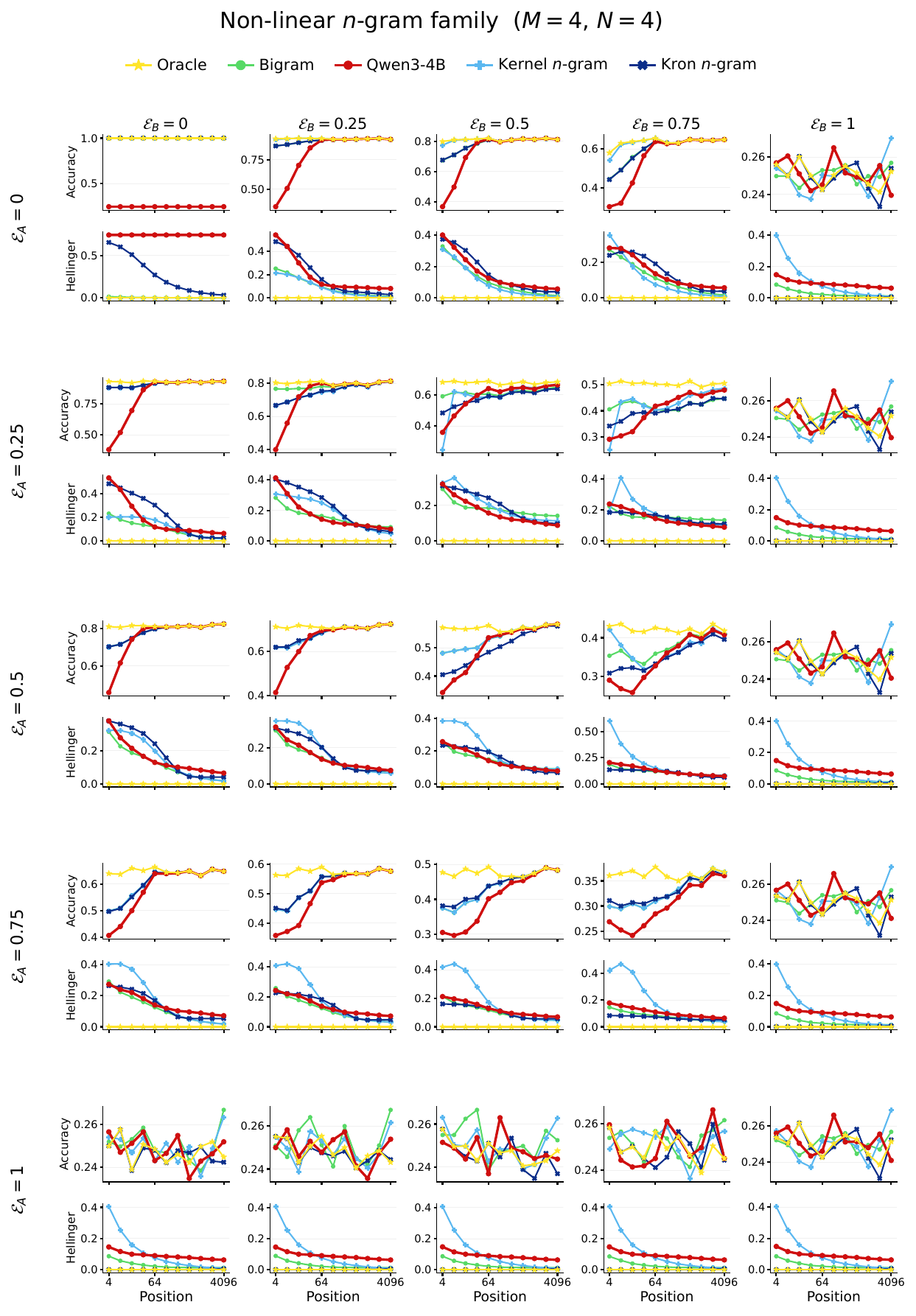}
    \caption{Grid over transition entropy $\mathcal{E}_A$ (rows) and emission entropy $\mathcal{E}_B$ (columns); each cell shows top-1 accuracy (top) and Hellinger distance to the oracle posterior (bottom) as a function of sequence position. Higher accuracy / lower Hellinger is better.}
    \label{fig:nonlinear_m4n4}
\end{figure}

\begin{figure}
    \centering
    \includegraphics[width=\linewidth]{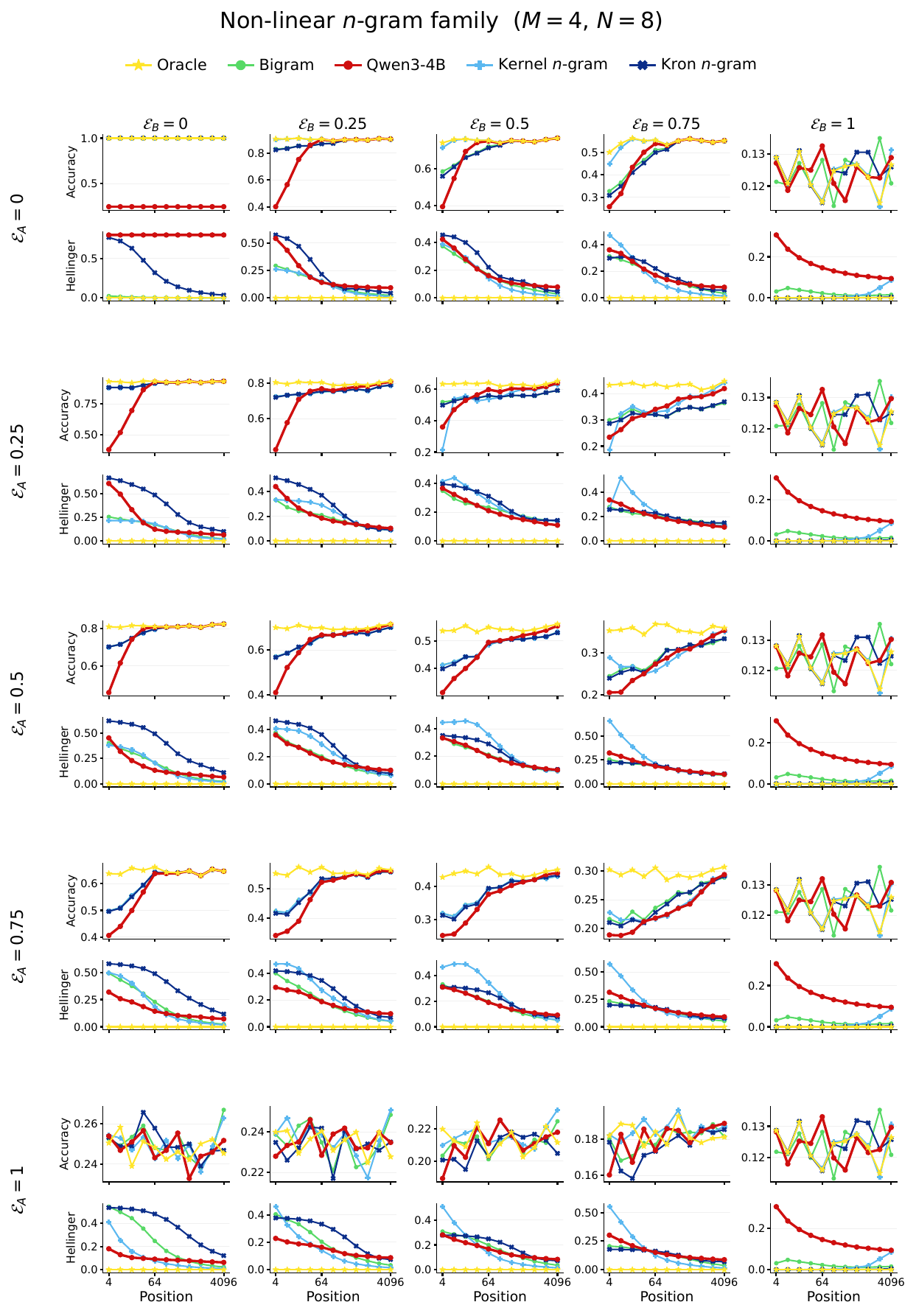}
    \caption{Grid over transition entropy $\mathcal{E}_A$ (rows) and emission entropy $\mathcal{E}_B$ (columns); each cell shows top-1 accuracy (top) and Hellinger distance to the oracle posterior (bottom) as a function of sequence position. Higher accuracy / lower Hellinger is better.}
    \label{fig:nonlinear_m4n8}
\end{figure}

\begin{figure}
    \centering
    \includegraphics[width=\linewidth]{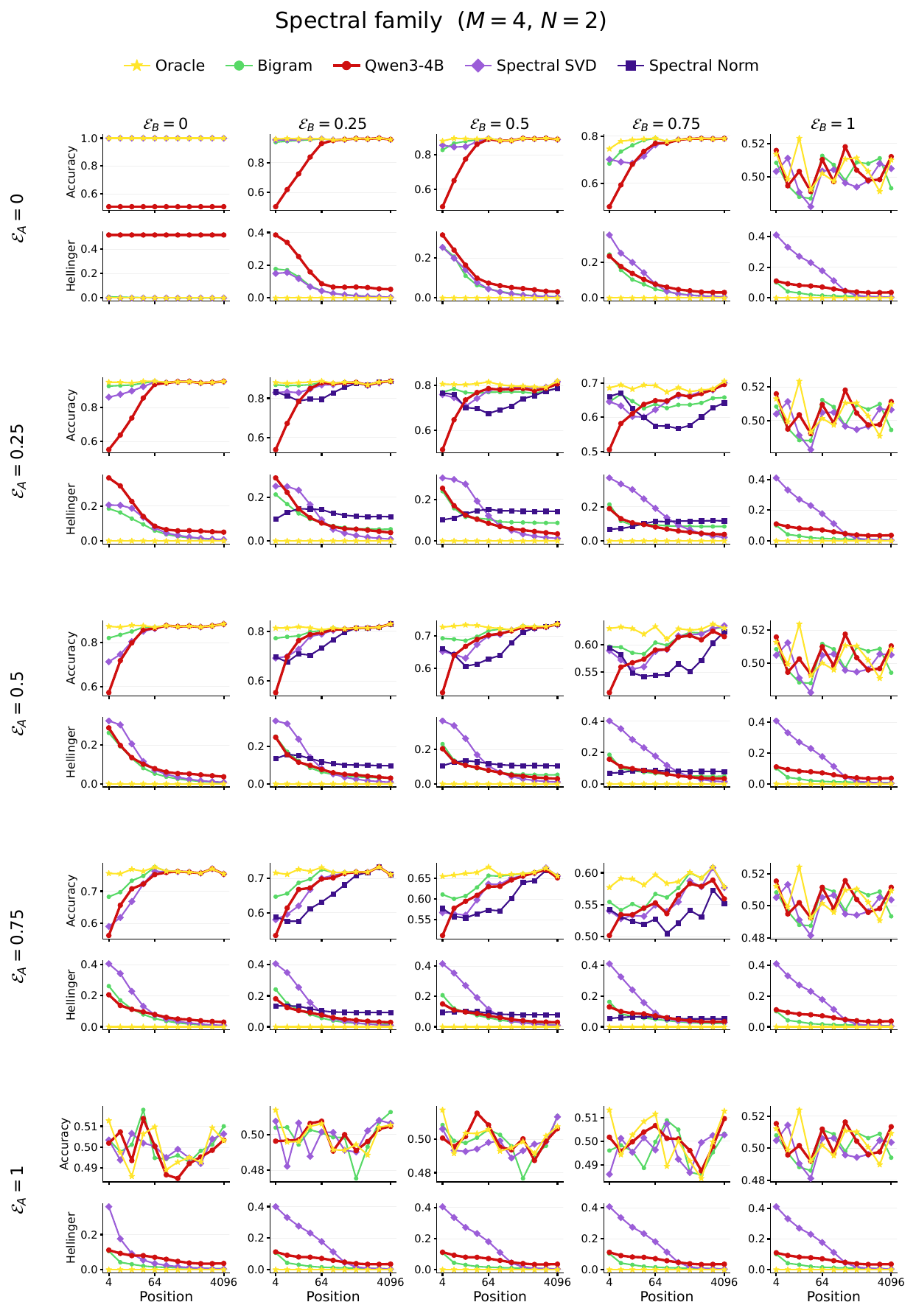}
    \caption{Grid over transition entropy $\mathcal{E}_A$ (rows) and emission entropy $\mathcal{E}_B$ (columns); each cell shows top-1 accuracy (top) and Hellinger distance to the oracle posterior (bottom) as a function of sequence position. Higher accuracy / lower Hellinger is better.}
    \label{fig:spectral_m4n2}
\end{figure}

\begin{figure}
    \centering
    \includegraphics[width=\linewidth]{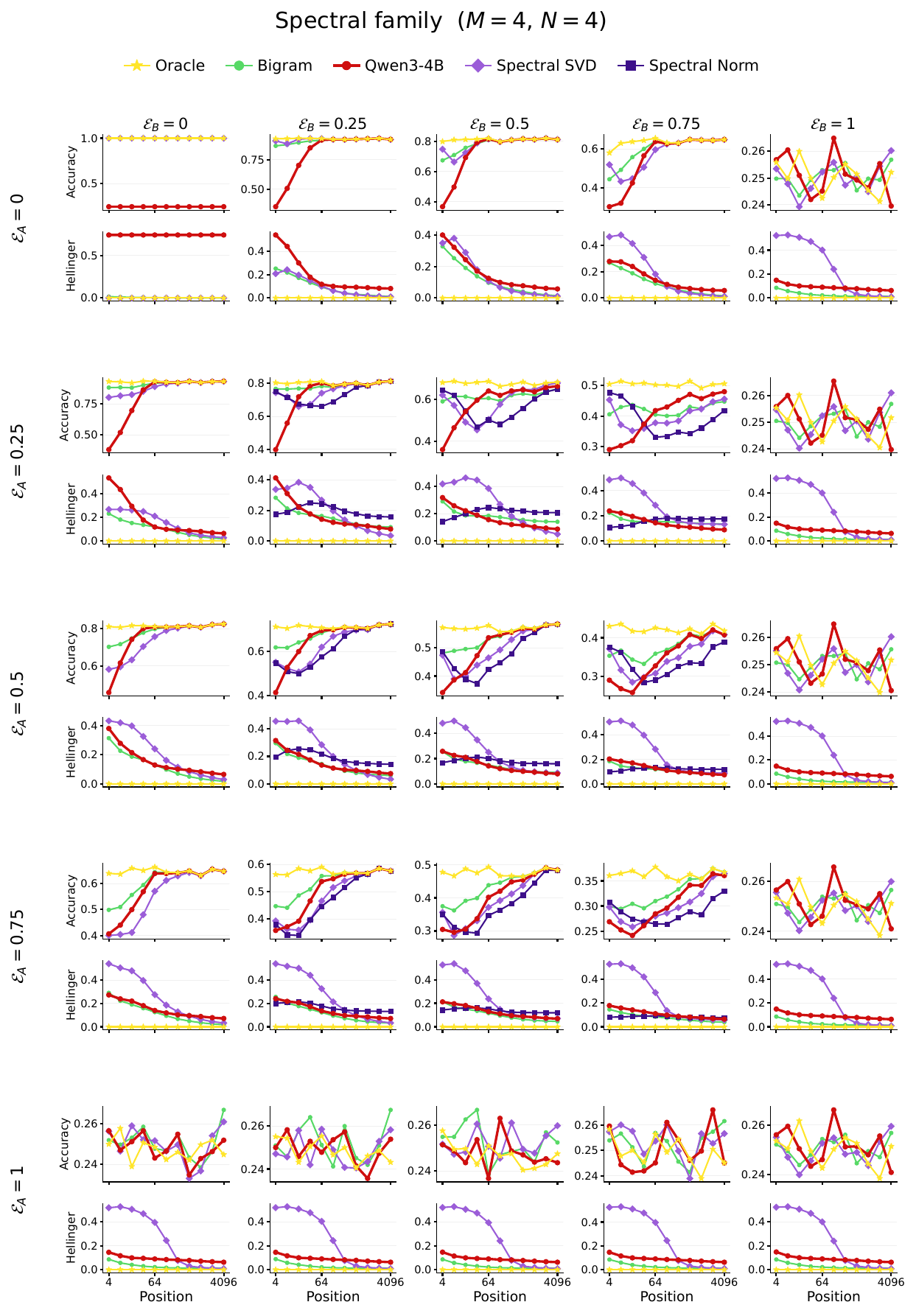}
    \caption{Grid over transition entropy $\mathcal{E}_A$ (rows) and emission entropy $\mathcal{E}_B$ (columns); each cell shows top-1 accuracy (top) and Hellinger distance to the oracle posterior (bottom) as a function of sequence position. Higher accuracy / lower Hellinger is better.}
    \label{fig:spectral_m4n4}
\end{figure}

\begin{figure}
    \centering
    \includegraphics[width=\linewidth]{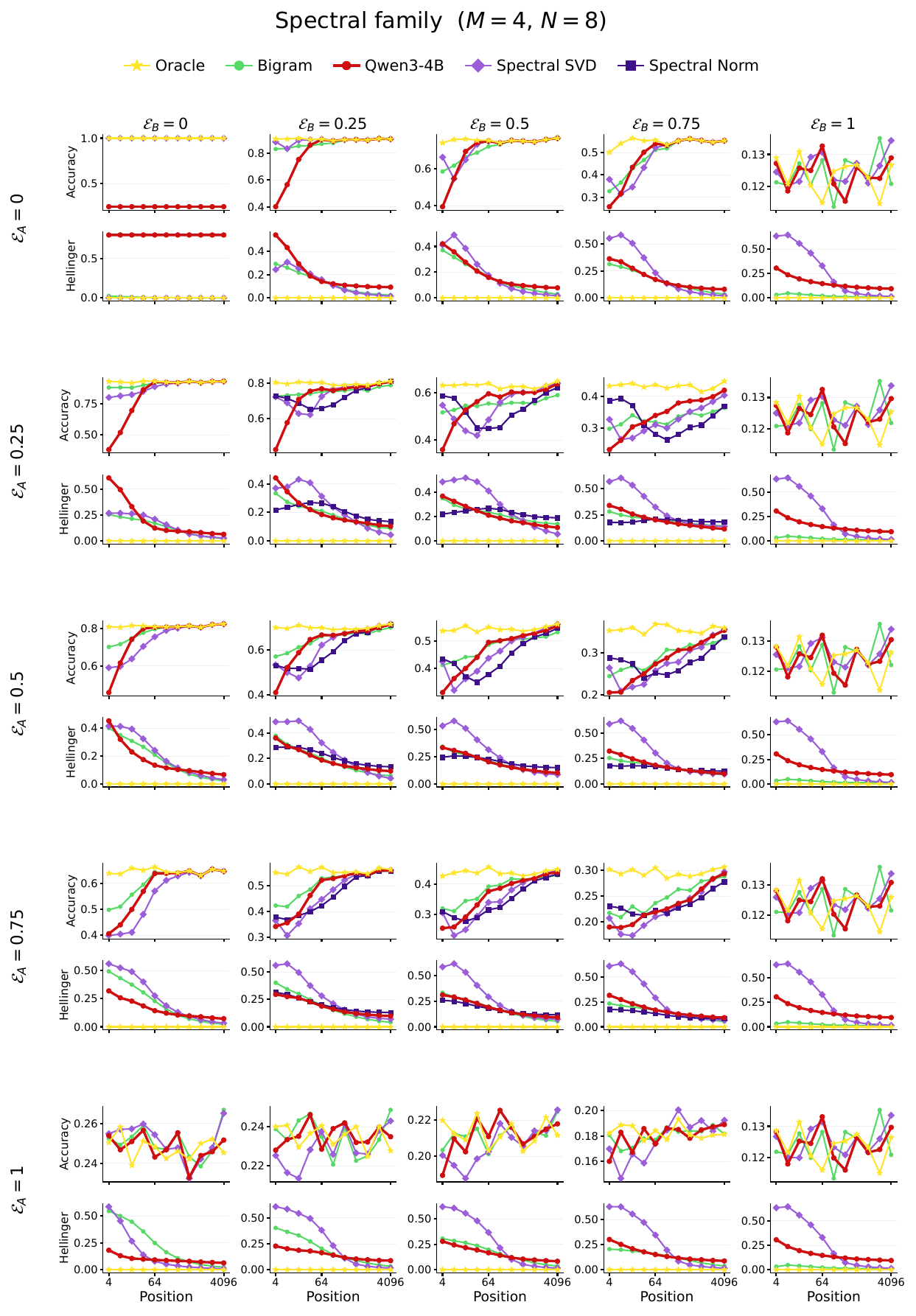}
    \caption{Grid over transition entropy $\mathcal{E}_A$ (rows) and emission entropy $\mathcal{E}_B$ (columns); each cell shows top-1 accuracy (top) and Hellinger distance to the oracle posterior (bottom) as a function of sequence position. Higher accuracy / lower Hellinger is better.}
    \label{fig:spectral_m4n8}
\end{figure}

\subsection{Soft $n$-gram Baselines with Varying $n$}

We examine how the order $n$ of the $n$-gram models affects predictive performance across HMM configurations. Recall that $n$ controls the length of the conditioning history: larger $n$ allows the model to capture longer-range dependencies in the observation sequence, at the cost of requiring more data to reliably estimate the associated statistics. We report results for both linear and non-linear variants across $N \in \{2, 4, 8\}$ to assess how the optimal order interacts with alphabet size and entropy.

\begin{figure}[H]
    \centering
    \includegraphics[width=0.82\linewidth]{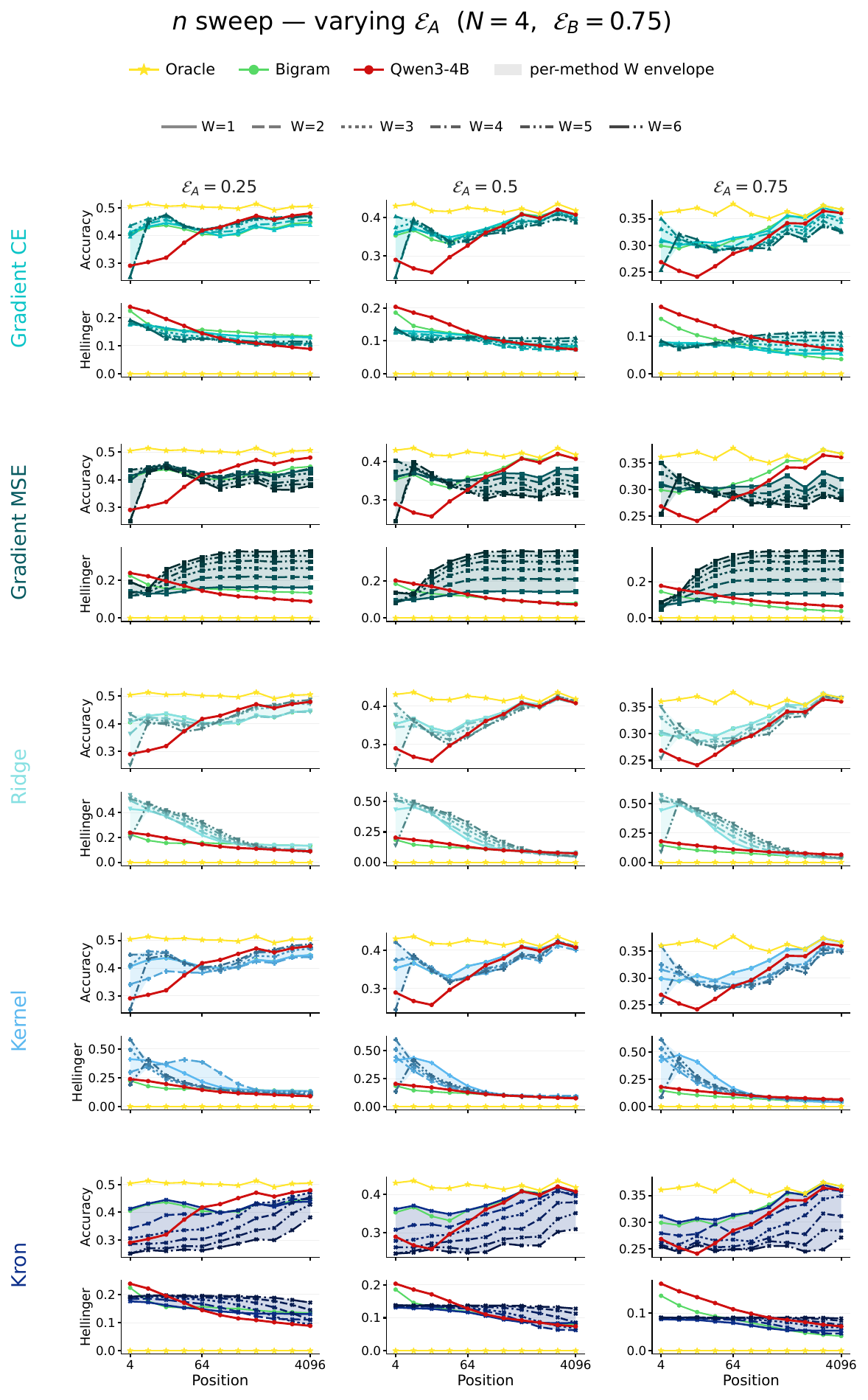}
    \caption{Rows are the windowed baselines (Gradient CE / MSE / Ridge / Kernel / Kron); columns vary the transition entropy $\mathcal{E}_A \in \{0.25, 0.5, 0.75\}$. Each cell shows top-1 accuracy (top) and Hellinger distance to the oracle posterior (bottom) vs. sequence position (log axis). Within a cell, the colored lines are the per-window fits ($n=1\ldots6$, distinguished by linestyle and a light→dark shade) and the shaded band is their min–max envelope; Oracle, Bigram, and Qwen3-4B are drawn for reference.}
    \label{fig:sweep_n_m4_A_e}
\end{figure}

\begin{figure}[H]
    \centering
    \includegraphics[width=0.82\linewidth]{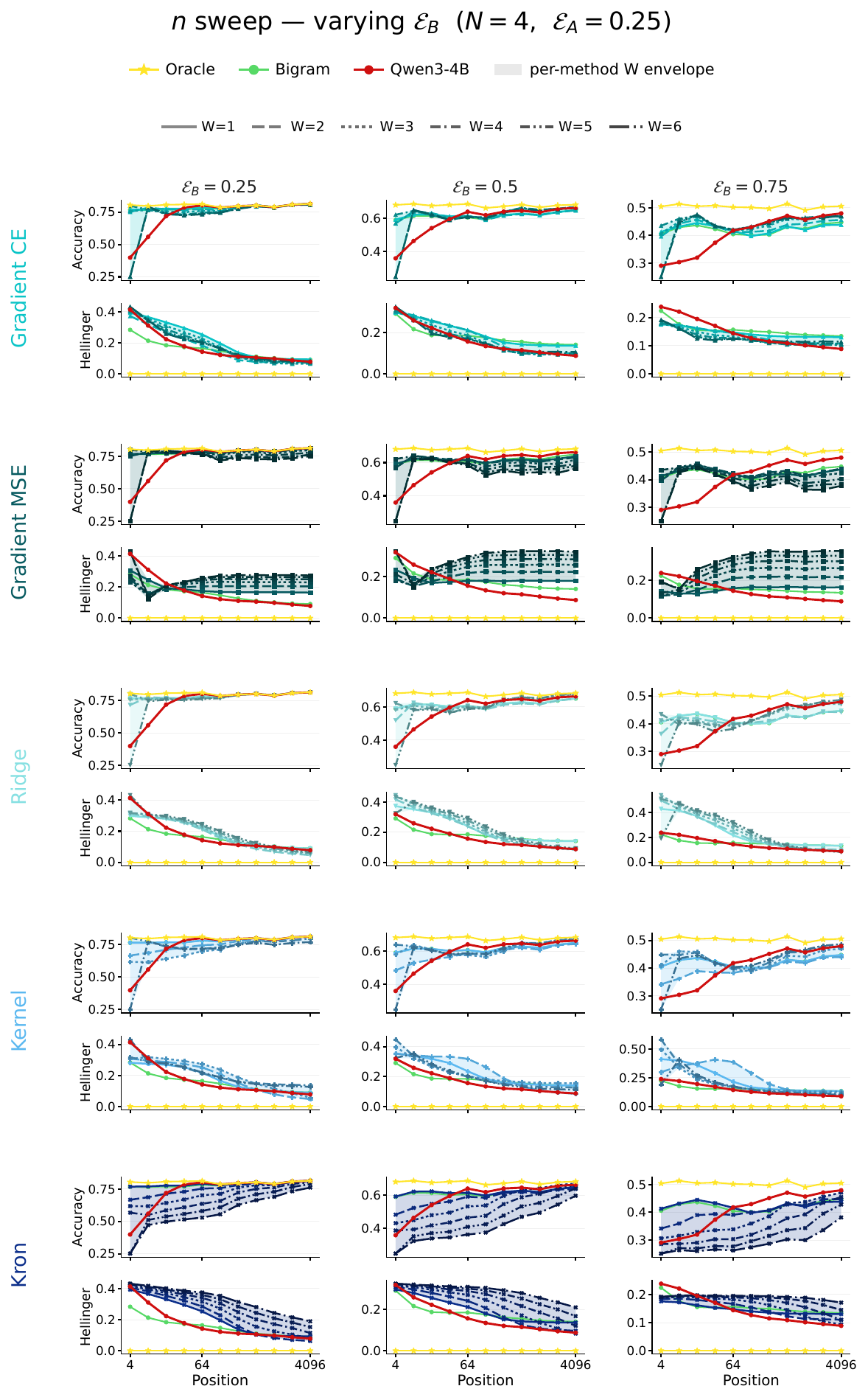}
    \caption{Rows are the windowed baselines (Gradient CE / MSE / Ridge / Kernel / Kron); columns vary the emission entropy $\mathcal{E}_B \in \{0.25, 0.5, 0.75\}$. Each cell shows top-1 accuracy (top) and Hellinger distance to the oracle posterior (bottom) vs. sequence position (log axis). Within a cell, the colored lines are the per-window fits ($n=1\ldots6$, distinguished by linestyle and a light→dark shade) and the shaded band is their min–max envelope; Oracle, Bigram, and Qwen3-4B are drawn for reference.}
    \label{fig:sweep_n_m4_B_e}
\end{figure}

\begin{figure}[H]
    \centering
    \includegraphics[width=\linewidth]{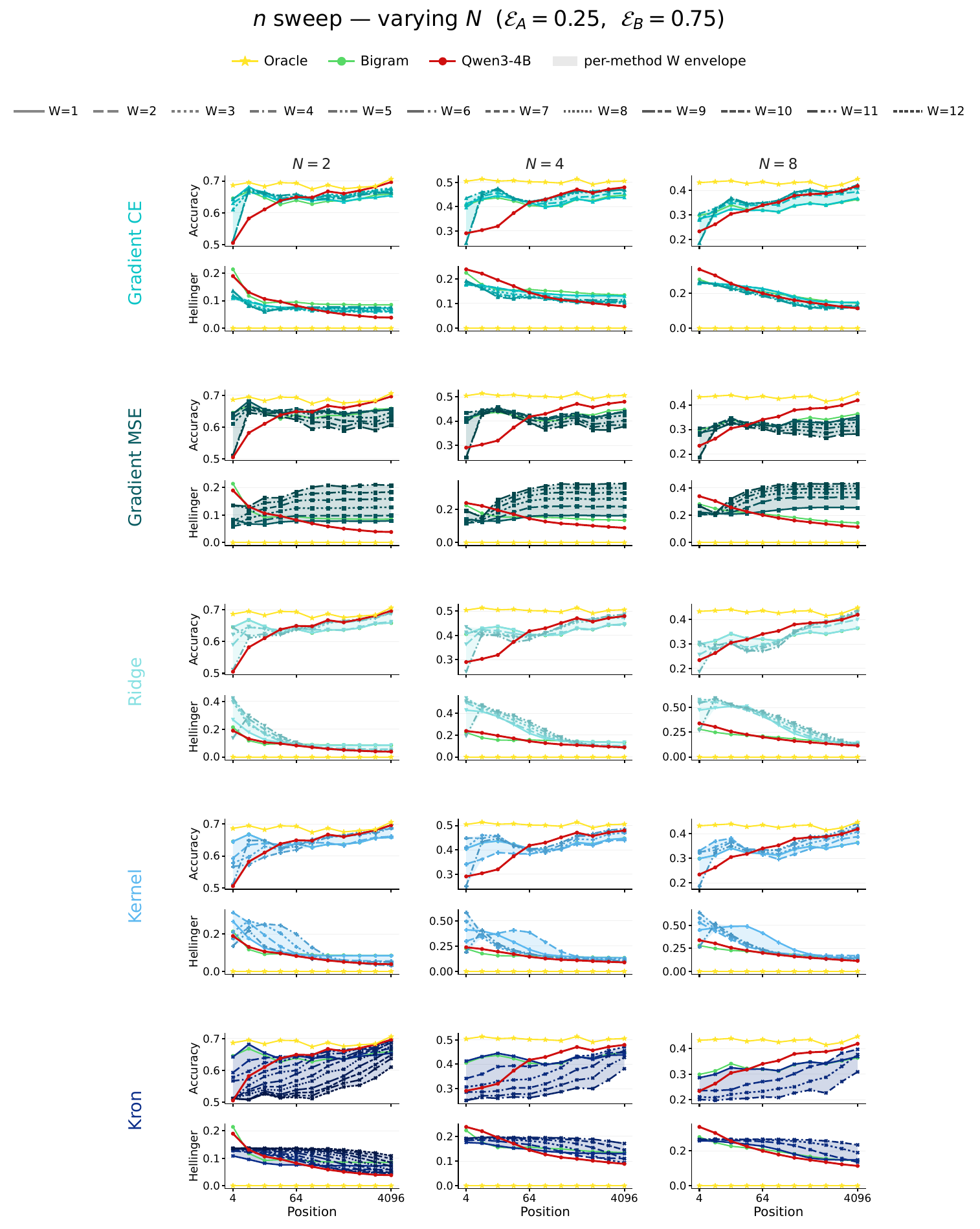}
    \caption{Rows are the windowed baselines (Gradient CE / MSE / Ridge / Kernel / Kron); columns vary the number of observations $N\in\{2,4,8\}$. Each cell shows top-1 accuracy (top) and Hellinger distance to the oracle posterior (bottom) vs. sequence position (log axis). Within a cell, the colored lines are the per-window fits ($n=1\ldots6$, distinguished by linestyle and a light→dark shade) and the shaded band is their min–max envelope; Oracle, Bigram, and Qwen3-4B are drawn for reference.}
    \label{fig:sweep_n_m4_A_e}
\end{figure}

\subsection{More Hidden States $(M=12)$}
\label{app:more_state}
For all the results we show above, they are done with a fixed number of hidden states $M=4$. We also compare LLM performance on a more complex HMM with 12 hidden states, to see if the insights we gain above is extensible. And we can also find if there's trend w.r.t. $M$. We mainly focus on baseline comparison, as it has potential to inform us something new and meaningful.

\begin{figure}[H]
    \centering
    \includegraphics[width=\linewidth]{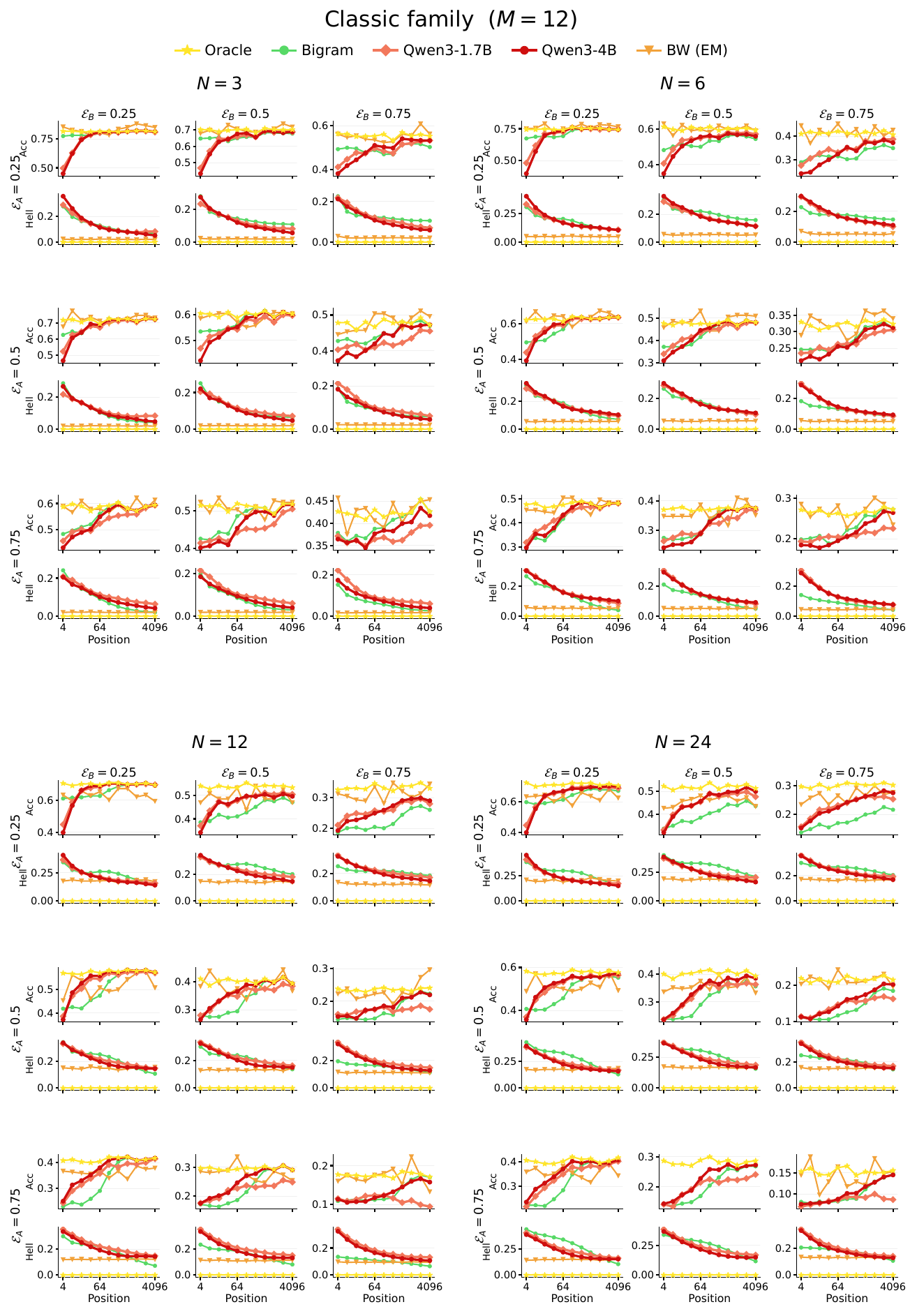}
    \caption{Classic baselines (Bigram, Baum--Welch) vs.\ position on 12-state HMMs ($M=12$).
    Each of the four blocks is a fixed observation-alphabet size $N\in\{3,6,12,24\}$; within a
    block, rows vary the transition entropy $\mathcal{E}_A$ and columns the emission entropy
    $\mathcal{E}_B$ (both $\in\{0.25,0.5,0.75\}$). Every cell shows top-1 accuracy (top) and
    Hellinger distance to the oracle posterior (bottom) versus sequence position (log axis).
    Oracle, Bigram, Qwen3-1.7B and Qwen3-4B are drawn for reference alongside Baum--Welch
    (BW/EM). Higher accuracy / lower Hellinger is better.}
    \label{fig:baselines_m12_classic}
\end{figure}

\begin{figure}[H]
    \centering
    \includegraphics[width=\linewidth]{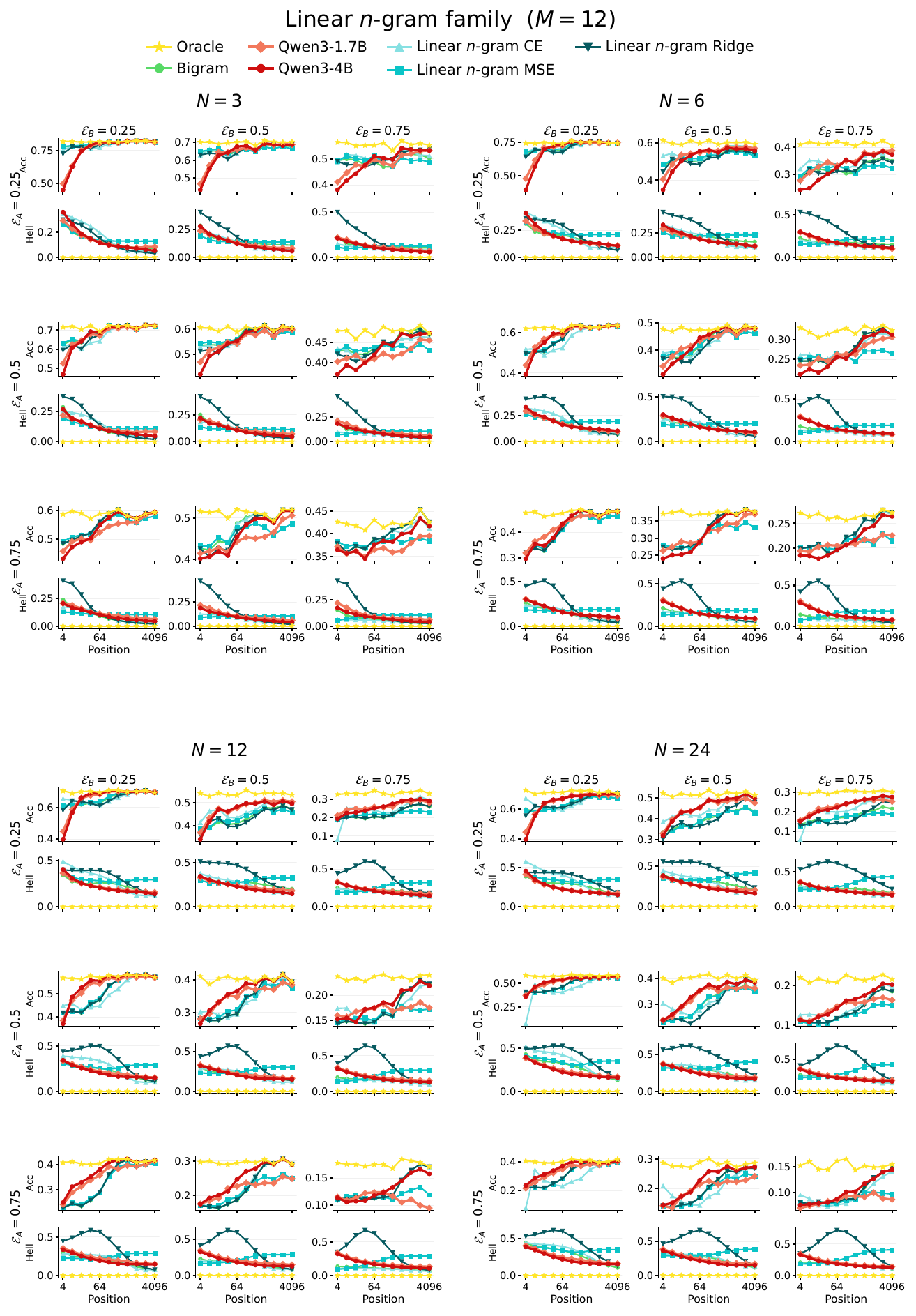}
    \caption{Linear $n$-gram family vs.\ position on 12-state HMMs ($M=12$).
    Each of the four blocks is a fixed observation-alphabet size $N\in\{3,6,12,24\}$; within a
    block, rows vary the transition entropy $\mathcal{E}_A$ and columns the emission entropy
    $\mathcal{E}_B$ (both $\in\{0.25,0.5,0.75\}$). Every cell shows top-1 accuracy (top) and
    Hellinger distance to the oracle posterior (bottom) versus sequence position (log axis).
    Oracle, Bigram, Qwen3-1.7B and Qwen3-4B are drawn for reference alongside Linear $n$-gram
    CE / MSE / Ridge; the $n$-gram window is selected per position by validation. Higher
    accuracy / lower Hellinger is better.}
    \label{fig:baselines_m12_linear}
\end{figure}

\begin{figure}[H]
    \centering
    \includegraphics[width=\linewidth]{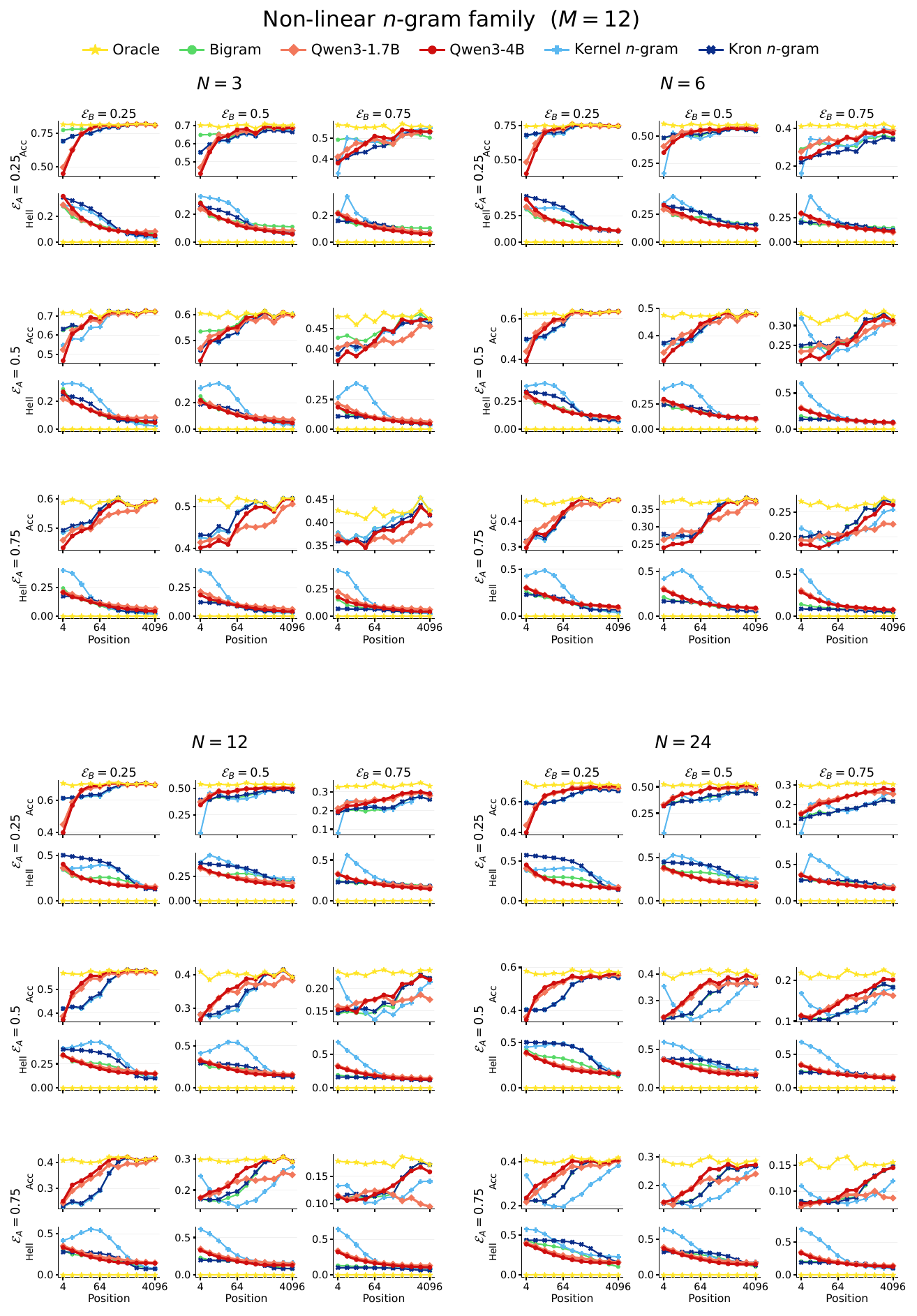}
    \caption{Non-linear $n$-gram family vs.\ position on 12-state HMMs ($M=12$).
    Each of the four blocks is a fixed observation-alphabet size $N\in\{3,6,12,24\}$; within a
    block, rows vary the transition entropy $\mathcal{E}_A$ and columns the emission entropy
    $\mathcal{E}_B$ (both $\in\{0.25,0.5,0.75\}$). Every cell shows top-1 accuracy (top) and
    Hellinger distance to the oracle posterior (bottom) versus sequence position (log axis).
    Oracle, Bigram, Qwen3-1.7B and Qwen3-4B are drawn for reference alongside Kernel $n$-gram
    and Kron $n$-gram; the $n$-gram window is selected per position by validation. Higher
    accuracy / lower Hellinger is better.}
    \label{fig:baselines_m12_nonlinear}
  \end{figure}

\begin{figure}[H]
    \centering
    \includegraphics[width=\linewidth]{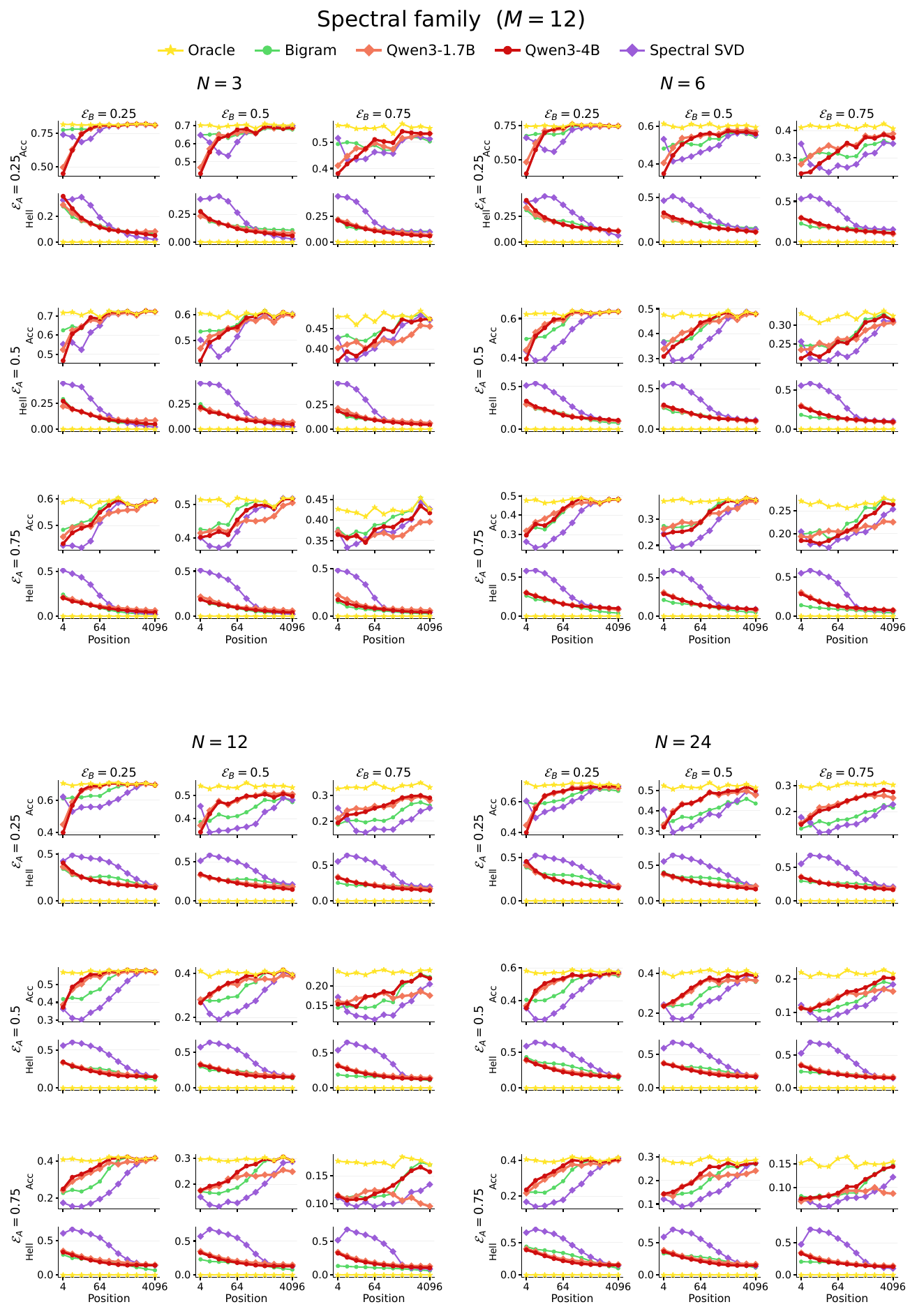}
    \caption{Spectral family vs.\ position on 12-state HMMs ($M=12$).
    Each of the four blocks is a fixed observation-alphabet size $N\in\{3,6,12,24\}$; within a
    block, rows vary the transition entropy $\mathcal{E}_A$ and columns the emission entropy
    $\mathcal{E}_B$ (both $\in\{0.25,0.5,0.75\}$). Every cell shows top-1 accuracy (top) and
    Hellinger distance to the oracle posterior (bottom) versus sequence position (log axis).
    Oracle, Bigram, Qwen3-1.7B and Qwen3-4B are drawn for reference alongside Spectral SVD
    (Spectral Norm was not run at $M=12$). Higher accuracy / lower Hellinger is better.}
    \label{fig:baselines_m12_spectral}
\end{figure}

%% file: appendix_principal_probe.tex
\subsection{Null test on linear probing the belief simplex.}

A high $R^2$ on a held-out split shows that the probe generalizes, but it does not by itself rule out that a sufficiently expressive linear map could fit any target from the residual stream. To check that the recovered geometry reflects alignment between the activations and the belief, we run a label-shuffling null test. For each activation $x_{\ell,t}$ in the training split we permute the corresponding belief label $b_t$ across examples, breaking the activation--belief correspondence while preserving the marginal distribution of beliefs. We then refit the probe on the shuffled pairs using the same PCA rank and ridge regularization as in the main experiments, and evaluate it on an unshuffled test split.

\begin{figure}[h]
  \centering
  \includegraphics[width=\textwidth]{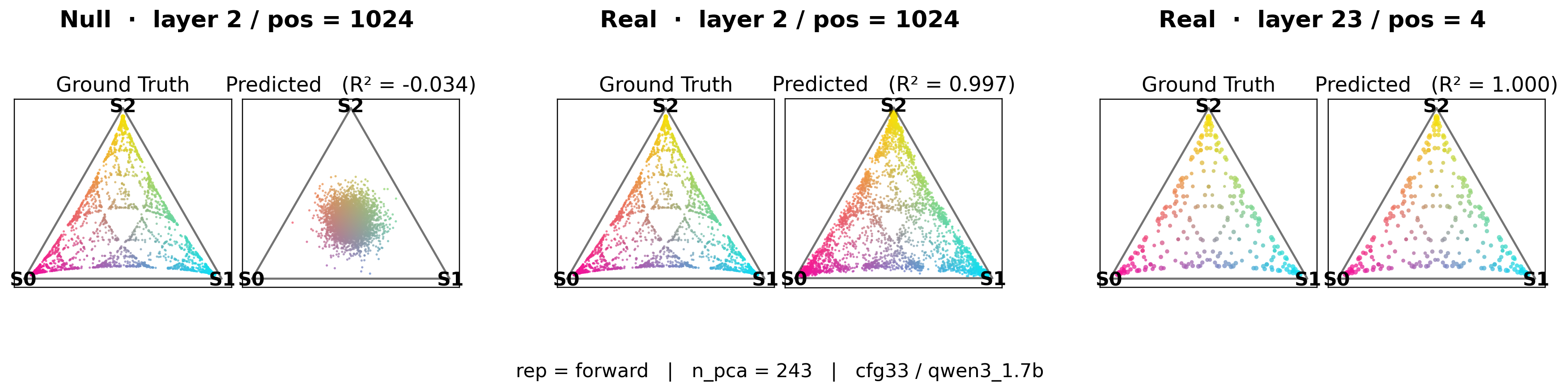}
  \caption{(Left) Null test of linear belief representation probing. (Middle) Comparison setting. (Right) Geometry when sequence length is 4.}
  \label{fig:app_simplex}
\end{figure}

\subsection{Belief probing at very short sequence length}

When sequence length is 4, the $R^2$ scores trend behaves differently than longer sequence. The reason is that the belief vectors has less diversity at such short sequence length (this HMM's mixing rate is 4.18). One can see the simplex in Figure \ref{fig:app_simplex}, right.

\subsection{Layers as stages of computation}
Both PAP and causal intervention results show strong indication that early layers and late layers have distinct computation roles. We further compute the centered Kernel alignment scores between each layers, as shown in Figure \ref{fig:layers}, left. We see that early-to-mid layers are aligned closely, and late layers form another group. We further investigate the projection from each layer residual to the output token space using the \texttt{lm\_head}, we calculate the mass of the valid observation tokens from the projected outputs. As shown in \ref{fig:layers} right, the later layers live in the output space. This partially explains the phenomenon we see in Section \ref{sec:llm_internal} on the belief-dominated regime --- patching (only) the late layers causally affect the LLM outputs.

\begin{figure}[h]
  \centering
  \vspace{-1em}
  \includegraphics[width=0.6\textwidth]{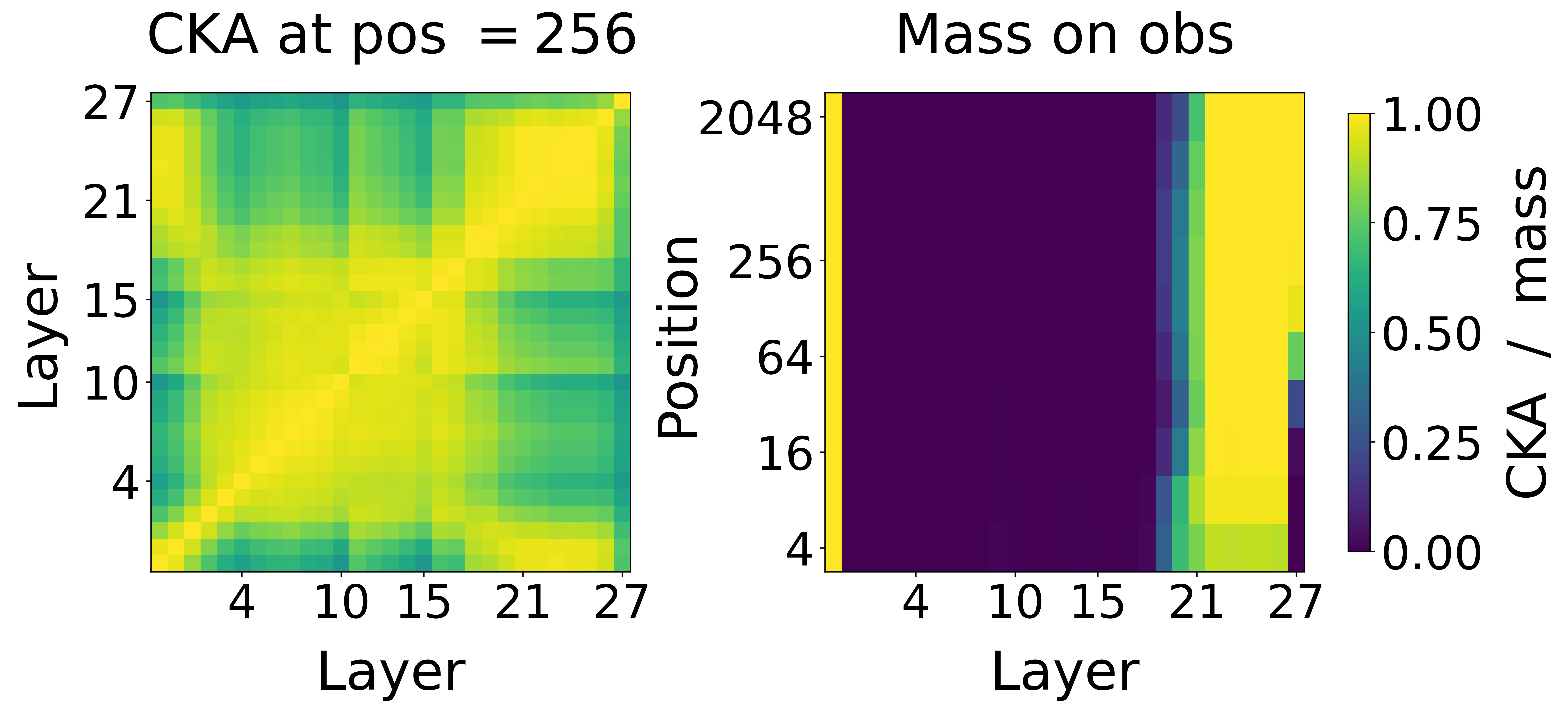}
  \caption{Left: CKA between layers. Right: Probability mass on valid observations.}
  \label{fig:layers}
  \vspace{-1em}
\end{figure}

\subsection{Quality at Varying top-$k$ Principal Components}

To see how much of the belief geometry lives in the leading directions of the residual stream, we project each layer's activations onto their top-$k$ principal components and fit a linear probe to recover the oracle (Bayesian) belief $b_t$ from that $k$-dimensional subspace. Figure~\ref{fig:probe_r2_kcols_forward_belief} reports the held-out probe $R^2$ as a function of transformer layer ($x$-axis) and token position ($y$-axis), across the nine HMM configurations with $\mathcal{E}(\mathbf{A}),\mathcal{E}(\mathbf{B})\in\{0.25,0.5,0.75\}$ and PCA dimensions $k\in\{4,8,16,32,64,128,2048\}$, where $k=2048$ effectively corresponds to the full residual stream. Three trends emerge: (i) probe performance saturates quickly with $k$, with most belief information captured by roughly $64$--$128$ principal components and little improvement beyond that; (ii) belief information is distributed non-uniformly across depth, as small-$k$ probes perform worst in the middle layers, suggesting that belief representations are spread across many directions there but become concentrated into leading components in early and late layers; and (iii) the number of components required for saturation increases with inference difficulty, with low-entropy regimes saturating in very few PCs and higher emission entropy $\mathcal{E}(\mathbf{B})$ requiring more components. Together, these results suggest that the top-$k$ PCA subspace provides a faithful low-dimensional summary of the belief representation, motivating the use of modest values of $k$ in the PCA-based causal interventions below.

\begin{figure}[H]
    \centering
    \includegraphics[width=\linewidth]{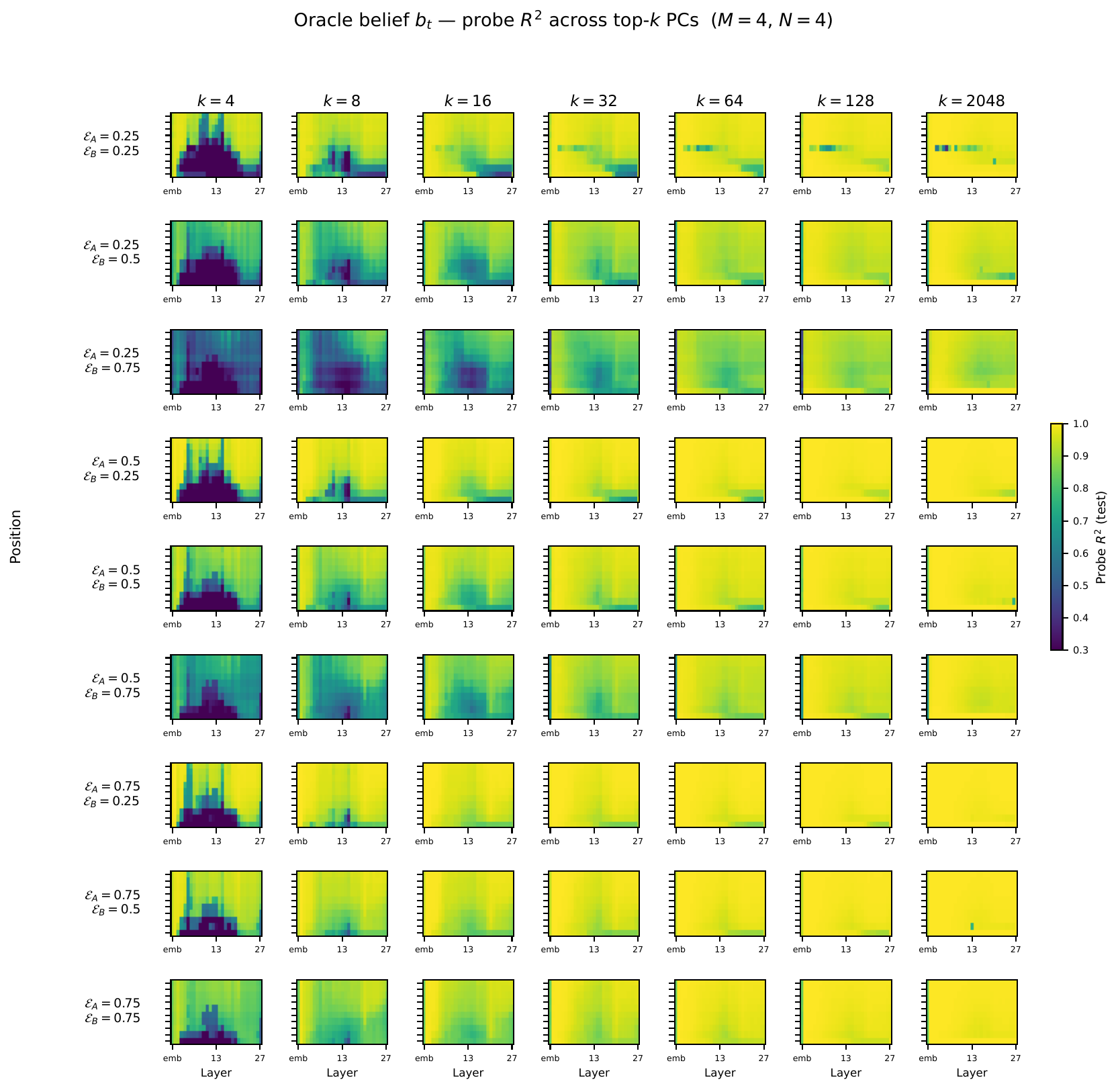}
    \caption{\textbf{Probe $R^2$ for recovering the oracle belief from top-$k$ principal components.} Each panel shows the held-out linear-probe $R^2$ for predicting the oracle belief $b_t$ from the top-$k$ PCA subspace of Qwen3-1.7B residual activations, as a function of transformer layer ($x$: \texttt{emb}, $0,\ldots,27$) and token position ($y$: $4,\ldots,2048$). Rows correspond to the nine HMM configurations with $\mathcal{E}(\mathbf{A}),\mathcal{E}(\mathbf{B})\in\{0.25,0.5,0.75\}$, and columns vary the retained PCA dimension $k$, with $k=2048$ corresponding to the full residual stream.}
    \label{fig:probe_r2_kcols_forward_belief}
\end{figure}

\newpage
\subsection{Additional Results for Causal Interventions}

In Section~\ref{sec:llm_internal} of the main paper, we illustrated principal activations patching (PAP) using two representative HMM configurations, corresponding to $belief$-dominated and $observation$-dominated regimes. Here, we extend the analysis to the full family of HMMs spanning transition entropy $\mathcal{E}(\mathbf{A}) \in \{0.25, 0.5, 0.75\}$, emission entropy $\mathcal{E}(\mathbf{B}) \in \{0.25, 0.5, 0.75\}$, and observation alphabet size $N \in \{2,4,8\}$.

For the PCA-subspace and probe-inverse interventions, we fix the probe dimension to $k=8$. Results are shown for Qwen3-1.7B. Across all settings, we compare the oracle-belief probe $R^2$ with the interchange-intervention accuracy (IIA) obtained from full-residual, PCA-subspace, and probe-inverse patching. This allows us to assess how faithfully low-dimensional belief representations capture the causal variables used by the model during inference, as shown in Figures~\ref{fig:oracle_belief_patch_nobs2}--\ref{fig:oracle_belief_patch_nobs8}.

\begin{figure}
    \centering
    \includegraphics[width=\linewidth]{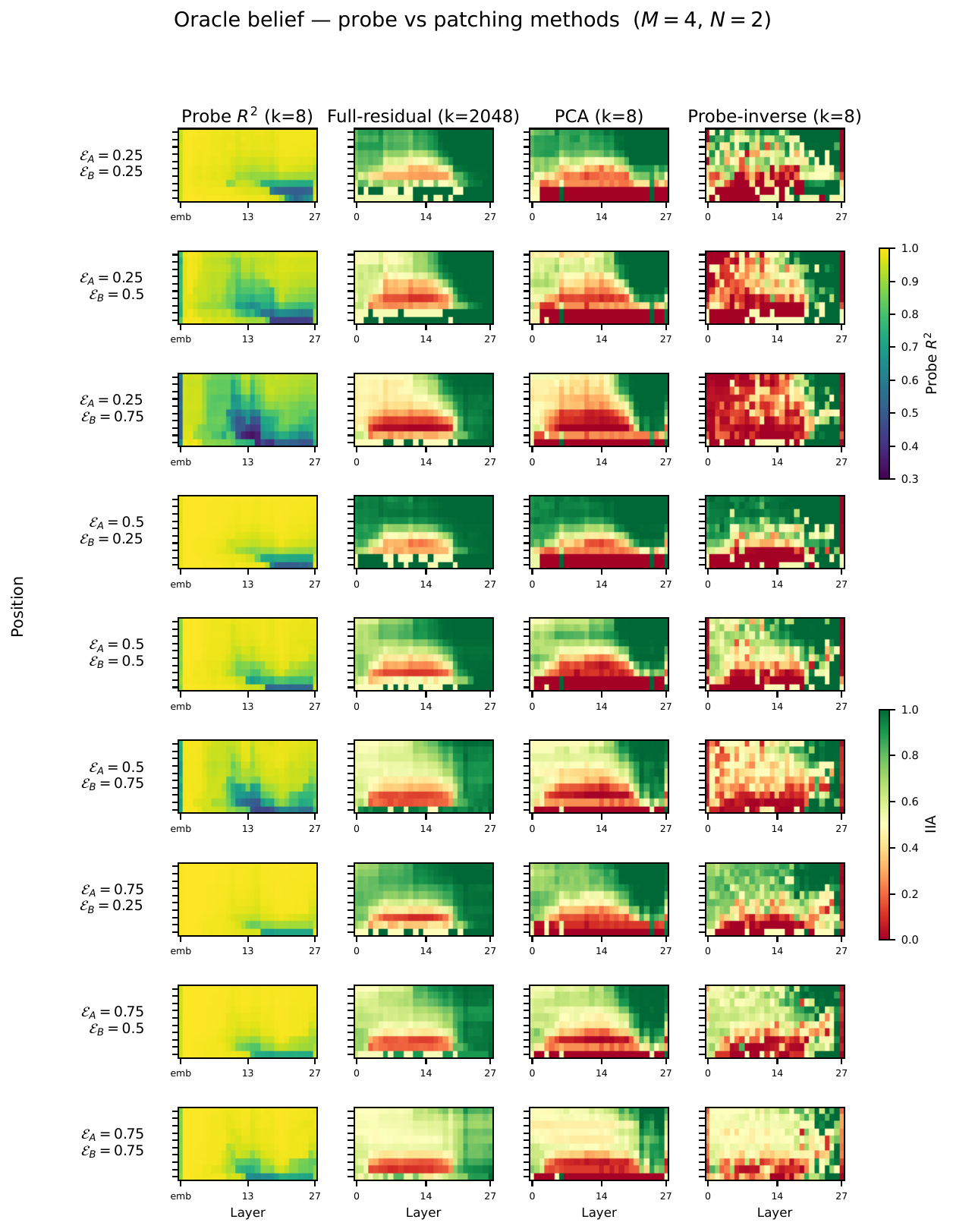}
    \caption{\textbf{Causal interventions for $N=2$.} Rows are HMM configs (transition entropy $\mathcal{E}(\mathbf{A})$, emission entropy $\mathcal{E}(\mathbf{B})$). Columns: linear-probe $R^2$ for the oracle belief $\mathbb P(h_t|\mathbf{o}_{1:t},\boldsymbol{\lambda})$ ($k{=}8$), then interchange-intervention accuracy (IIA) for full-residual ($k{=}2048$), PCA ($k{=}8$), and probe-inverse ($k{=}8$) patches. Each panel is a layer\,$\times$\,position heatmap.}
    \label{fig:oracle_belief_patch_nobs2}
\end{figure}

\begin{figure}
    \centering
    \includegraphics[width=\linewidth]{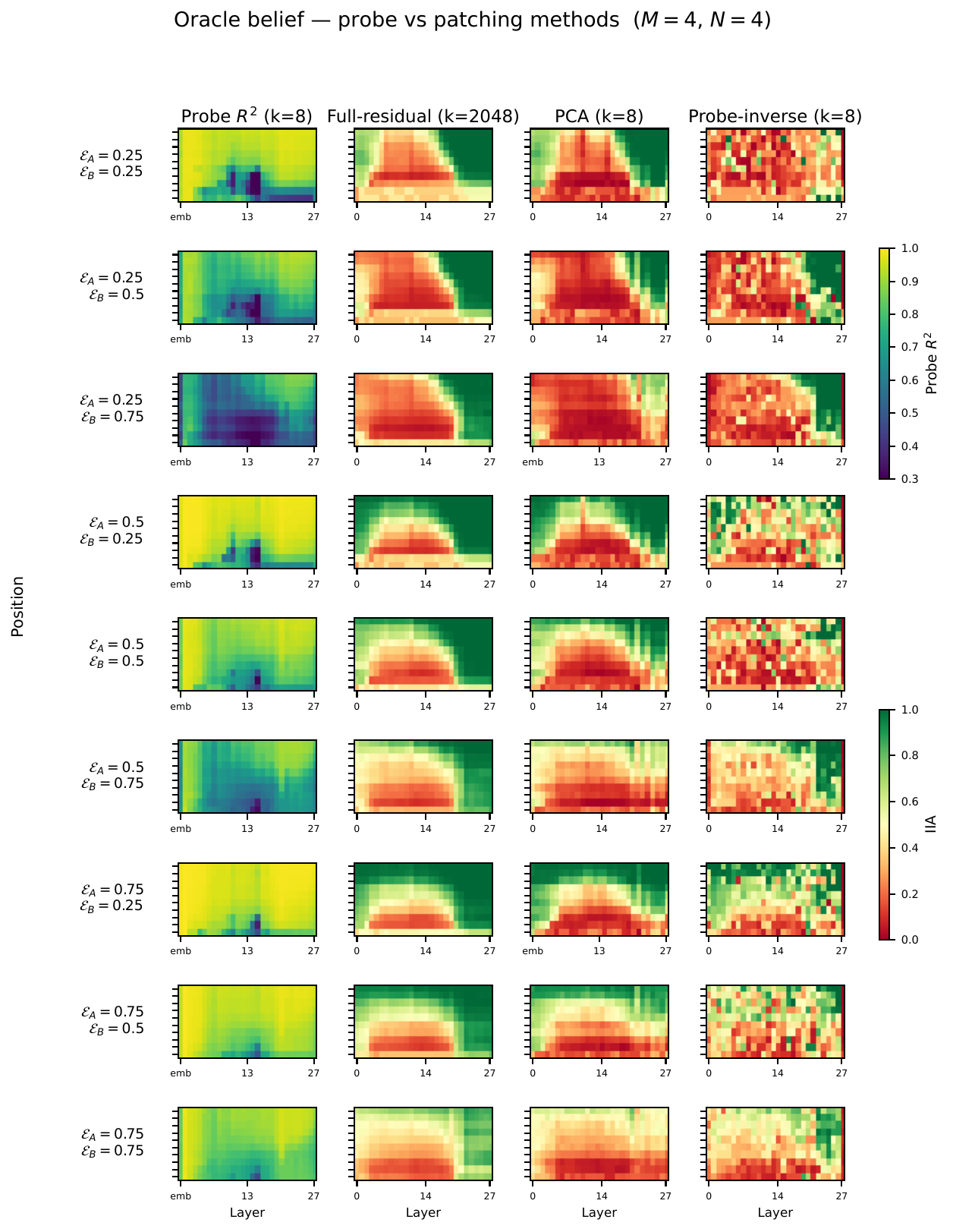}
    \caption{\textbf{Causal interventions for $N=4$.} Rows are HMM configs (transition entropy $\mathcal{E}(\mathbf{A})$, emission entropy $\mathcal{E}(\mathbf{B})$). Columns: linear-probe $R^2$ for the oracle belief $\mathbb P(h_t|\mathbf{o}_{1:t},\boldsymbol{\lambda})$ ($k{=}8$), then interchange-intervention accuracy (IIA) for full-residual ($k{=}2048$), PCA ($k{=}8$), and probe-inverse ($k{=}8$) patches. Each panel is a layer\,$\times$\,position heatmap.}
    \label{fig:oracle_belief_patch_nobs4}
\end{figure}

\begin{figure}
    \centering
    \includegraphics[width=\linewidth]{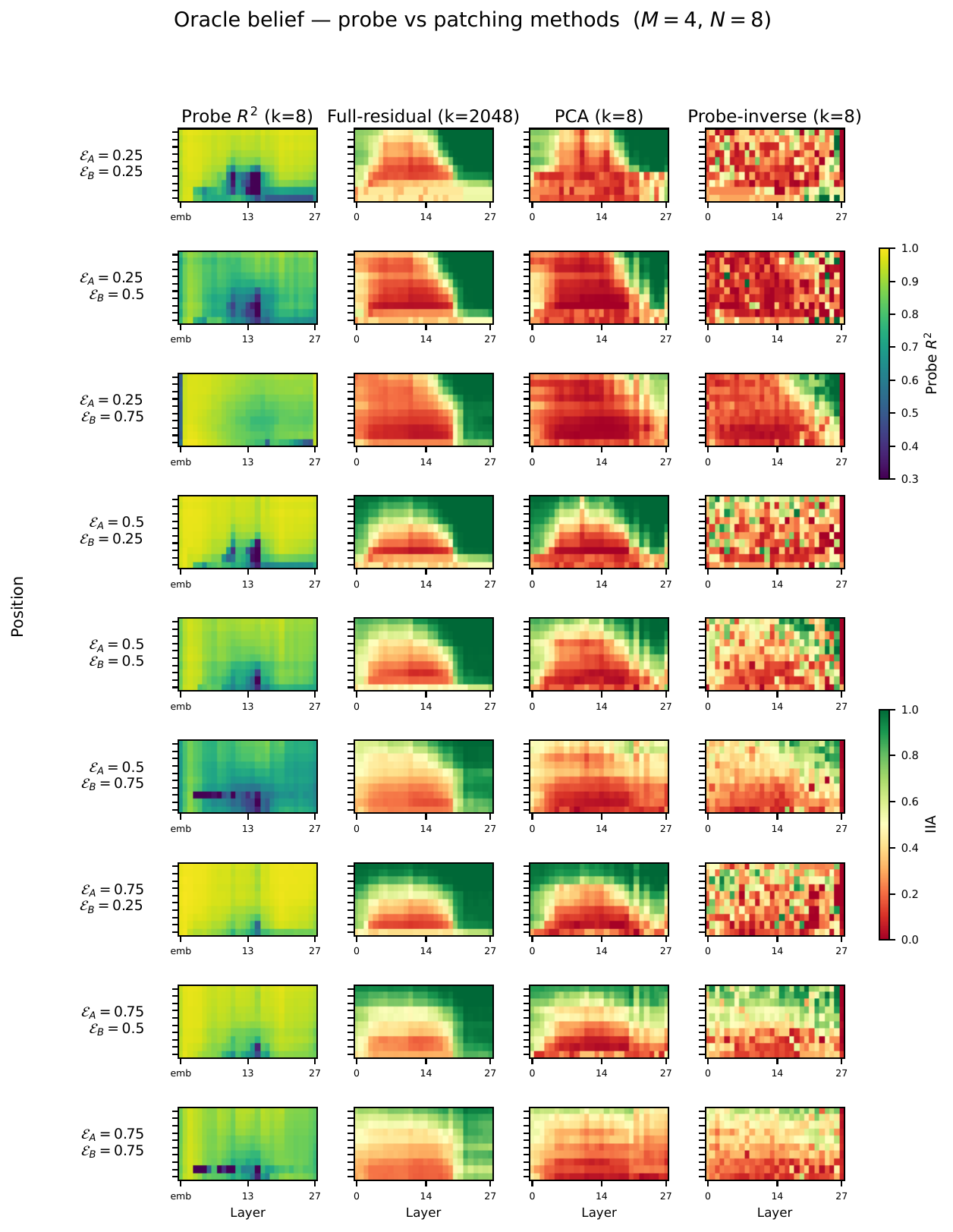}
    \caption{\textbf{Causal interventions for $N=8$.} Rows are HMM configs (transition entropy $\mathcal{E}(\mathbf{A})$, emission entropy $\mathcal{E}(\mathbf{B})$). Columns: linear-probe $R^2$ for the oracle belief $\mathbb P(h_t|\mathbf{o}_{1:t},\boldsymbol{\lambda})$ ($k{=}8$), then interchange-intervention accuracy (IIA) for full-residual ($k{=}2048$), PCA ($k{=}8$), and probe-inverse ($k{=}8$) patches. Each panel is a layer\,$\times$\,position heatmap.}
    \label{fig:oracle_belief_patch_nobs8}
\end{figure}

\subsection{Additional Results on Algorithm Beliefs}
We consider four families of algorithm beliefs: classical $n$-gram models (Bigram, Trigram), linear $n$-gram learners (Gradient CE, Gradient MSE, Ridge MSE), non-linear $n$-gram learners (Kernel, Kron), and spectral predictors (Norm and Norm (GT)). For each algorithm, we train a linear probe with dimension $k=8$ to predict that algorithm's next-token distribution from the residual stream. We then evaluate both the held-out probe $R^2$ and the interchange-intervention accuracy (IIA) obtained from a probe-inverse intervention, which patches the corresponding $k=8$ latent subspace from a source sequence into a target sequence and measures the resulting agreement with the algorithm's own prediction. Using the same subspace for both decoding and intervention enables a direct comparison between representational accessibility and causal influence.

Results are shown for Qwen3-1.7B across the nine HMM configurations defined by transition entropy $\mathcal{E}(\mathbf{A}) \in \{0.25, 0.5, 0.75\}$ and emission entropy $\mathcal{E}(\mathbf{B}) \in \{0.25, 0.5, 0.75\}$, with $M=N=4$, as shown in Figures~\ref{fig:alg_belief_inverse_ngram}--\ref{fig:alg_belief_inverse_spectral}.

\begin{figure}
    \centering
    \includegraphics[width=0.9\linewidth]{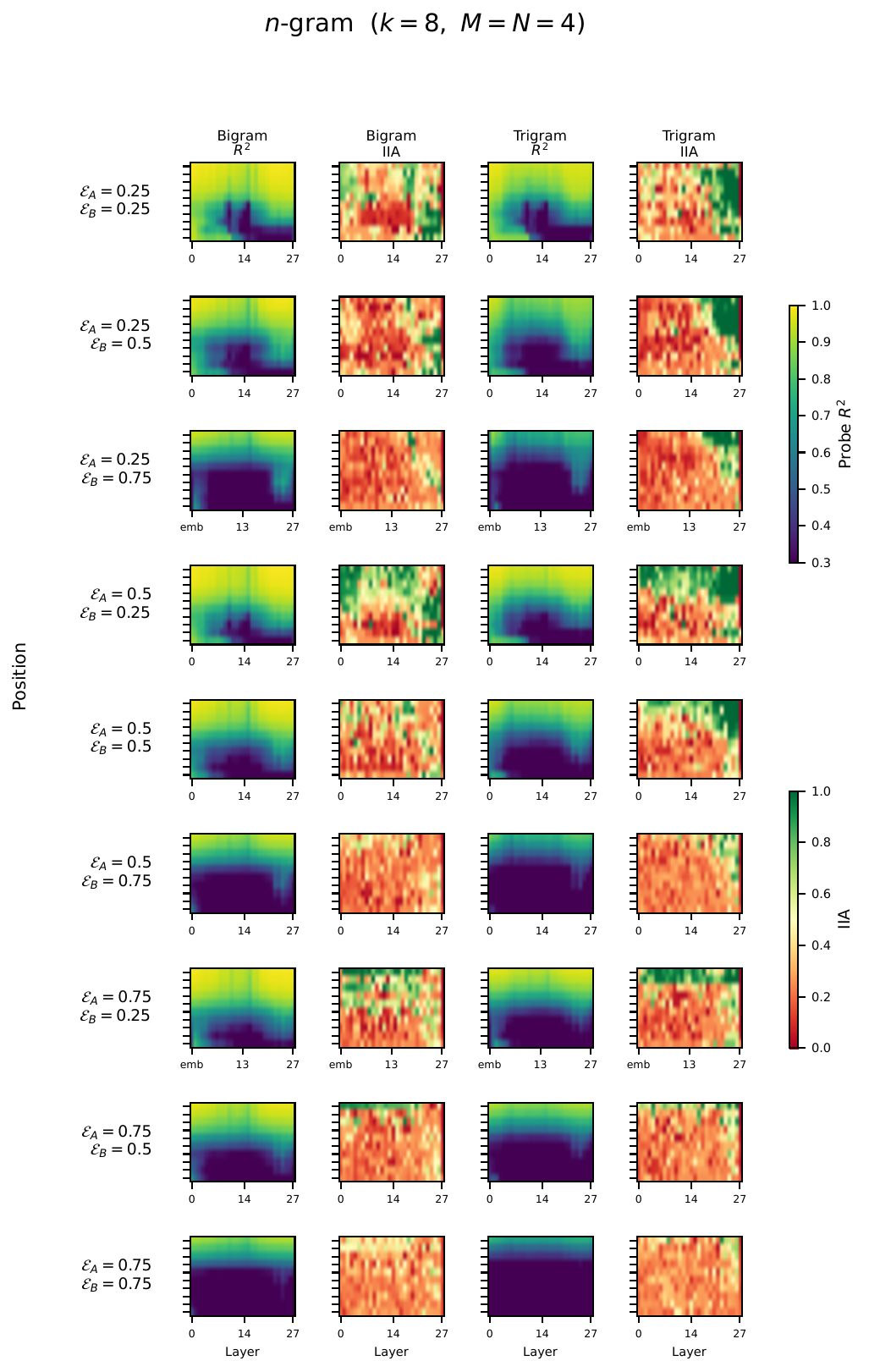}
    \caption{\textbf{Algorithm beliefs: classical $n$-gram predictors ($k=8$, $M=N=4$).} Rows correspond to HMM configurations indexed by transition entropy $\mathcal{E}(\mathbf{A})$ and emission entropy $\mathcal{E}(\mathbf{B})$. For each target predictor (Bigram, Trigram), the first column shows the held-out linear-probe $R^2$ for decoding the predictor's next-token distribution from the residual stream (\texttt{viridis}, $[0.3,1]$), and the second column shows the probe-inverse interchange-intervention accuracy (IIA) obtained by patching the corresponding $k=8$ latent subspace (\texttt{RdYlGn}, $[0,1]$). Each panel is a layer$\times$position heatmap. Results are shown for Qwen3-1.7B.}
    \label{fig:alg_belief_inverse_ngram}
\end{figure}

\begin{figure}
    \centering
    \includegraphics[width=\linewidth]{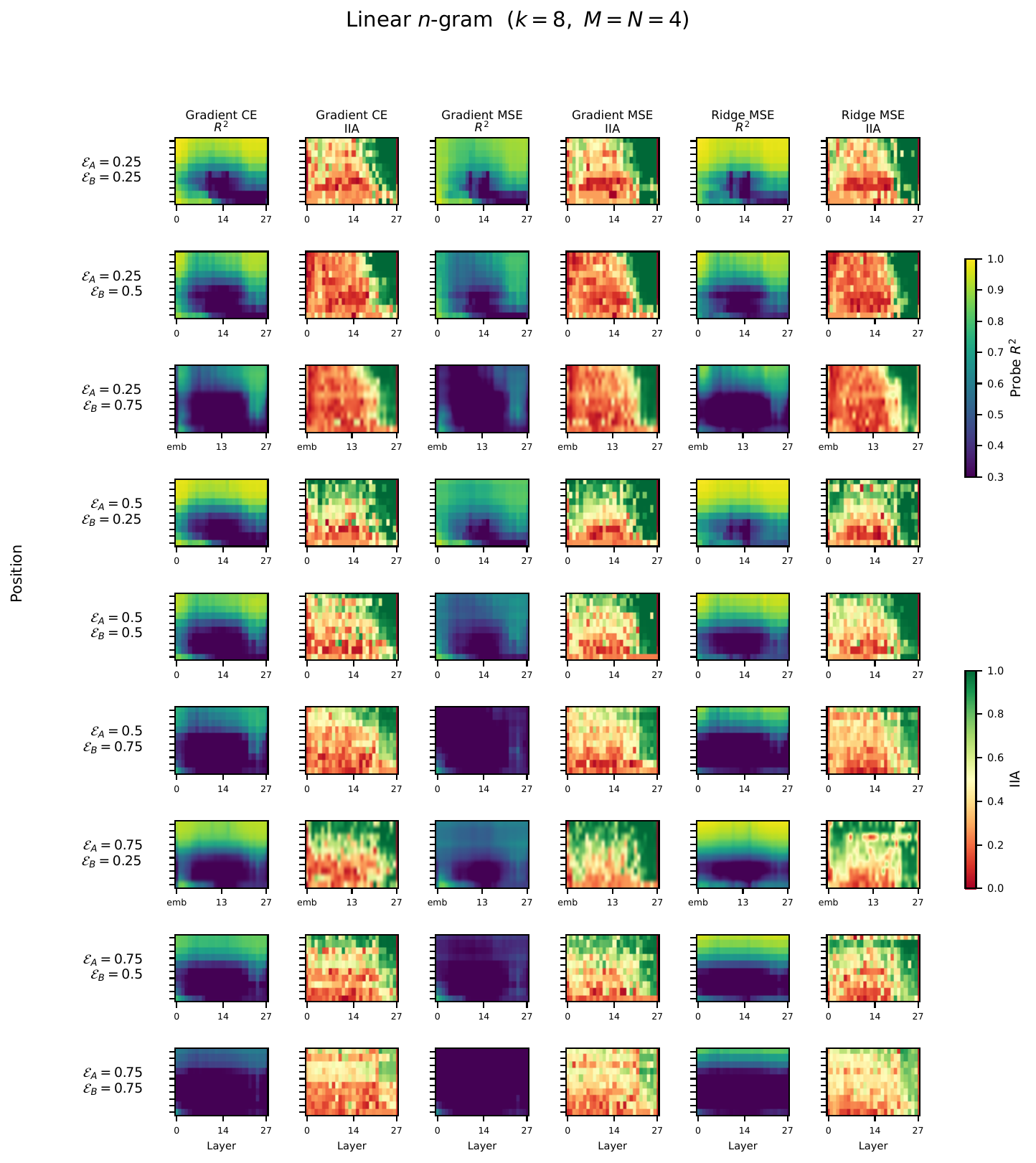}
    \caption{\textbf{Algorithm beliefs: linear $n$-gram learners ($k=8$, $M=N=4$).} Same format as Fig.~\ref{fig:alg_belief_inverse_ngram}, for the linear predictors Gradient CE, Gradient MSE, and Ridge MSE.}

    \label{fig:alg_belief_inverse_gradlinear}
\end{figure}

\begin{figure}
    \centering
    \includegraphics[width=0.9\linewidth]{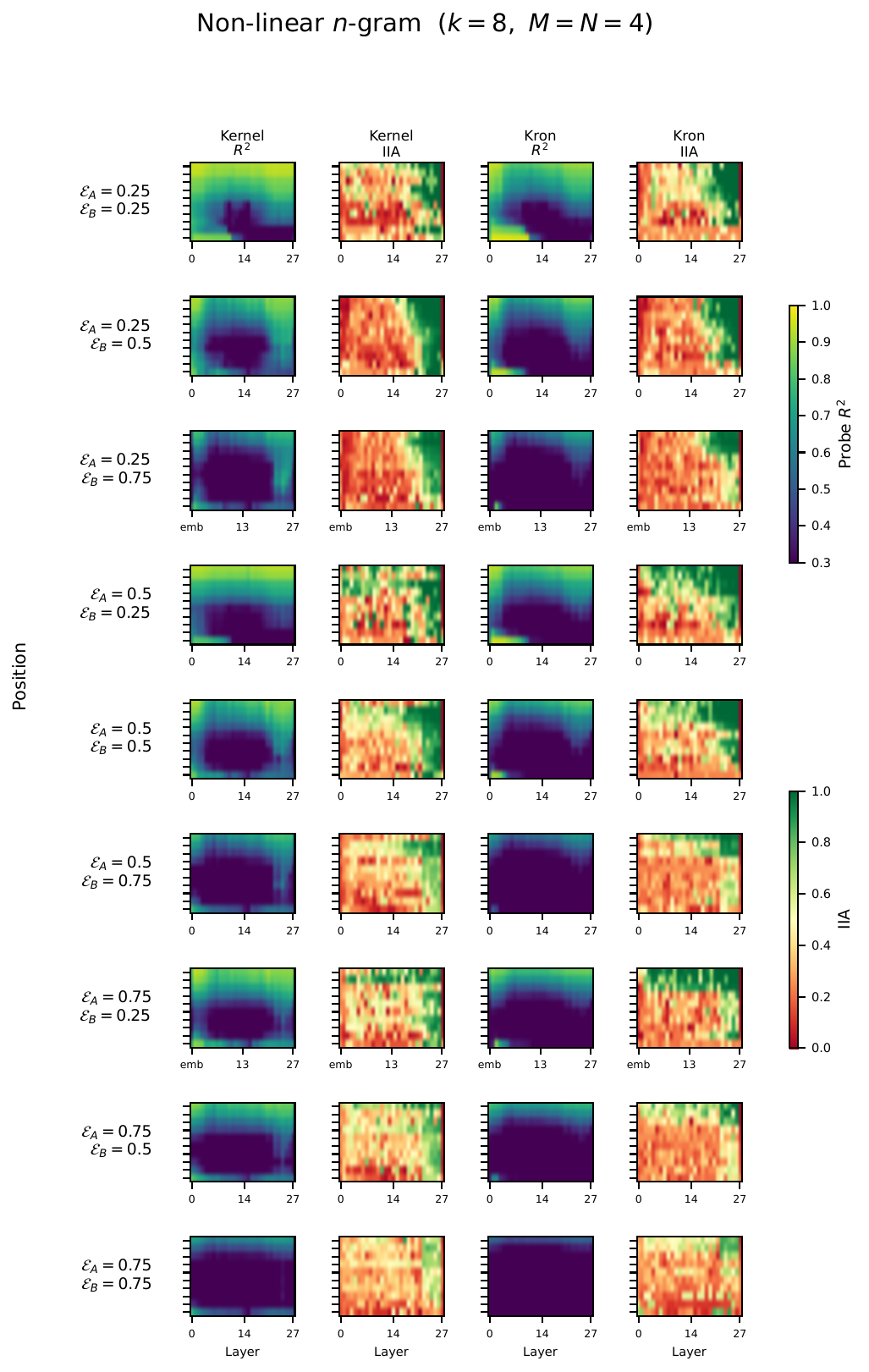}
    \caption{\textbf{Algorithm beliefs: non-linear $n$-gram learners ($k=8$, $M=N=4$).} Same format as Fig.~\ref{fig:alg_belief_inverse_ngram}, for the non-linear predictors Kernel and Kron.}

    \label{fig:alg_belief_inverse_kron_kernel}
\end{figure}

\begin{figure}
    \centering
    \includegraphics[width=0.9\linewidth]{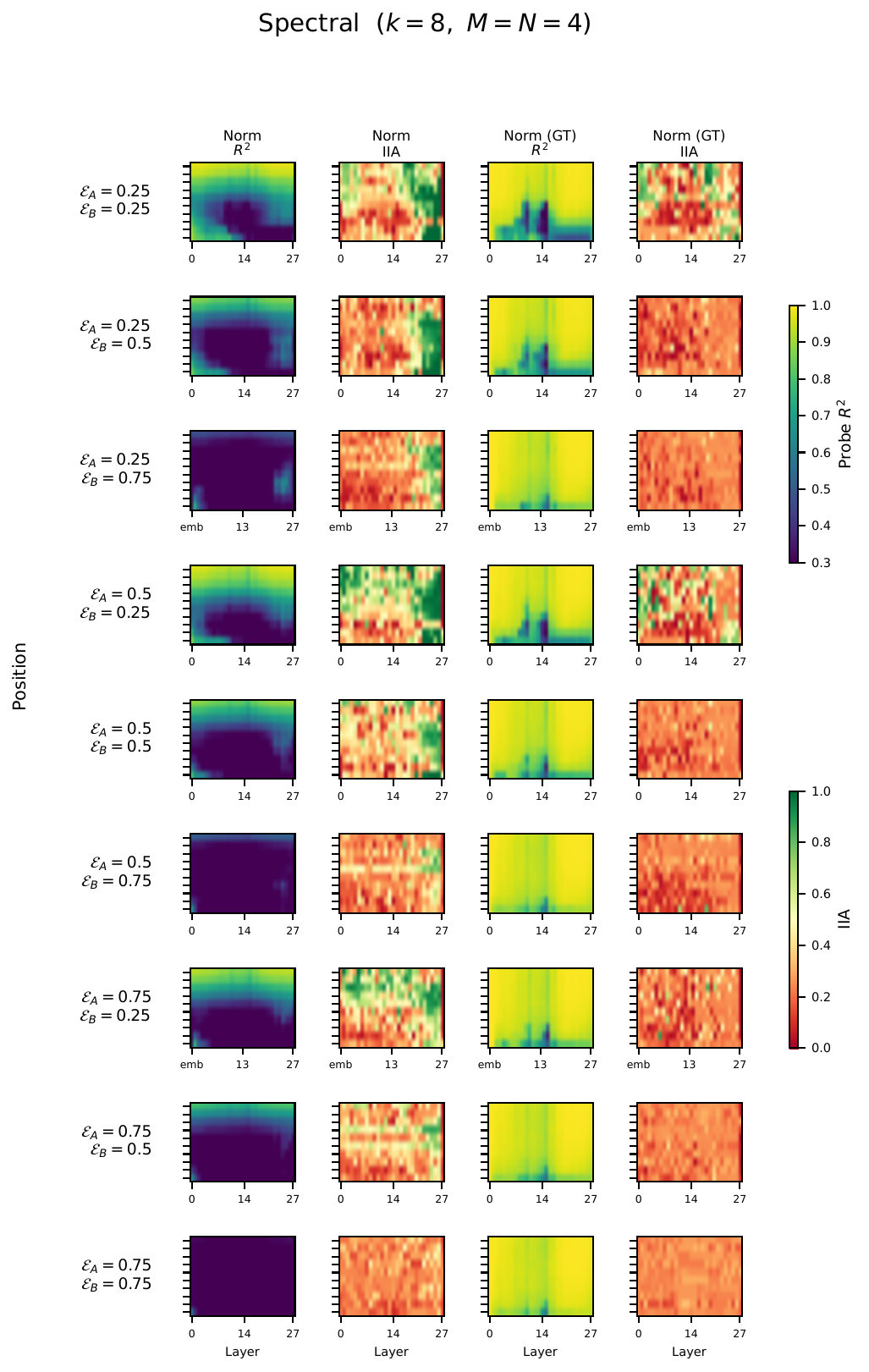}
    \caption{\textbf{Algorithm beliefs: spectral predictors ($k=8$, $M=N=4$).} Same format as Fig.~\ref{fig:alg_belief_inverse_ngram}, for the spectral next-token predictors Norm (normalized observable-operator recursion using empirically estimated moments) and Norm (GT), which uses the same recursion with ground-truth moments.}
    \label{fig:alg_belief_inverse_spectral}
\end{figure}

\subsection{More Examples of Algorithm Operator Probing}
Complementing the algorithm \emph{beliefs}, we probe for the algorithm
\emph{operators}---the weight matrices that parameterize each predictor---to test
whether the residual stream linearly encodes the operators themselves, not just
their outputs. For each operator we fit a PCA($k$)\,+\,ridge probe at every layer (token position $64$) and report the held-out MSE of the recovered operator, one
curve per probe dimension $k$. We consider three operators: the linear soft $n$-gram weight $W\in\mathbb{R}^{N\times Nn}$ (window $n{=}4$, lags $o_{t-3:t}$), the spectral observation operator $B^{\mathrm{norm}}_k$, and the non-linear
Kronecker $n$-gram weight (window $n{=}4$). A representative ground-truth vs. recovered operator matrix appears in the main paper (Fig.~\ref{fig:LLM_W}); here we summarize the recovery MSE across the nine non-trivial entropy configurations at $N{=}4$ (Qwen3-1.7B). MSE decreases monotonically with the probe dimension $k$ and is lowest at the early and late layers (worst in the middle), mirroring the belief and prediction probes, as shown in Figures~\ref{fig:op_probe_W}--\ref{fig:op_probe_kron}.

\begin{figure}
  \centering
  \includegraphics[width=0.7\linewidth]{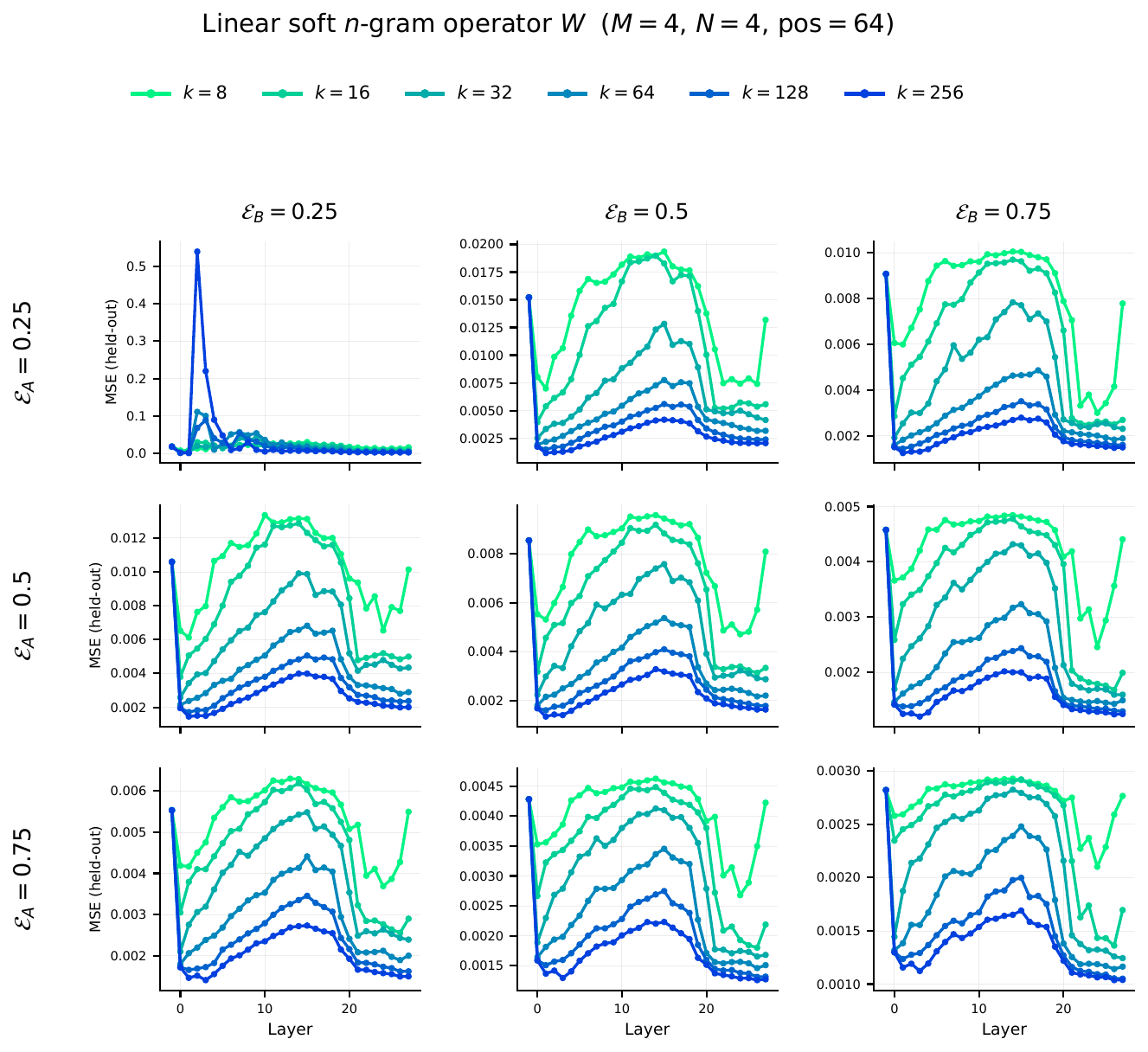}
  \caption{\textbf{Linear soft $n$-gram operator $W$ (window $4$, $M{=}N{=}4$).}
  Rows are transition entropy $\mathcal{E}(\mathbf{A})$, columns emission entropy
  $\mathcal{E}(\mathbf{B})$. Each cell shows the held-out MSE of the recovered weight $W\in\mathbb{R}^{K\times K\cdot4}$ vs.\ layer at position $64$, one curve per probe dimension $k$. $y$-axes are autoscaled per cell. Qwen3-1.7B.}
  \label{fig:op_probe_W}
\end{figure}

\begin{figure}
  \centering
  \includegraphics[width=0.7\linewidth]{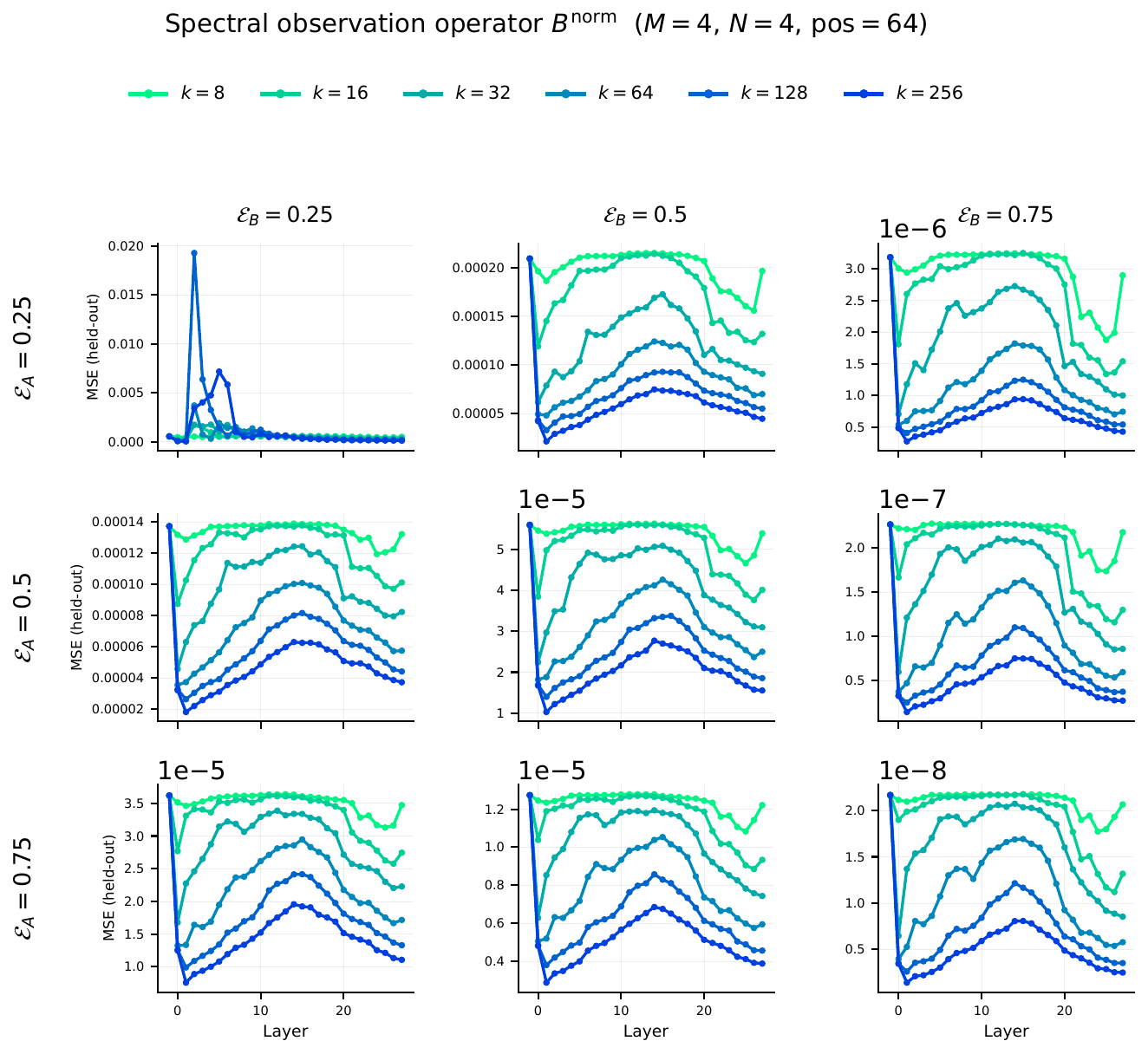}
  \caption{\textbf{Spectral observation operator $B^{\mathrm{norm}}$  ($M{=}N{=}4$).} As Fig.~\ref{fig:op_probe_W}, for the normalized observation operator $B^{\mathrm{norm}}_k=\hat{P_2}^{\!\top}\hat{P_3}[k]$. Note the MSE scale
  shrinks sharply with emission entropy (from $\sim\!10^{-4}$ at $\mathcal{E}(\mathbf{B}){=}0.25$ to $\sim\!10^{-8}$ at $0.75$) as the operator
  becomes near-uniform.}
  \label{fig:op_probe_B}
\end{figure}

\begin{figure}
  \centering
  \includegraphics[width=0.7\linewidth]{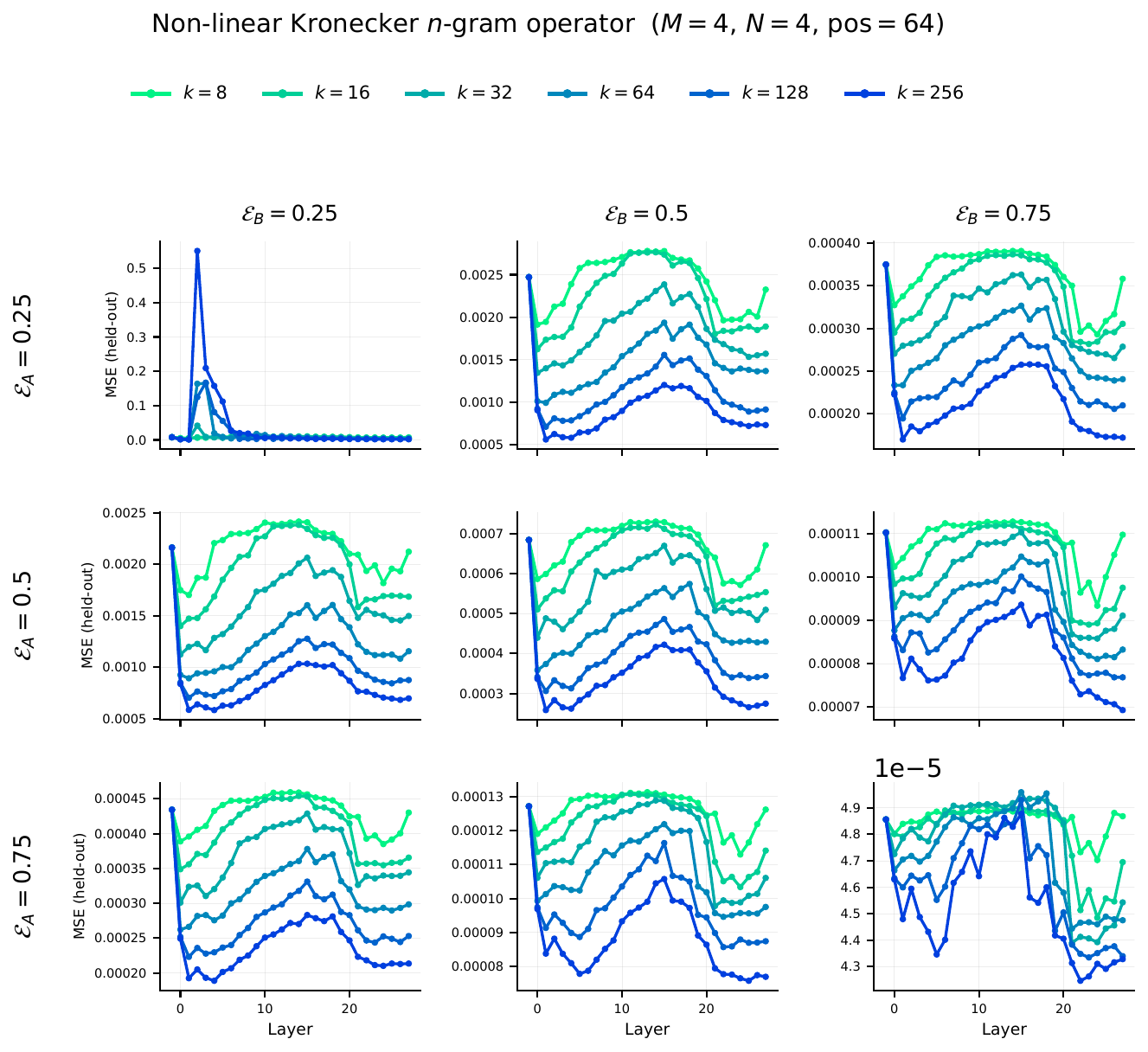}
  \caption{\textbf{Non-linear Kronecker $n$-gram operator (window $4$, $M{=}N{=}4$).} As  Fig.~\ref{fig:op_probe_W}, for the Kronecker $n$-gram weight (product
  features over the last $4$ observations, $K^4$ columns).}
  \label{fig:op_probe_kron}
\end{figure}

\subsection{Different LLM Sizes}
To assess whether our findings depend on model scale, we repeat the oracle-belief probing and full-residual interchange-intervention experiments using Qwen3-4B and compare the results to Qwen3-1.7B. Figure~\ref{fig:qwen3_4b_vs_1.7b_r2_iia_full_residual_appendix_k8} shows that the qualitative patterns are highly consistent across model sizes. In particular, the oracle belief remains strongly decodable from the residual stream, and belief patching continues to exhibit high interchange-intervention accuracy in the same regions of layer-position space. While the larger model generally achieves slightly higher probe $R^2$ and IIA, the overall geometry and causal localization of the belief representation remain largely unchanged, suggesting that the mechanisms identified in this work are robust across model scale.

\begin{figure}
    \centering
    \includegraphics[width=0.9\linewidth]{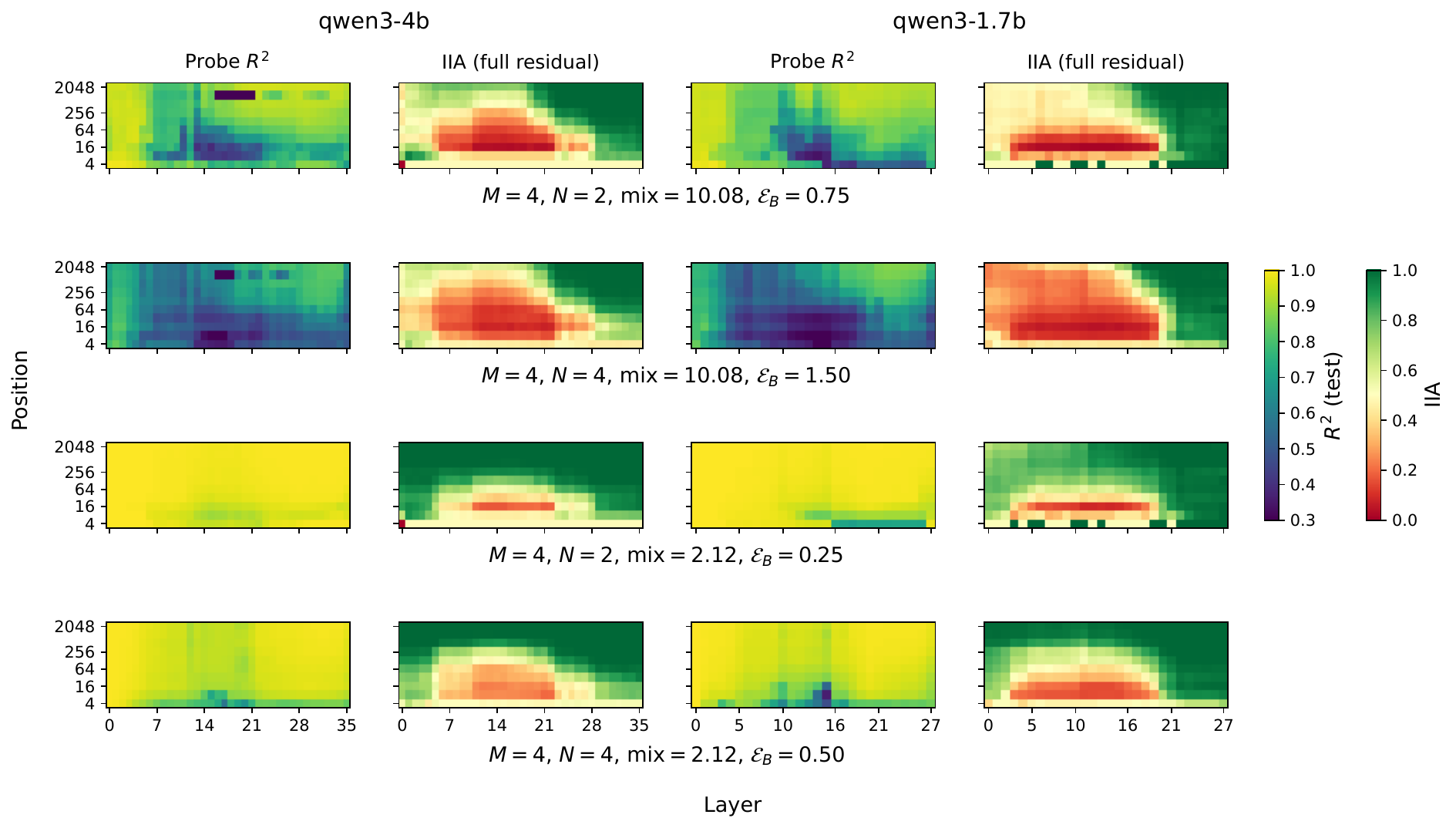}
    \caption{\textbf{Comparison across model sizes.} Oracle-belief probe $R^2$ and full-residual interchange-intervention accuracy (IIA) for Qwen3-1.7B and Qwen3-4B.}
    \label{fig:qwen3_4b_vs_1.7b_r2_iia_full_residual_appendix_k8}
\end{figure}

\newpage
\subsection{Patching at Different Timesteps $t$}
\label{sec:window_patch}

The interventions in the main text patch a single token position. To study where in the sequence the belief driving the next-token prediction is stored, we instead patch contiguous windows of residual activations. For a given source-target pair, we overwrite the target residual stream with the source residual stream at a fixed layer over a window of $n$ consecutive positions and then evaluate the resulting next-observation prediction at timestep $t$. We use full-residual interventions ($k=d_{\mathrm{model}}$), corresponding to direct residual replacement. Results are reported as \emph{IIA efficacy}.

We compare two window placements. In the \emph{non-shifted} setting, the patch window $[t-n+1:t]$ includes the query position $t$ (Fig.~\ref{fig:window_patch_nonshifted}). In the \emph{shifted} setting, the window $[t-n:t-1]$ excludes $t$ and patches only the preceding positions (Fig.~\ref{fig:window_patch_shifted}). Comparing these interventions isolates the extent to which predictive information is stored locally at the current token versus distributed across earlier context. We sweep window sizes $n\in\{1,2,4,6\}$ across all layers and timesteps for the nine HMM configurations with $\mathcal{E}(\mathbf{A}),\mathcal{E}(\mathbf{B})\in\{0.25,0.5,0.75\}$ at $N=4$.

Two consistent patterns emerge. First, patching windows that include the query position almost completely transfers the source prediction, even for $n=1$, with efficacy increasing further as the window expands. This indicates that the belief used for prediction is largely represented in the residual stream at the current token. Second, shifted windows excluding the query position have little effect when $n=1$, but their efficacy grows substantially with window size, particularly at later layers and later timesteps. This growth is most pronounced in high-emission-entropy regimes, where the current observation provides less information about the latent state and belief must be accumulated over longer contexts. In these settings, patching the preceding $4$--$6$ positions is often sufficient to transplant the source belief. The contrast between shifted and non-shifted interventions therefore quantifies the degree to which predictive information is stored locally versus distributed across context, mirroring the observation-dominated and belief-dominated regimes discussed in the main text.

\newpage
\begin{figure}[H]
    \centering
    \includegraphics[width=0.9\linewidth]{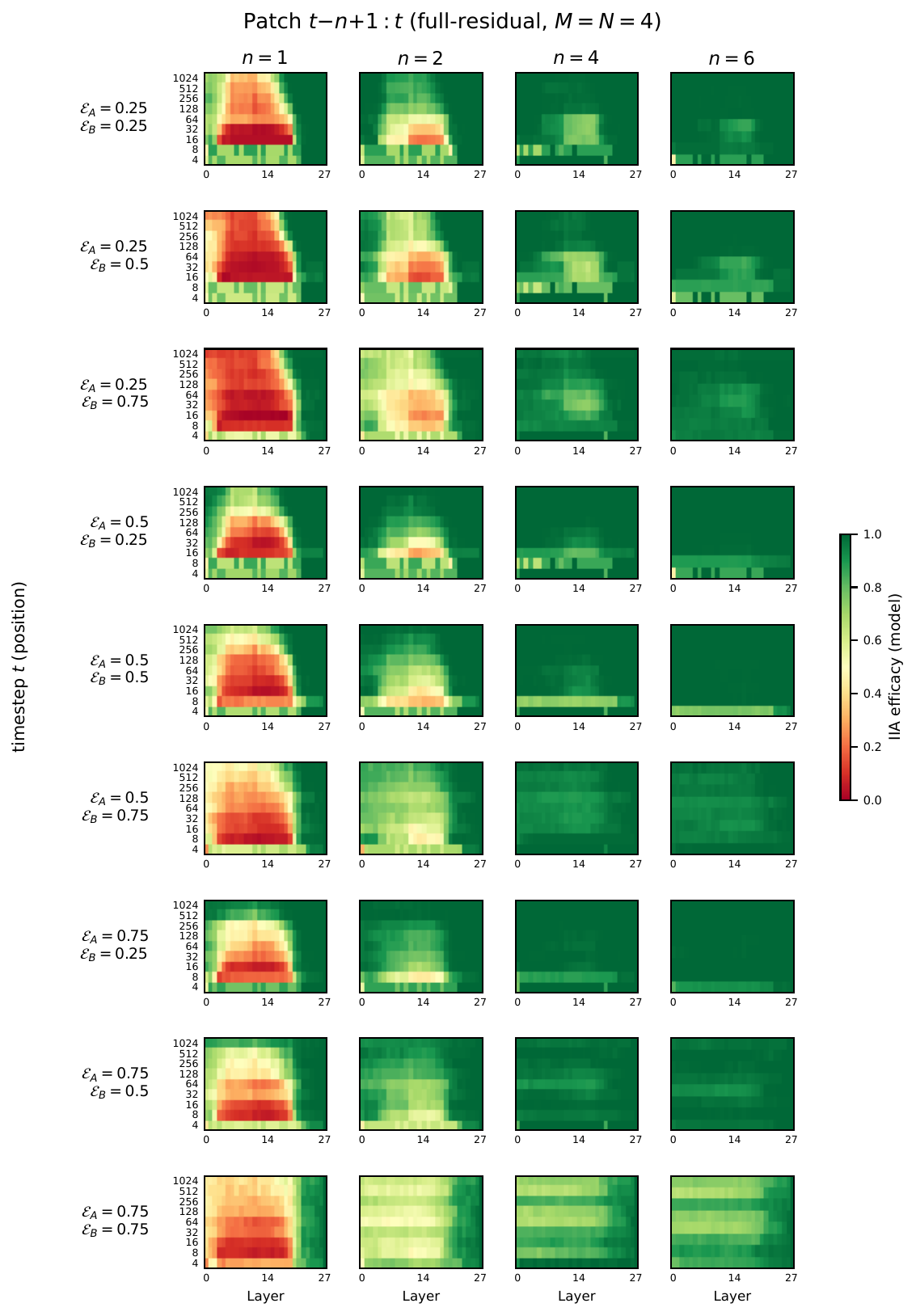}
    \caption{\textbf{Window patching including the query timestep.} IIA efficacy of a full-residual interchange intervention that overwrites the target residual stream over the window $[t-n+1:t]$, which includes the query position $t$, and evaluates the resulting next-observation prediction at $t$. Rows correspond to the nine HMM configurations with $\mathcal{E}(\mathbf{A}),\mathcal{E}(\mathbf{B})\in\{0.25,0.5,0.75\}$ at $N=4$, while columns vary the window size $n\in\{1,2,4,6\}$. Each panel is a layer ($x$: $0,\ldots,27$) $\times$ query timestep ($y$: $4,\ldots,1024$) heatmap of IIA efficacy (RdYlGn; green $=1$ indicates complete recovery of the source prediction, red $=0$ indicates no effect). Efficacy is high even for $n=1$, demonstrating that the prediction-relevant belief is largely encoded in the residual representation at the current token. Results are shown for Qwen3-1.7B.}
    \label{fig:window_patch_nonshifted}
\end{figure}

\begin{figure}[H]
    \centering
    \includegraphics[width=0.9\linewidth]{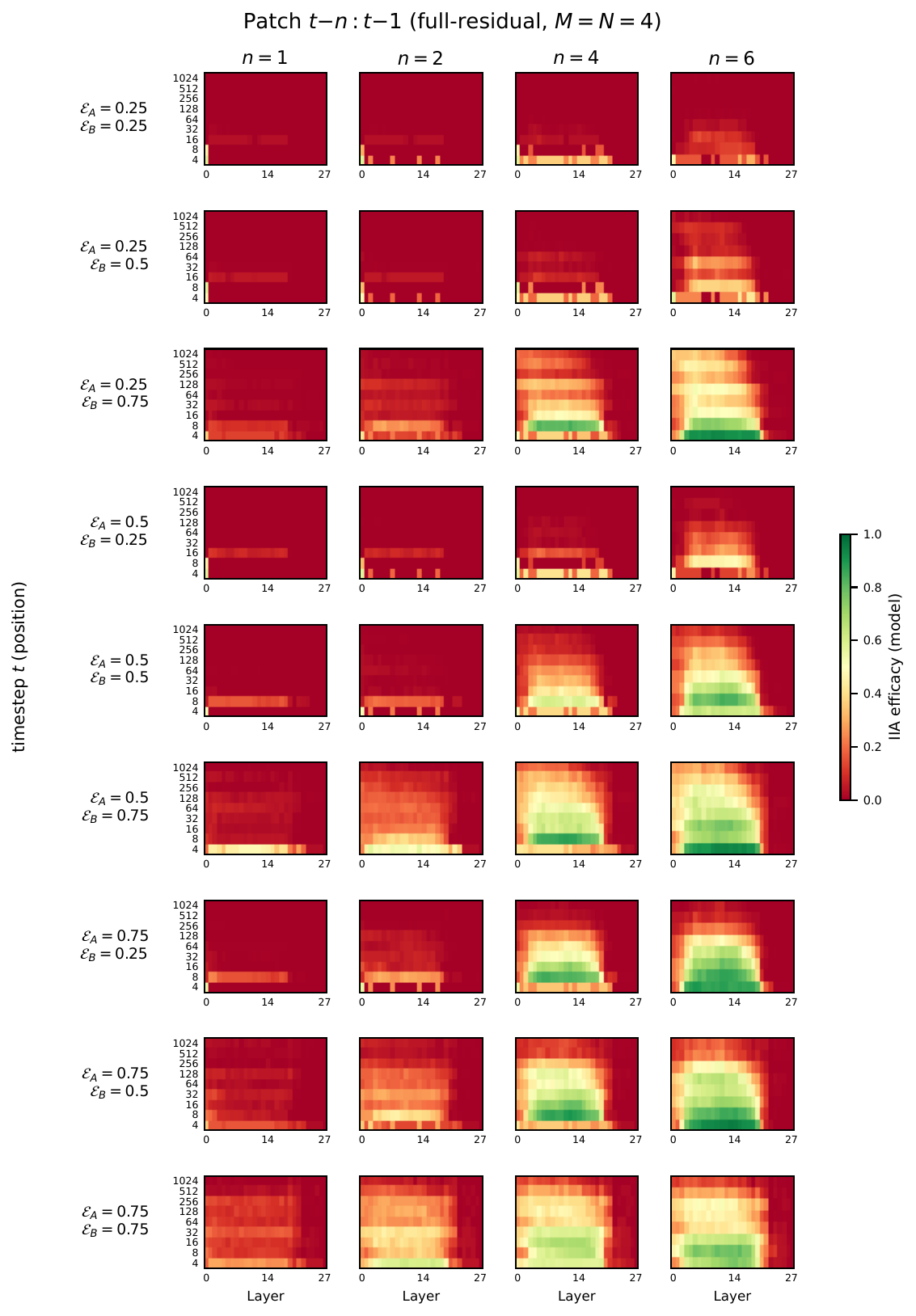}
    \caption{\textbf{Window patching excluding the query timestep.} As in Fig.~\ref{fig:window_patch_nonshifted}, but the patched window $[t-n:t-1]$ is shifted one step left so that the query position $t$ remains \emph{untouched} and only the $n$ preceding positions are overwritten, while the readout is still performed at $t$. At $n=1$, the intervention is near-inert (mostly red), indicating that the prediction does not rely solely on the immediately preceding token. Efficacy increases with window size $n$, with stronger effects at later layers and timesteps, particularly in high-emission-entropy regimes (larger $\mathcal{E}_B$, lower rows), where belief must be integrated over a longer context. Results are shown for Qwen3-1.7B using full-residual swaps ($k=d_{\mathrm{model}}$).} \label{fig:window_patch_shifted}
\end{figure}

%% file: checklist.tex
\section*{NeurIPS Paper Checklist}

The checklist is designed to encourage best practices for responsible machine learning research, addressing issues of reproducibility, transparency, research ethics, and societal impact. Do not remove the checklist: {\bf The papers not including the checklist will be desk rejected.} The checklist should follow the references and follow the (optional) supplemental material.  The checklist does NOT count towards the page
limit. 

Please read the checklist guidelines carefully for information on how to answer these questions. For each question in the checklist:
\begin{itemize}
    \item You should answer \answerYes{}, \answerNo{}, or \answerNA{}.
    \item \answerNA{} means either that the question is Not Applicable for that particular paper or the relevant information is Not Available.
    \item Please provide a short (1--2 sentence) justification right after your answer (even for \answerNA). 
\end{itemize}

{\bf The checklist answers are an integral part of your paper submission.} They are visible to the reviewers, area chairs, senior area chairs, and ethics reviewers. You will also be asked to include it (after eventual revisions) with the final version of your paper, and its final version will be published with the paper.

The reviewers of your paper will be asked to use the checklist as one of the factors in their evaluation. While \answerYes{} is generally preferable to \answerNo{}, it is perfectly acceptable to answer \answerNo{} provided a proper justification is given (e.g., error bars are not reported because it would be too computationally expensive'' or ``we were unable to find the license for the dataset we used''). In general, answering \answerNo{} or \answerNA{} is not grounds for rejection. While the questions are phrased in a binary way, we acknowledge that the true answer is often more nuanced, so please just use your best judgment and write a justification to elaborate. All supporting evidence can appear either in the main paper or the supplemental material, provided in appendix. If you answer \answerYes{} to a question, in the justification please point to the section(s) where related material for the question can be found.

IMPORTANT, please:
\begin{itemize}
    \item {\bf Delete this instruction block, but keep the section heading ``NeurIPS Paper Checklist"},
    \item  {\bf Keep the checklist subsection headings, questions/answers and guidelines below.}
    \item {\bf Do not modify the questions and only use the provided macros for your answers}.
\end{itemize}


\begin{enumerate}

\item {\bf Claims}
    \item[] Question: Do the main claims made in the abstract and introduction accurately reflect the paper's contributions and scope?
    \item[] Answer: \answerYes{}{} 
    \item[] Justification: Yes the abstract clearly states the claims that are made in the paper, with evidence from theoretical and expiremental results in the main text and appendix.

\item {\bf Limitations}
    \item[] Question: Does the paper discuss the limitations of the work performed by the authors?
    \item[] Answer: \answerYes{} 
    \item[] Justification: Limtations of this work are discussed in the appendix.

\item {\bf Theory assumptions and proofs}
    \item[] Question: For each theoretical result, does the paper provide the full set of assumptions and a complete (and correct) proof?
    \item[] Answer: \answerYes{} 
    \item[] Justification: Full proof is given in the appendix.

    \item {\bf Experimental result reproducibility}
    \item[] Question: Does the paper fully disclose all the information needed to reproduce the main experimental results of the paper to the extent that it affects the main claims and/or conclusions of the paper (regardless of whether the code and data are provided or not)?
    \item[] Answer: \answerYes{} 
    \item[] Justification: Yes, detailed explanation is given both in the main text and appendix.

\item {\bf Open access to data and code}
    \item[] Question: Does the paper provide open access to the data and code, with sufficient instructions to faithfully reproduce the main experimental results, as described in supplemental material?
    \item[] Answer: \answerNo{} 
    \item[] Justification: Because code is not hard to reproduce given the detailed descriptions, we do not supply code at submission time, but commit to releasing code upon acceptance.

\item {\bf Experimental setting/details}
    \item[] Question: Does the paper specify all the training and test details (e.g., data splits, hyperparameters, how they were chosen, type of optimizer) necessary to understand the results?
    \item[] Answer: \answerYes{} 
    \item[] Justification: This is detailed in the appendix.

\item {\bf Experiment statistical significance}
    \item[] Question: Does the paper report error bars suitably and correctly defined or other appropriate information about the statistical significance of the experiments?
    \item[] Answer: \answerYes{} 
    \item[] Justification: The crucial experiments were done using linear regression which has a closed form solution.

\item {\bf Experiments compute resources}
    \item[] Question: For each experiment, does the paper provide sufficient information on the computer resources (type of compute workers, memory, time of execution) needed to reproduce the experiments?
    \item[] Answer: \answerYes{} 
    \item[] Justification: Please see appendix.
    
\item {\bf Code of ethics}
    \item[] Question: Does the research conducted in the paper conform, in every respect, with the NeurIPS Code of Ethics \url{https://neurips.cc/public/EthicsGuidelines}?
    \item[] Answer: \answerYes{}{} 
    \item[] Justification: Yes. We have reviewed the code of ethics.

\item {\bf Broader impacts}
    \item[] Question: Does the paper discuss both potential positive societal impacts and negative societal impacts of the work performed?
    \item[] Answer: \answerNo{} 
    \item[] Justification: We do not feel there are significant immediate negative societal impacts of this work.
    
\item {\bf Safeguards}
    \item[] Question: Does the paper describe safeguards that have been put in place for responsible release of data or models that have a high risk for misuse (e.g., pre-trained language models, image generators, or scraped datasets)?
    \item[] Answer: \answerNA{} 
    \item[] Justification: No such risk.

\item {\bf Licenses for existing assets}
    \item[] Question: Are the creators or original owners of assets (e.g., code, data, models), used in the paper, properly credited and are the license and terms of use explicitly mentioned and properly respected?
    \item[] Answer: \answerYes{} 
    \item[] Justification: Yes, all assets are properly credited or copyright of the authors.

\item {\bf New assets}
    \item[] Question: Are new assets introduced in the paper well documented and is the documentation provided alongside the assets?
    \item[] Answer: \answerNA{} 
    \item[] Justification: At the time of submission, the paper does not release new assets.

\item {\bf Crowdsourcing and research with human subjects}
    \item[] Question: For crowdsourcing experiments and research with human subjects, does the paper include the full text of instructions given to participants and screenshots, if applicable, as well as details about compensation (if any)? 
    \item[] Answer: \answerNA{}{} 
    \item[] Justification: Does not involve crowdsourcing.

\item {\bf Institutional review board (IRB) approvals or equivalent for research with human subjects}
    \item[] Question: Does the paper describe potential risks incurred by study participants, whether such risks were disclosed to the subjects, and whether Institutional Review Board (IRB) approvals (or an equivalent approval/review based on the requirements of your country or institution) were obtained?
    \item[] Answer: \answerNA{}{} 
    \item[] Justification: The paper does not involve crowdsourcing nor research with human subjects.

\item {\bf Declaration of LLM usage}
    \item[] Question: Does the paper describe the usage of LLMs if it is an important, original, or non-standard component of the core methods in this research? Note that if the LLM is used only for writing, editing, or formatting purposes and does \emph{not} impact the core methodology, scientific rigor, or originality of the research, declaration is not required.
    \item[] Answer: \answerNo{} 
    \item[] Justification: LLMs were used to refine writing but not other reasons.

\end{enumerate}